\documentclass{article}
\usepackage[final]{neurips_2026}

\usepackage[utf8]{inputenc}
\usepackage[T1]{fontenc}
\usepackage{url}
\usepackage{microtype}
\usepackage{nicefrac}
\usepackage{pifont}

\usepackage{booktabs}
\usepackage{multirow}
\usepackage{placeins} 
\usepackage{subcaption}
\usepackage{graphicx}
\usepackage{xcolor}
\usepackage{algorithm}
\usepackage{algorithmic}
\usepackage{enumitem}
\usepackage{float}

\usepackage{amsmath}
\usepackage{amssymb}
\usepackage{amsfonts}
\usepackage{amsthm}

\newtheorem{proposition}{Proposition}
\newtheorem*{propstar}{Proposition} 
\newtheorem{theorem}{Theorem}
\newtheorem{lemma}{Lemma}
\newtheorem{remark}{Remark}

\providecommand{\answerYes}[1][]{[Yes] #1}
\providecommand{\answerNo}[1][]{[No] #1}
\providecommand{\answerNA}[1][]{[N/A] #1}

\usepackage{hyperref}

\title{SphereVAD: Training-Free Video Anomaly Detection via Geodesic Inference on the Unit Hypersphere}

\author{%
  Chao Huang\textsuperscript{1} \quad
  Pengfei Wei\textsuperscript{1} \quad
  Wei Wang\textsuperscript{1} \quad
  Jie Wen\textsuperscript{2} \quad
  Zhihua Wang\textsuperscript{1} \\[2pt]
  Li Shen\textsuperscript{1} \quad
  Wenqi Ren\textsuperscript{1} \quad
  Xiaochun Cao\textsuperscript{1} \\[4pt] \\
\texttt{\{huangch253, wangwei29, wangzh, lis, caoxiaochun\}@email.sysu.edu.cn} \\
\texttt{weipf@mail2.sysu.edu.cn}\\
\texttt{wenjie@hit.edu.cn}
}

\newcommand{\cmark}{{\color{green!70!black}\ding{51}}}%
\newcommand{\xmark}{{\color{red}\ding{55}}}%
\makeatletter
\providecommand{\@trackname}{}
\makeatother

\begin{document}

\maketitle

\begin{abstract}
Video anomaly detection (VAD) aims to automatically identify events that deviate from normal patterns in untrimmed surveillance videos.
Existing methods universally depend on large-scale annotations or task-specific training procedures, severely limiting their rapid deployment to novel scenes.
We observe that intermediate-layer features of pre-trained multimodal large language models (MLLMs) already encode rich anomaly semantics, yet existing approaches rely on the language output pathway and fail to exploit the geometric discriminability latent in these representations.
Based on this finding, we propose \textbf{SphereVAD}, a fully training-free, zero-shot VAD framework that recasts anomaly discrimination as von Mises--Fisher (vMF) likelihood-ratio geodesic inference on the unit hypersphere, unleashing latent discriminability through principled geometric reasoning rather than learning new representations.
Specifically, SphereVAD first applies Fr\'{e}chet mean centering to unfold feature distributions and eliminate domain biases, then employs Holistic Scene Attention (HSA) to reinforce feature consistency using cross-video priors, and finally performs vMF-guided Spherical Geodesic Pulling (SGP) to align ambiguous segments with directional prototypes on the spherical manifold.
This training-free pipeline requires only minimal synthetic images for calibration.
SphereVAD establishes new state-of-the-art results among zero-shot, training-free methods on three major benchmarks and remains competitive with weakly supervised baselines.
The code is publicly available at \url{https://github.com/S1inetzz/SphereVAD}.
\end{abstract}

\section{Introduction}
\label{sec:intro}
Video anomaly detection (VAD) aims to automatically identify events that
deviate from expected behaviour in untrimmed surveillance videos, with
broad applications in public safety, intelligent transportation, and
beyond~\cite{ramachandra2020survey,sultani2018real}.
The key challenge of this task lies in the open-ended nature of
anomalies: anomalous event categories are diverse and virtually
inexhaustible, rendering large-scale annotation prohibitively expensive.
Existing approaches, whether weakly
supervised~\cite{sultani2018real,tian2021weakly},
unsupervised / self-trained~\cite{thakare2023dyannet,pang2020self}, or few-shot
paradigms~\cite{lu2020few}, all rely on some form of
training procedure and therefore face bottlenecks including high
computational cost, strong dependence on target-domain data, and the
necessity of retraining when transferring across scenes.

Recent advances in vision-language models and multimodal large language models
(MLLMs)~\cite{li2023blip,radford2021learning} have opened new
possibilities for training-free VAD.
Several methods attempt to leverage the language reasoning capability of
MLLMs to judge anomalies directly~\cite{zanella2024harnessing},
yet a significant performance gap with respect to supervised counterparts
persists.
We find that MLLM intermediate-layer features already encode rich
discriminative information about anomalies.
The real bottleneck is that current methods rely exclusively on the
language output pathway for inference, leaving the abundant semantic
representations in these intermediate layers largely untapped.

However, directly utilising MLLM intermediate-layer features for anomaly
detection exposes three widely overlooked geometric issues.
\textbf{First, conical collapse of features.}
As illustrated in Figure~\ref{fig:motivation}(a), MLLM
intermediate-layer features exhibit a highly conical distribution in
Euclidean space, concentrated within a narrow cone of merely
${\sim}13^{\circ}$.
In such a crowded space, Euclidean distance metrics become extremely
fragile, the most discriminative directional information remains underexploited.
Even after $\ell_2$-normalisation onto the unit hypersphere
(Figure~\ref{fig:motivation}(b)), the distribution remains clustered in
a small spherical cap (angular standard deviation
$\approx 1.7^{\circ}$, as confirmed by the narrow spread in Figure~\ref{fig:discriminative_structure}, left column).
Spherical centering (Figure~\ref{fig:motivation}(c)) effectively unfolds
these features over the full hypersphere (angular standard deviation
$\approx 17.1^{\circ}$, Figure~\ref{fig:discriminative_structure}, middle column), and the separation between normal and
anomalous features improves significantly.
\textbf{Second, domain shift across different data sources.}
Features extracted from different data sources occupy distinct regions on
the hypersphere, manifesting as a systematic rotational offset between
their respective mean directions.
If left unaligned, this rotational bias causes reference directions
calibrated on one domain to be misaligned with features from another,
leading to substantial performance degradation (see
Appendix~\ref{app:domain_shift} for empirical evidence).
\textbf{Third, ambiguous features near the decision boundary.}
In decision-space-based methods, a considerable proportion of features
fall into the ambiguous zone between normal and anomalous regions.
These borderline features are highly sensitive to small perturbations and
can easily be assigned incorrect labels, undermining both the accuracy
and robustness of the detection.

\begin{figure}[t]\centering
  \includegraphics[width=0.95\linewidth]{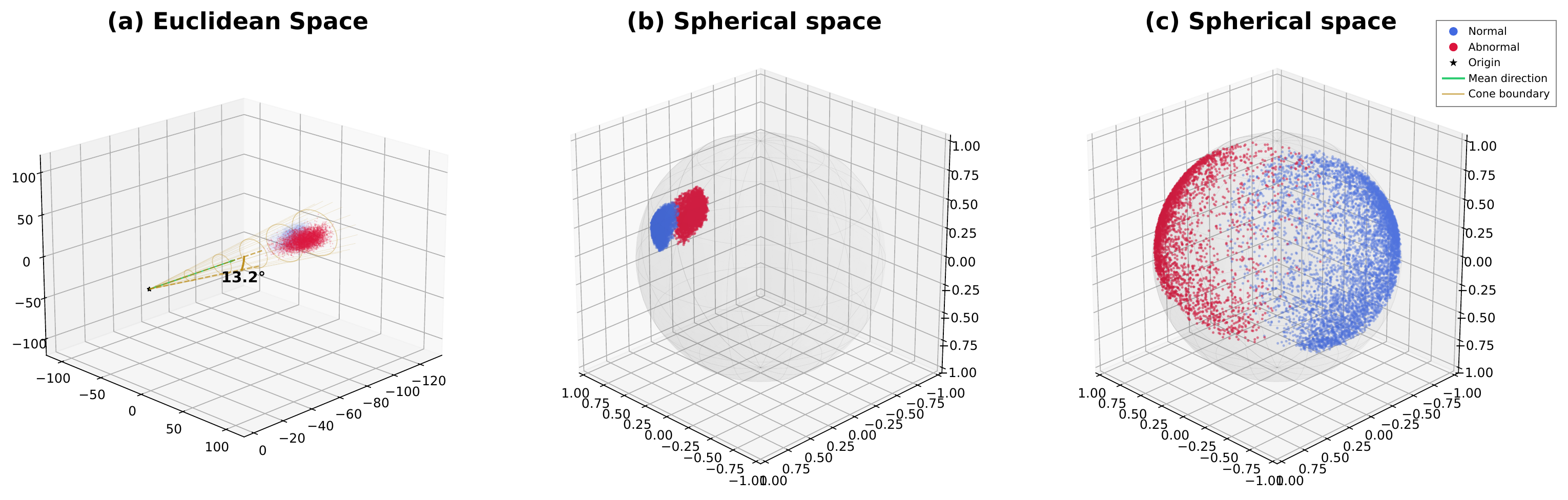}
\caption{Motivation of SphereVAD on XD-Violence. Features are extracted from intermediate layers of a frozen Qwen3.5.
(a)~Original features in Euclidean space concentrate within a ${\sim}13^{\circ}$ cone.
(b)~After $\ell_2$-normalisation, features remain in a small spherical cap.
(c)~After spherical centering, features unfold over the full hypersphere with substantially improved class separation.}\label{fig:motivation}
\end{figure}

Based on the observations above, we propose SphereVAD, a fully
training-free video anomaly detection framework that recasts anomaly
discrimination as geodesic inference on the unit hypersphere
$\mathcal{S}^{D-1}$.
SphereVAD comprises three complementary components.
\textbf{1.~Unified Fr\'{e}chet Mean Spherical Centering}, which aligns
features from different data sources to a common tangent-space
coordinate system while simultaneously unfolding the conical distribution
over the entire sphere.
\textbf{2.~Holistic Scene Attention (HSA)}, which constructs cross-video
attention maps based on scene-similarity priors to reinforce feature
consistency.
\textbf{3.~vMF-guided Spherical Geodesic Pulling (SGP)}, which moves
ambiguous segments along geodesics toward the dominant vMF prototype
direction, with pull strength that adapts to the degree of uncertainty,
ensuring that features always remain on the spherical manifold.
The entire pipeline requires only approximately 2{,}000 synthetic
calibration images, involves no gradient descent, and contains only four
core hyperparameters.

Our contributions are summarised as follows:
\begin{itemize}
  \item We propose SphereVAD, the first training-free framework that models video anomaly detection as vMF likelihood-ratio geodesic
        inference on the unit hypersphere, offering genuine plug-and-play
        capability.
  \item We propose a unified Fr\'{e}chet mean spherical centering
        strategy that effectively alleviates the rotational bias across
        different data sources. We further prove that this operation does
        not reduce, and typically enlarges, inter-class angular distances
        (Proposition~\ref{prop:centering}).
  \item We design HSA, a cross-video holistic scene attention module
        that strengthens feature consistency across videos, together with
        SGP, a score-adaptive spherical geodesic pulling mechanism that
        corrects the direction of ambiguous features under manifold
        constraints with pull strength adapting inversely to confidence.
  \item SphereVAD establishes new state-of-the-art results among zero-shot, training-free approaches on XD-Violence (AP 86.99\%, +16.83\,pp over PANDA), UCF-Crime (AUC 86.38\%, +1.49\,pp over PANDA), and UBnormal (AUC 76.46\%), substantially narrowing the gap to task-specific training paradigms and remaining competitive with weakly supervised baselines.
\end{itemize}

\section{Related Work}

\textbf{Video Anomaly Detection.}
Existing VAD methods fall into two paradigms.
Weakly supervised methods, represented by MIL~\cite{sultani2018real} and Prompt-Enhanced MIL~\cite{chen2024prompt}, train classifiers with video-level labels but are constrained by training-set coverage.
MLLM-based methods offer an alternative: zero-shot approaches such as LAVAD~\cite{zanella2024harnessing}, LAVIDA~\cite{dai2026no}, and Sherlock~\cite{ma2025sherlock} leverage large-model common sense for anomaly reasoning, while SteerVAD~\cite{cai2026steering} and HiProbe~\cite{cai2025hiprobe} adopt parameter-efficient fine-tuning with minimal data.
However, the former suffer from textual hallucinations and high inference overhead, and the latter still require gradient optimization.
SphereVAD instead performs geodesic inference directly on the unit hypersphere $\mathbb{S}^{D-1}$ of intermediate-layer MLLM features, calibrating vMF directional prototypes from sparse synthetic samples without parameter updates---bypassing the language output pathway with negligible computational overhead.

\textbf{Synthetic Data for Anomaly Calibration.}
Synthetic data have been extensively used in industrial defect detection~\cite{zavrtanik2021draem,li2021cutpaste}, whereas their application in video VAD typically demands large-scale samples coupled with model training.
Inspired by the finding of HiProbe~\cite{cai2025hiprobe} that a minimal number of samples suffices to activate the latent discriminative capacity of intermediate MLLM layers, SphereVAD calibrates vMF directional prototypes with only approximately 2{,}000 synthetic images and involves no model training whatsoever.
Because vMF parameter estimation requires only the mean direction, sample efficiency is inherently high; synthetic data thus serve as a directional calibrator in feature space rather than a conventional data augmentation resource.


\section{Method}

\subsection{Overview and Problem Definition}

\textbf{Problem definition.}
Given a set of untrimmed test videos $\{V_1,\ldots,V_N\}$, each uniformly divided into $T_i$ clips, the goal is to produce a per-frame anomaly score $s\in[0,1]$ without any gradient-based training.
The available resources comprise a pre-trained MLLM $\theta$ whose parameters are entirely frozen and a small synthetic calibration set $\mathcal{D}_{\mathrm{syn}}=\{(x_j,y_j)\}_{j=1}^{M}$, $y_j\in\{\mathrm{norm},\mathrm{abn}\}$, $M\approx 2000$.

\textbf{Core idea.}
SphereVAD builds on a key observation: intermediate-layer features of MLLMs already encode rich anomaly semantics, yet the raw features are highly concentrated within a narrow cone (Figure~\ref{fig:motivation}(b)), severely under-utilizing the angular discriminative space.
We recast anomaly detection as a vMF likelihood-ratio geodesic inference problem on the unit hypersphere $\mathcal{S}^{D-1}$, unleashing this latent discriminability through principled geometric reasoning.

As illustrated in Figure~\ref{fig:pipeline}, the pipeline comprises four stages:
\textbf{(a)}~Feature preparation---extracting dual-stream intermediate-layer features and applying spherical centering to remove domain bias;
\textbf{(b)}~vMF center construction: unified Fr\'{e}chet mean computation, logarithmic-map spherical centering, and spherical $K$-Means prototype calibration.
\textbf{(c)}~HSA---cross-video attention aggregation for feature consistency;
\textbf{(d)}~vMF spherical inference---SGP on ambiguous clips followed by final scoring.
The entire pipeline is training-free and involves only four core hyperparameters $(K_N,K_A,\alpha_G,\beta_{\mathrm{base}})$.

\begin{figure}[t]\centering
  \includegraphics[width=\linewidth]{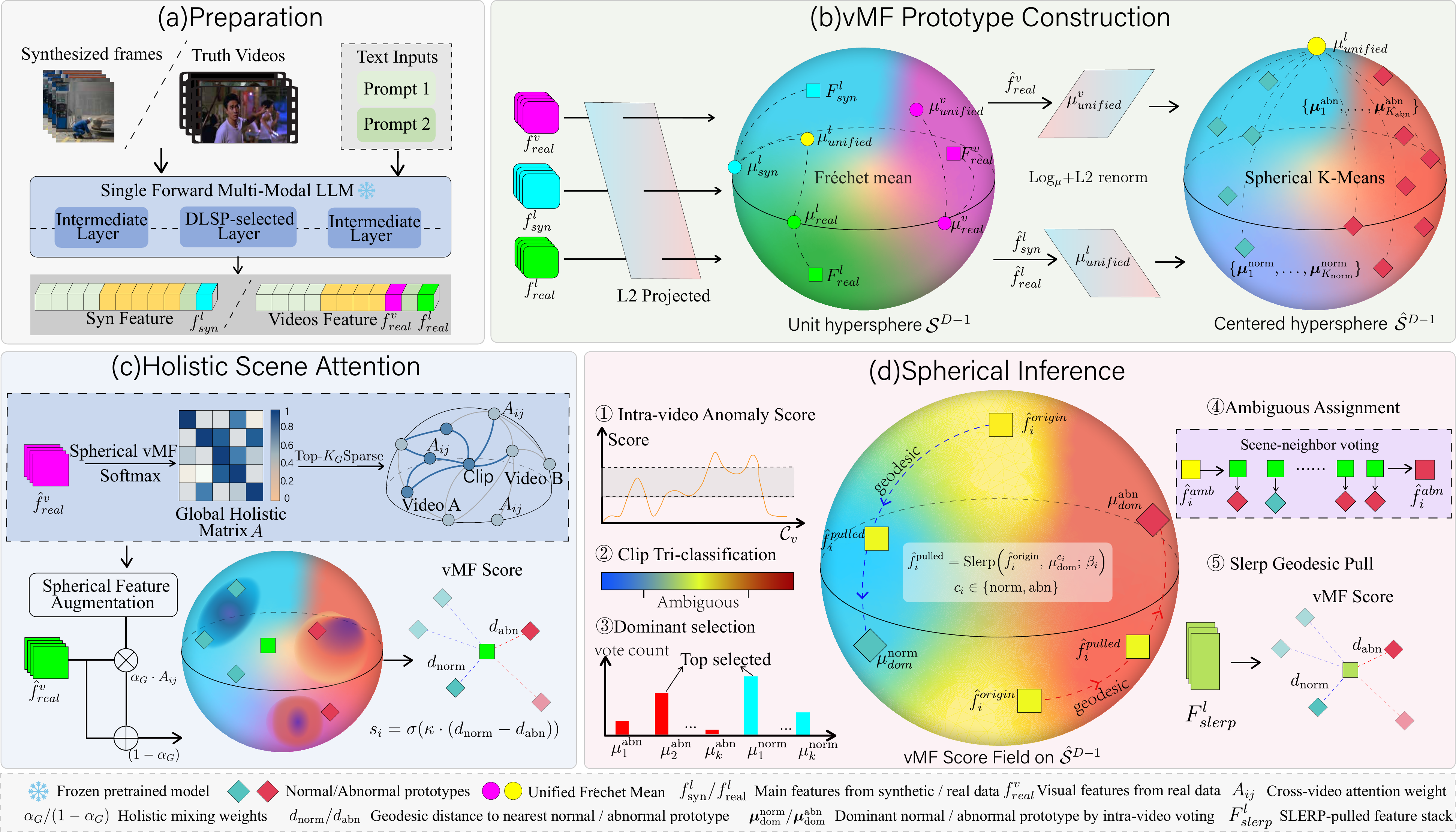}
  \caption{Overall pipeline of SphereVAD.}
  \label{fig:pipeline}
\end{figure}

\textbf{Notation.}
Raw features are denoted $f$, $\ell_2$-normalized features $\tilde{f}$, and spherically centered features $\hat{f}$.
Superscripts $l$ and $v$ indicate the main feature (last token) and the visual feature (visual last), respectively.

\subsection{Feature Extraction and Spherical Centering}

\textbf{Intermediate-layer feature extraction.}
From the frozen MLLM $\theta$ we extract two complementary representations:
the main feature $f^{l}\!\in\!\mathbb{R}^{D}$ (the hidden state of the last token at layer $\ell$, encoding global semantics) and the visual feature $f^{v}\!\in\!\mathbb{R}^{D}$ (the hidden state of the final visual-sequence token at layer $\ell'$, encoding scene appearance).
The selection of layer indices is detailed in Appendix~\ref{app:dlsp}.
All features are projected onto the unit hypersphere by $\ell_2$-normalization: $\tilde{f}=f/\|f\|_2\in\mathcal{S}^{D-1}$.

\textbf{Unified Fr\'{e}chet mean and spherical centering.}
As discussed in Section~\ref{sec:intro}, features from different data sources exhibit a systematic rotational bias on the hypersphere due to input-domain discrepancies (see Appendix~\ref{app:domain_shift} for empirical evidence).
Meanwhile, the raw features cluster within a narrow cone (Figure~\ref{fig:motivation}(b)), under-utilizing the angular discriminative capacity.
Spherical centering addresses both issues at once.

We pool the main features from both domains to compute a unified Fr\'{e}chet mean (details in Appendix~\ref{app:karcher}):
\begin{equation}
  \boldsymbol{\mu}_{\mathrm{unified}}
  = \mathrm{Fr\acute{e}chetMean}\!\bigl(\{\tilde{f}_{\mathrm{syn}}^{l}\}\cup\{\tilde{f}_{\mathrm{real}}^{l}\}\bigr).\label{eq:frechet}
\end{equation}
Spherical centering is then performed with $\boldsymbol{\mu}_{\mathrm{unified}}$ as the base point:
\begin{equation}
  \hat{f}
  = \frac{\mathrm{Log}_{\boldsymbol{\mu}_{\mathrm{unified}}}(\tilde{f})}{\bigl\|\mathrm{Log}_{\boldsymbol{\mu}_{\mathrm{unified}}}(\tilde{f})\bigr\|_2}\;\in\;\mathcal{S}^{D-1}.\label{eq:centering}
\end{equation}
The logarithmic map~\cite{absil2008optimization} removes the mean-direction bias, and the subsequent $\ell_2$-normalization projects the result back onto the sphere. This procedure is equivalent to spherical zero-centering, which unfolds the narrow-cone distribution over the full sphere and substantially expands the usable angular discrimination range.
Visual features are independently centered with their own Fr\'{e}chet mean $\boldsymbol{\mu}_{\mathrm{unified}}^{v}$ (computed only from the test set, as detailed in Appendix~\ref{app:dual_feature}).

A natural question is whether this centering operation might inadvertently reduce the angular separation between normal and anomalous features. Proposition~\ref{prop:centering} shows that this is not the case.

\begin{proposition}[Spherical centering does not reduce inter-class angular distance]
\label{prop:centering}
Let $\mathcal{X}=\mathcal{X}_0\cup\mathcal{X}_1\subset\mathcal{S}^{D-1}$, where each class follows $\mathrm{vMF}(\boldsymbol{\mu}_c,\kappa_c)$.
Denote by $\phi$ the centering map of Eq.~\eqref{eq:centering}.
Under this idealized two-class vMF model, $d_{\mathrm{geo}}\!\bigl(\phi(\boldsymbol{\mu}_0),\phi(\boldsymbol{\mu}_1)\bigr)\geq d_{\mathrm{geo}}(\boldsymbol{\mu}_0,\boldsymbol{\mu}_1)$, with equality only in the degenerate case $d_{\mathrm{geo}}(\boldsymbol{\mu}_0,\boldsymbol{\mu}_1)=\pi$.
In practice, the bound holds approximately for finite samples when the empirical Fr\'{e}chet mean is close to the population mean (Remark~2 in Appendix~\ref{app:proof_centering}).
\emph{Proof in Appendix~\ref{app:proof_centering}.}
\end{proposition}

Furthermore, the choice of a unified Fr\'{e}chet mean computed from both domains is not arbitrary. Proposition~\ref{prop:alignment} guarantees that this mean naturally absorbs the rotational bias between domains, serving as an optimal shared reference point.

\begin{proposition}[Unified Fr\'{e}chet mean as rotational alignment]
\label{prop:alignment}
If the two domain distributions differ only by a rotation $R\!\in\!\mathrm{SO}(D)$ and share equal concentration and sample size, the unified Fr\'{e}chet mean coincides with the geodesic midpoint of the domain-specific means. In the asymmetric case it still lies on the connecting geodesic, providing a reference that absorbs rotational bias from both sides.
\emph{Proof in Appendix~\ref{app:proof_alignment}.}
\end{proposition}

\subsection{vMF Directional Prototype Construction and Scoring}
\label{sec:vmf_proto}

With the centered synthetic calibration features in hand, we construct directional prototypes to serve as reference anchors on the hypersphere.
Specifically, we perform spherical $K$-Means~\cite{dhillon2001concept} clustering separately for the normal and anomalous classes, using cosine similarity for assignment and intra-cluster mean followed by $\ell_2$-normalization for update.
This yields $K_N$ normal and $K_A$ anomalous prototypes:
\begin{equation}
  \{\boldsymbol{\mu}_k^{\mathrm{norm}}\}_{k=1}^{K_N},\quad
  \{\boldsymbol{\mu}_k^{\mathrm{abn}}\}_{k=1}^{K_A}
  \;\subset\;\mathcal{S}^{D-1}.\label{eq:prototypes}
\end{equation}
These prototypes serve as the directional parameters of vMF distributions $\mathrm{vMF}(\mathbf{x};\boldsymbol{\mu},\kappa)\propto\exp(\kappa\,\boldsymbol{\mu}^{\!\top}\mathbf{x})$~\cite{banerjee2005clustering}.
Crucially, the synthetic data do not train a model but merely calibrate directional references on the sphere.
Because vMF parameter estimation requires only the mean direction, the procedure is highly sample-efficient.

\textbf{vMF likelihood-ratio scoring.}
For a clip feature $\hat{f}$, we compute the geodesic distance to the nearest prototype of each class, $d_c(\hat{f})=\min_k\arccos(\hat{f}^{\!\top}\boldsymbol{\mu}_k^{c})$, and define the anomaly score:
\begin{equation}
  s
  = \sigma\!\bigl(\kappa\,(d_{\mathrm{norm}}-d_{\mathrm{abn}})\bigr)
  = \frac{\exp(-\kappa\,d_{\mathrm{abn}})}{\exp(-\kappa\,d_{\mathrm{norm}})+\exp(-\kappa\,d_{\mathrm{abn}})}.\label{eq:vmf_score}
\end{equation}
A clip geodesically closer to an anomalous prototype receives $s>0.5$ and is deemed anomalous, while $\kappa$ controls the decision sharpness.

\subsection{Cross-Video Holistic Scene Attention}
\label{sec:holistic}

The prototypes constructed in Section~\ref{sec:vmf_proto} provide reliable scoring for clips whose features are well separated from the decision boundary.
However, clips from visually similar scenes may receive inconsistent scores due to isolated noise.
HSA addresses this by leveraging the visual features $\hat{f}^{v}$ to capture cross-video scene correlations and reinforce the consistency of the main features $\hat{f}^{l}$.

A sparse attention matrix $\mathbf{A}\!\in\!\mathbb{R}^{N_{\mathrm{total}}\times N_{\mathrm{total}}}$ is constructed from the pairwise cosine similarities of all test-clip visual features, with threshold sparsification, top-$K_G$ truncation, and row-wise softmax normalization.
The enhanced main feature is:
\begin{equation}
  \hat{f}_{i}^{\,l,\mathrm{enh}}
  = (1-\alpha_G)\,\hat{f}_{i}^{l}
  + \alpha_G\!\sum_{j}A_{ij}\,\hat{f}_{j}^{l}.\label{eq:holistic}
\end{equation}
The result is $\ell_2$-normalized back onto the sphere.
This aggregation enables clips from visually similar scenes to reinforce each other across videos, suppressing isolated noise and promoting consistent anomaly detection.
Initial vMF scores $\{s_i^{\mathrm{init}}\}$ are then obtained by applying Eq.~\eqref{eq:vmf_score} to the enhanced features.

\subsection{vMF-Guided Spherical Geodesic Pulling}
\label{sec:slerp}

After initial scoring via Eq.~\eqref{eq:vmf_score}, a subset of clips remains in the ambiguous zone near the decision boundary.
SGP proposes a within-video, manifold-constrained feature correction mechanism that leverages vMF prototype priors to adaptively adjust the feature directions of these ambiguous clips.
The procedure is applied independently to each video and comprises the following steps (full algorithm and design rationale in Appendix~\ref{app:sgp}).

\textbf{Adaptive tri-classification.}
Based on the initial scores $\{s_i^{\mathrm{init}}\}$ obtained in Section~\ref{sec:holistic}, an ambiguity interval $[\rho_{\mathrm{low}},\rho_{\mathrm{high}}]$ is adaptively determined so that clips within each video are partitioned into three groups: $\{\mathrm{norm},\mathrm{amb},\mathrm{abn}\}$.
The interval width adapts to the score distribution, narrowing when scores are clearly bimodal and widening when they concentrate near $0.5$ (details in Appendix~\ref{app:mad}).

\textbf{Dominant prototype selection.}
Within each video, the clearly normal and clearly anomalous clips vote for their respective dominant prototypes from the prototype sets defined in Eq.~\eqref{eq:prototypes}.
These dominant prototypes represent the prevailing normal and anomalous directions for the current video.
If the number of anomalous clips falls below a minimum threshold, the video is treated as entirely normal (Appendix~\ref{app:dominant}).

\textbf{Ambiguous-clip assignment.}
Rather than relying on its own unreliable score, each ambiguous clip is assigned a label via a sparse intra-video attention matrix $\mathbf{A}_H$ built from visual features.
The attention-weighted aggregation of neighbouring clip features is compared with the dominant prototypes by cosine similarity, so that neighbour consensus resolves the ambiguity (Appendix~\ref{app:neighbor}).

\textbf{Score-adaptive geodesic pull.}
Once assigned, the feature of each ambiguous clip is moved along the geodesic toward its target prototype via Spherical Linear Interpolation (SLERP)~\cite{shoemake1985animating}:
\begin{equation}
  \hat{f}_i^{\,\mathrm{pulled}}
  = \mathrm{Slerp}\!\bigl(\hat{f}_i,\;\boldsymbol{\mu}_{\mathrm{dom}}^{c_i},\;\beta_i\bigr).\label{eq:slerp}
\end{equation}
The pull strength $\beta_i=\beta_{\mathrm{base}}\times(1-\tfrac{1}{2}\hat{d}_i)$ scales monotonically with prediction uncertainty: a score closer to $0.5$ ($\hat{d}_i \to 0$) yields the maximum pull $\beta_{\mathrm{base}}$, whereas clips with more confident scores receive a linearly decayed correction, preventing over-adjustment.
SLERP traverses the great-circle arc at constant angular velocity, ensuring that the result remains strictly on $\mathcal{S}^{D-1}$ and thereby achieving manifold-constrained feature correction without distribution drift (Appendix~\ref{app:slerp_derivation}).

The corrected features are scored via Eq.~\eqref{eq:vmf_score} to produce the final clip-level anomaly scores, which are expanded to frame level and temporally smoothed with a Gaussian kernel~\cite{sultani2018real}.

\section{Experiments}

We validate SphereVAD on three large-scale video anomaly detection benchmarks, conducting SOTA comparison~(\S\ref{sec:sota}), module ablation~(\S\ref{sec:ablation}), analysis of spherical geometry advantages~(\S\ref{sec:geometry}), backbone generality verification~(\S\ref{sec:backbone}), and qualitative analysis~(\S\ref{sec:qualitative}).

\textbf{Datasets and Metrics.}
(1)~\textbf{XD-Violence}~\cite{wu2020not} comprises 4,754 untrimmed videos covering 6 violent anomaly categories; we report frame-level AP (primary) and AUC.
(2)~\textbf{UCF-Crime}~\cite{sultani2018real} consists of 1,900 surveillance videos spanning 13 crime categories; frame-level AUC is the standard metric.
(3)~\textbf{UBnormal}~\cite{acsintoae2022ubnormal} is a synthetic open-set benchmark with pixel-level annotations across 29 anomaly types; we report frame-level AUC.
Following common practice, all metrics are computed by concatenating frames from all test videos.

\textbf{Implementation Details.}
We employ a frozen Qwen3.5~\cite{qwen35blog} as the default backbone.
The feature extraction layer is selected as layer 31 via the procedure described in Appendix~\ref{app:dlsp}.
We synthesize approximately 2{,}000 calibration images (${\sim}$1{,}000 normal/anomalous pairs) using only binary labels; no anomaly category information is leaked and there is no content overlap with any test set (details in Appendix~\ref{app:synth}).
All experiments run on 4$\times$A100 GPUs at inference time with no gradient descent.
Beyond offline batch processing, SphereVAD also supports an \textbf{online streaming mode} in which each incoming clip is scored immediately after feature extraction and spherical centering using a variant of M1 with synthetic-only centering (without HSA, SGP, or transductive real-data statistics). In this mode, the per-clip geometric inference is negligible ($<$0.1\,ms), so throughput is dominated by the single MLLM forward pass, achieving \textbf{59.3\,FPS} on a single A100 GPU (measured at ${\sim}$67.5\,ms per 4-frame clip forward pass)---comfortably exceeding real-time requirements. A detailed computational cost analysis covering both deployment modes is provided in Appendix~\ref{app:compute}.
Full hyperparameter settings and sensitivity analyses are in Appendix~\ref{app:hyper}.
\subsection{Comparison with State-of-the-Art Methods}
\label{sec:sota}
\begin{table}[t]
\caption{\textbf{VAD performance comparison.}
\cmark/\xmark\ indicate zero-shot and training-free status. Best training-free results are in \textbf{bold}.}
\label{tab:main_results}
\centering
\small
\setlength{\tabcolsep}{4pt}
\begin{tabular}{@{}l|c|c|c|c|c|cc@{}}
\toprule
\multirow{2}{*}{\textbf{Method}} & \multirow{2}{*}{\textbf{Venue}} & \multirow{2}{*}{\textbf{Zero-shot}} & \multirow{2}{*}{\textbf{Train-free}} & \textbf{UCF-Crime} & \textbf{UBnormal} & \multicolumn{2}{c}{\textbf{XD-Violence}}\\
& & & & \textbf{AUC(\%)} & \textbf{AUC(\%)} & \textbf{AUC(\%)} & \textbf{AP(\%)} \\
\midrule
Sultani et al.~\cite{sultani2018real} & CVPR'18    & \xmark & \xmark & 77.92 & 50.30 & --    & 73.20\\
RTFM~\cite{tian2021weakly}            & ICCV'21    & \xmark & \xmark & 83.31 & 64.94 & --    & 77.81\\
UR-DMU~\cite{zhou2023dual}            & AAAI'23    & \xmark & \xmark & 86.97 & --    & 94.02 & 81.66\\
CLIP-TSA~\cite{joo2023clip}           & ICIP'23    & \xmark & \xmark & 87.58 & --    & --    & 82.19\\
MGFN~\cite{chen2023mgfn}              & AAAI'23    & \xmark & \xmark & 86.98 & --    & --    & 80.11\\
VadCLIP~\cite{wu2024vadclip}          & AAAI'24    & \xmark & \xmark & 88.02 & --    & --    & 84.51\\
OVVAD~\cite{wu2024open}               & CVPR'24    & \xmark & \xmark & 86.40 & 62.94 & --    & 66.53\\
STPrompt~\cite{wu2024weakly}          & ACM MM'24  & \xmark & \xmark & 88.08 & 63.98 & --    & --\\
TPWNG~\cite{yang2024text}             & CVPR'24    & \xmark & \xmark & 87.79 & --    & --    & 83.68\\
Holmes-VAU~\cite{zhang2025holmes}     & CVPR'25    & \xmark & \xmark & 88.96 & 56.77 & --    & 87.68\\
$\pi$-VAD~\cite{majhi2025just}        & CVPR'25    & \xmark & \xmark & \textbf{90.33} & -- & --   & 85.37\\
RefineVAD~\cite{lee2026refinevad}     & AAAI'26    & \xmark & \xmark & 88.92 & --    & --    & 88.66\\
VERA~\cite{ye2025vera}                & CVPR'25    & \xmark & \cmark & 86.55 & --    & 88.26 & 70.54\\
LAVIDA~\cite{dai2026no}               & CVPR'26   & \cmark & \xmark & 82.18 & \textbf{76.45} & -- & \textbf{90.62}\\
\midrule
CLIP~\cite{radford2021learning}       & ICML'21    & \cmark & \cmark & 53.16 & --    & 38.21 & 17.83\\
LLaVA-1.5~\cite{liu2024improved}      & CVPR'24    & \cmark & \cmark & 72.84 & --    & 79.62 & 50.26\\
LAVAD~\cite{zanella2024harnessing}    & CVPR'24    & \cmark & \cmark & 80.28 & 51.06 & 85.36 & 62.01\\
MCANet~\cite{dev2024mcanet}           & ICPR'24    & \cmark & \cmark & 82.47 & 62.94 & 87.43 & 69.72\\
GLM-4.1V-9B~\cite{hong2025glm}        & ArXiv'25   & \cmark & \cmark & 61.80 & 60.81 & 72.73 & 52.93\\
EventVAD~\cite{shao2025eventvad}      & ACM MM'25  & \cmark & \cmark & 82.03 & --    & 87.51 & 64.04\\
Unified-VAD~\cite{lin2025unified}     & NeurIPS'25   & \cmark & \cmark & 84.28 & 69.02 & 91.34 & 68.07\\
VADTree~\cite{li2025vadtree}          & NeurIPS'25   & \cmark & \cmark & 84.74 & 65.80 & 90.55 & 68.85\\
PANDA~\cite{yang2025panda}            & NeurIPS'25   & \cmark & \cmark & 84.89 & 75.78 & --    & 70.16\\
\textbf{SphereVAD (Online)}           & NeurIPS'26 & \cmark & \cmark & 78.31 & 71.62 & 92.48 & 81.37\\
\textbf{SphereVAD (Offline)}          & NeurIPS'26 & \cmark & \cmark & \textbf{86.38} & \textbf{76.46} & \textbf{95.74} & \textbf{86.99}\\
\bottomrule
\end{tabular}
\end{table}

As shown in Table~\ref{tab:main_results}, SphereVAD establishes state-of-the-art performance among zero-shot, training-free methods across all three benchmarks.
On XD-Violence, it surpasses the previous best zero-shot training-free method (PANDA, AP 70.16\%) by \textbf{+16.83\,pp AP} (and exceeds the weakly-supervised training-free method VERA, 70.54\%, by \textbf{+16.45\,pp AP}), even outperforming established supervised methods such as VadCLIP (84.51\%) and $\pi$-VAD (85.37\%). While the task-specific fine-tuned zero-shot method LAVIDA achieves a higher AP (90.62\%) on XD-Violence by updating model weights, SphereVAD remains strictly zero-parameter-update and tuning-free. 
On UCF-Crime, SphereVAD achieves \textbf{86.38\% AUC}, advancing the zero-shot training-free ceiling (+1.49\,pp over PANDA, 84.89\%) and remaining highly competitive with VERA (86.55\%, which utilizes weak training-set labels despite being training-free), while narrowing the gap with the top supervised method ($\pi$-VAD, 90.33\%) to within 4 points.
On UBnormal, SphereVAD attains \textbf{76.46\% AUC}, outperforming PANDA (75.78\%) and matching the fine-tuned LAVIDA (76.45\%) without any task-specific gradient descent, while substantially surpassing supervised baselines (e.g., RTFM 64.94\%, STPrompt 63.98\%).
The XD-Violence gains are particularly striking because this dataset contains diverse violence types that benefit from the cross-video HSA mechanism.

\subsection{Ablation Studies}
\label{sec:ablation}

\begin{table}[t]
\caption{\textbf{Component-wise ablation.} Each row adds one module to its predecessor.
``Eucl.''\ denotes Euclidean distance-ratio scoring; ``vMF''\ denotes vMF
log-likelihood-ratio scoring. $\Delta$ is the improvement of vMF over Euclidean scoring
at each stage. Primary metrics are \underline{underlined}.
UBnormal does not use SGP, so M2 and M3 share the same results.}
\label{tab:ablation_module}
\centering
\small
\begin{tabular}{@{}cl|ccc|ccc|ccc@{}}
\toprule
& & \multicolumn{3}{c|}{\textbf{XD-Violence \underline{AP}\,(\%)}}
& \multicolumn{3}{c|}{\textbf{UCF-Crime \underline{AUC}\,(\%)}}
& \multicolumn{3}{c}{\textbf{UBnormal \underline{AUC}\,(\%)}} \\
& \textbf{Configuration}
& Eucl. & vMF & $\Delta$
& Eucl. & vMF & $\Delta$
& Eucl. & vMF & $\Delta$ \\
\midrule
M0a & Raw features
& 80.19 & -- & --
& 75.42 & -- & --
& 71.13 & -- & -- \\
M0b & + Spherical center
& 82.03 & -- & --
& 81.20 & -- & --
& 75.49 & -- & -- \\
\midrule
M1  & + vMF scoring
& 82.03 & 82.51 & +0.48
& 81.20 & 81.68 & +0.48
& 75.49 & 75.98 & +0.49 \\
M2  & + HSA
& 85.01 & 85.16 & +0.15
& 82.71 & 83.06 & +0.35
& 76.01 & 76.46 & +0.45 \\
M3  & + SGP
& 85.60 & \textbf{86.99} & +1.39
& 84.80 & \textbf{86.38} & +1.58
& 76.01$^\dagger$ & \textbf{76.46}$^\dagger$ & +0.45 \\
\bottomrule
\multicolumn{11}{@{}l}{\footnotesize $^\dagger$ UBnormal videos average ${\sim}$12\,seconds, too short for intra-video SGP; M3 results are identical to M2.}
\end{tabular}
\end{table}

Table~\ref{tab:ablation_module} presents the incremental contribution of each module across all three benchmarks.
Spherical centering accounts for the largest single gain on UCF-Crime and UBnormal.
On UCF-Crime it raises AUC from 75.42\% to 81.20\% ($+5.78$\,pp), on UBnormal from 71.13\% to 75.49\% ($+4.36$\,pp), and on XD-Violence from 80.19\% to 82.03\% ($+1.84$\,pp AP), confirming the necessity of unfolding narrow-cone features onto the full sphere.
Adding HSA (M2) brings $+2.65$\,pp AP on XD-Violence (representing the largest single gain on this dataset) and $+0.48$\,pp AUC on UBnormal, confirming the value of cross-video scene-consistency priors.
SGP (M3) further yields $+3.32$\,pp AUC on UCF-Crime and $+1.83$\,pp AP on XD-Violence, demonstrating that directional correction of ambiguous transitional segments is critical for long surveillance videos.
Note that UBnormal does not employ SGP because its test videos average only ${\sim}$12\,seconds each, providing insufficient temporal context for the intra-video tri-classification and neighbour-consensus mechanisms that SGP relies on. Consequently M2 and M3 share identical results.
Cumulatively, M1$\to$M3 improves by $+4.48$\,pp AP (XD-Violence) and $+4.70$\,pp AUC (UCF-Crime).

\subsection{Advantages of Spherical Geometry}
\label{sec:geometry}

\subsubsection{vMF Scoring vs. Distance-Ratio Scoring}
Across all three benchmarks and all pipeline stages, vMF log-likelihood-ratio scoring consistently outperforms Euclidean distance-ratio scoring (Table~\ref{tab:ablation_module}, $\Delta$ columns). In stage M1, the calibrated concentration $\kappa$ acts as an optimal temperature parameter that bounds anomalous extreme values prior to temporal Gaussian smoothing, producing a consistent $+0.48\%\sim+0.49\%$ gain over raw distance ratios. The advantage is amplified when moving to stage M3 ($+1.39$\,pp AP on XD-Violence and $+1.58$\,pp AUC on UCF-Crime), demonstrating that directional geodesic pulling and the bounded probability metric synergistically refine clips near the decision boundary.

\subsubsection{Visualization of Feature Discriminative Structure}

\begin{figure}[t]\centering
  \includegraphics[width=\linewidth]{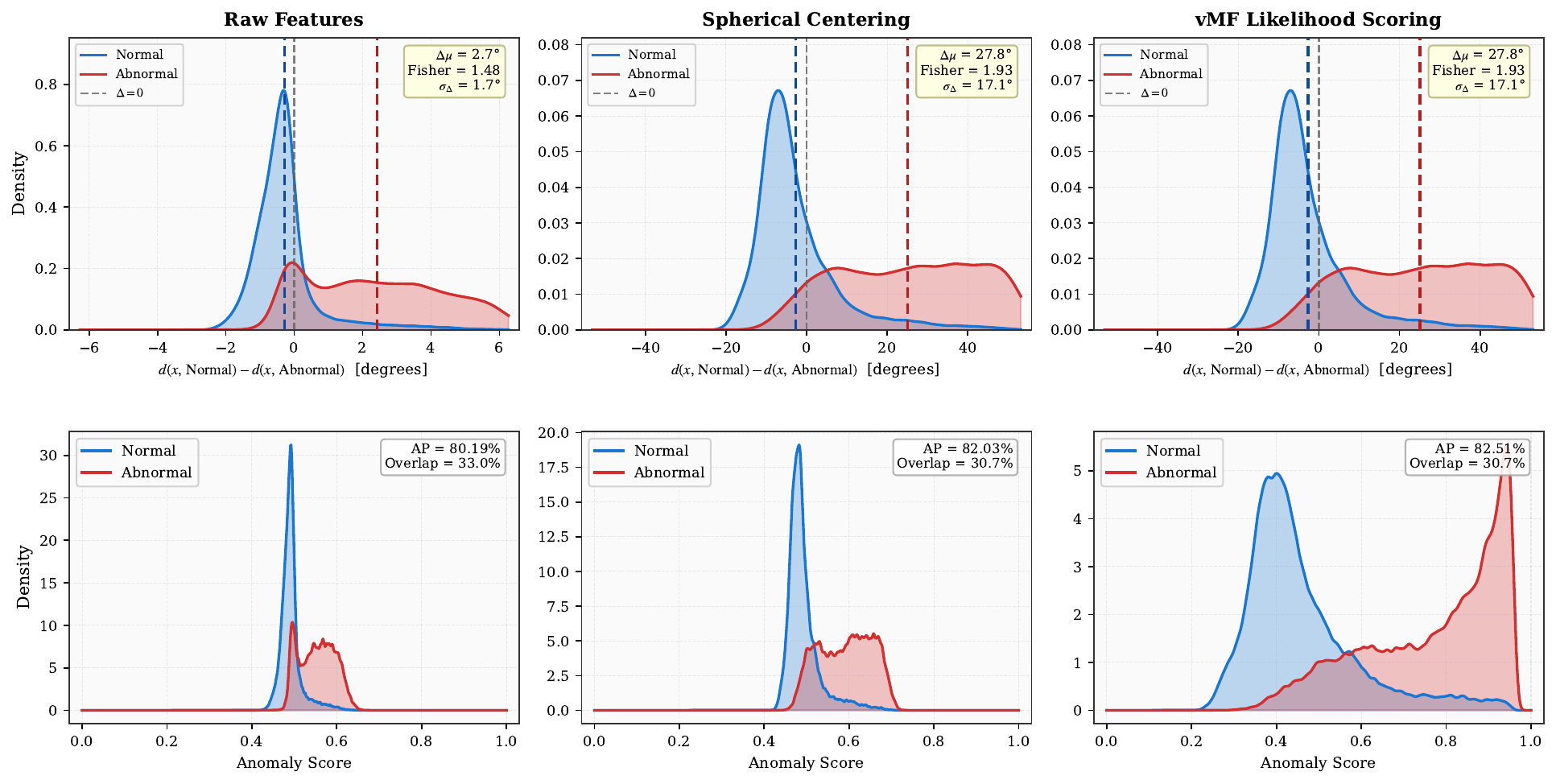}
  \caption{Progressive improvement of feature discriminative structure.
  \textbf{Top:} geodesic distance difference $\Delta = d(\mathbf{x}, \text{Normal}) - d(\mathbf{x}, \text{Abnormal})$.
  \textbf{Bottom:} anomaly score distributions.
  Columns from left to right: Raw / Spherical Centering / vMF Scoring.}
  \label{fig:discriminative_structure}
\end{figure}

Figure~\ref{fig:discriminative_structure} visualizes how each stage reshapes the feature discriminative structure.
Raw features exhibit near-complete overlap between normal and anomalous distance-difference distributions ($\Delta\mu{=}2.7^{\circ}$, $\sigma_{\Delta}{=}1.7^{\circ}$, Fisher${=}1.48$, score overlap 33.0\%, AP${=}80.19\%$).
Spherical centering expands $\Delta\mu$ to $27.8^{\circ}$ ($\sigma_{\Delta}{=}17.1^{\circ}$, Fisher${=}1.93$) and reduces score overlap to 30.7\% (AP${=}82.03\%$), confirming its role as the single most impactful component. It unfolds features compressed within a narrow cone onto the full sphere, recovering wasted angular discriminative capacity.
vMF scoring preserves the same geodesic separation ($\Delta\mu{=}27.8^{\circ}$, Fisher${=}1.93$) while nonlinearly reshaping the score distribution through the vMF log-likelihood ratio, further lifting AP to 82.51\%, consistent with the ablation results in Table~\ref{tab:ablation_module}.

\subsection{Backbone Generality}
\label{sec:backbone}
\begin{table}[t]
\caption{\textbf{Backbone generality.} SphereVAD with four different MLLM backbones. The first row uses the same Qwen3.5 backbone but directly prompts its language head to rate anomaly severity, without our projection framework.}
\label{tab:ablation_backbone}
\centering
\small
\begin{tabular}{@{}l|c|c|c@{}}
\toprule
\textbf{Backbone / Method} & \textbf{XD-Violence AP(\%)} & \textbf{UCF-crime AUC(\%)} & \textbf{UBnormal AUC(\%)} \\
\midrule
Qwen3.5 direct inference$^\dagger$ & 58.70 & 72.64 & 55.74\\
\midrule
Llava-ov-1.5~\cite{an2025llavaonevision15fullyopenframework}   & 76.93 & 79.36 & 69.83\\
InternVL3~\cite{zhu2025internvl3} & 80.41 & 84.76 & 72.39\\
Qwen3-VL~\cite{bai2025qwen3}& 82.81 & 83.42 & 74.51 \\
Qwen3.5~\cite{qwen35blog}& \textbf{86.99}~{\scriptsize(+28.29)} & \textbf{86.38}~{\scriptsize(+13.74)} & \textbf{76.46}~{\scriptsize(+20.72)} \\
\bottomrule
\end{tabular}
\vspace{2pt}
\par\raggedright\scriptsize $^\dagger$ Language-only baseline: prompting Qwen3.5's language head to score anomaly severity without SphereVAD.
\end{table}

As shown in Table~\ref{tab:ablation_backbone}, SphereVAD produces effective results across all four MLLM backbones, demonstrating that the spherical geometric inference framework is backbone-agnostic.
Even the weakest backbone (LLaVA-OV-1.5) achieves 76.93\% AP on XD,79.36\% AUC on UCF, and 69.83\% on UBnormal, while the best backbone (Qwen3.5) reaches 86.99\%, 86.38\%, and 76.46\%, respectively.
Comparing with the LLM direct inference baseline (58.70\% / 72.64\% / 55.74\%), SphereVAD with the same Qwen3.5 backbone yields improvements of +28.29\%, +13.74\%, and +20.72\% on the three benchmarks.
This reveals the bottleneck of the language decoding head in fine-grained anomaly understanding---what matters is not what the model ``says,'' but what the model ``sees.''

\subsection{Qualitative Analysis}
\label{sec:qualitative}

\begin{figure}[t]\centering
  \includegraphics[width=\linewidth]{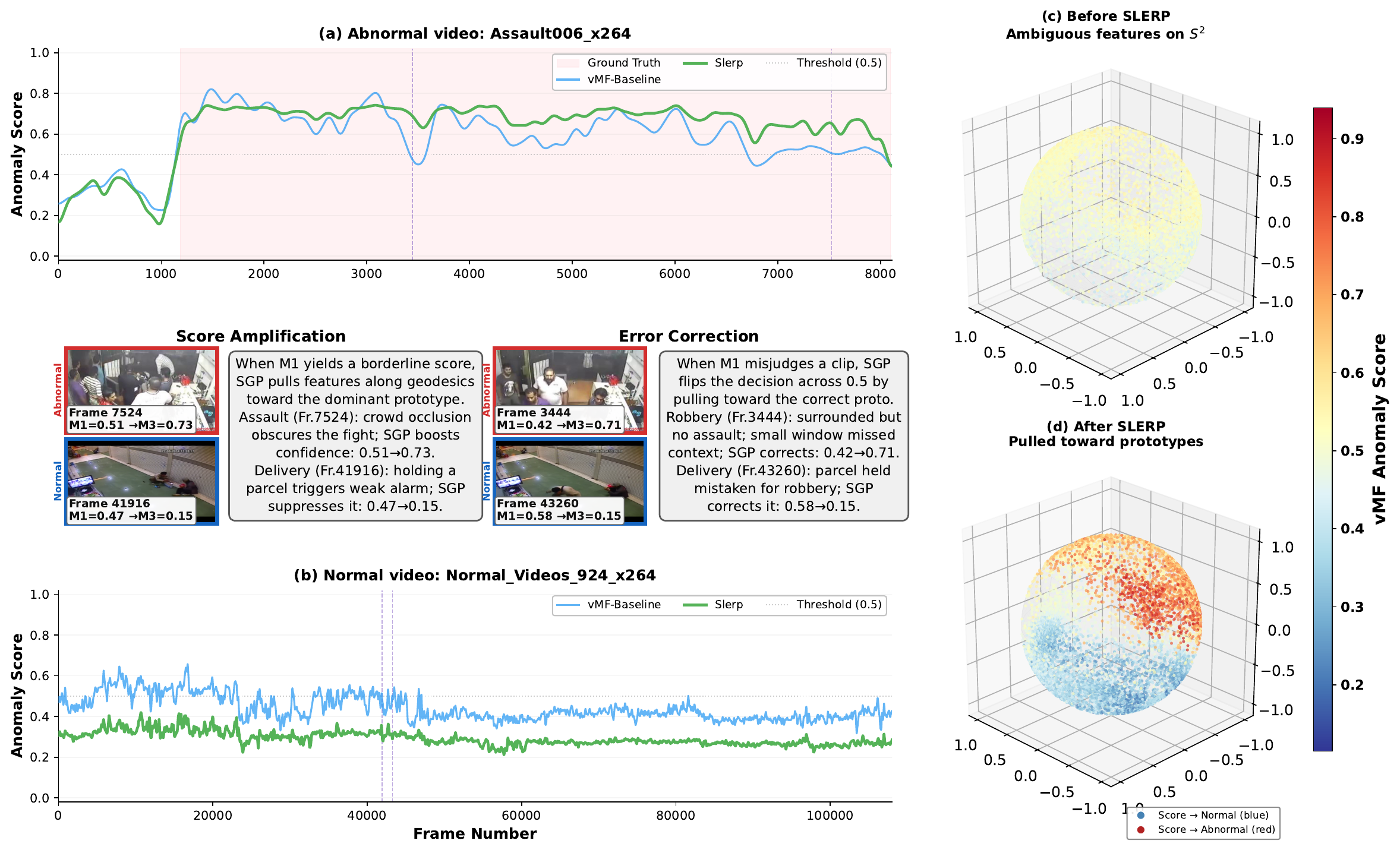}
  \caption{Qualitative results of SGP.
  \textbf{Left (a)(b):} Frame-level score curves (vMF-Baseline\,/\,SGP vs.\ GT) for two videos.
  \textbf{Right (c)(d):} Spherical visualization before vs.\ after pulling.}
  \label{fig:qualitative_slerp}
\end{figure}

Figure~\ref{fig:qualitative_slerp} illustrates both the temporal and geometric effects of SGP.
For the anomalous video (a), the vMF-Baseline (M1) captures the main anomalous intervals but fluctuates severely; SGP (M3) corrects boundary segments, aligning the curve closely with the ground truth (e.g., Frame\,7524: $0.51{\to}0.73$, and Frame\,3444: $0.42{\to}0.71$).
For the normal video (b), SGP suppresses M1's false positives (e.g., Frame\,41916: $0.47{\to}0.15$ and Frame\,43260: $0.58{\to}0.15$) by pulling ambiguous representations toward the dominant normal prototype.
The spherical views (c--d) confirm the mechanism: before pulling, ambiguous features scatter between prototypes with mixed labels; after pulling, they migrate along geodesics toward assigned prototypes, yielding clearly separated normal (blue) and anomalous (red) clusters.

\section{Conclusion}

We propose SphereVAD, a training-free video anomaly detection framework that recasts anomaly discrimination as vMF likelihood-ratio geodesic inference on the unit hypersphere. SphereVAD comprises three complementary components: unified Fr\'{e}chet mean spherical centering, cross-video holistic scene attention, and manifold-constrained spherical geodesic pulling, all operating without gradient descent. Using only ${\sim}$2{,}000 synthetic calibration images, SphereVAD achieves state-of-the-art results among zero-shot, training-free approaches on XD-Violence, UCF-Crime, and UBnormal, while remaining competitive with weakly supervised methods.

{
\small
\bibliographystyle{plain}
\bibliography{references}
}

%
%
\newpage
\appendix
\section*{Appendix}
\addcontentsline{toc}{section}{Appendix}
\noindent
The appendix is organised as follows.
Appendix~\ref{app:theory} provides theoretical proofs.
Appendix~\ref{app:synth} describes the synthetic calibration dataset.
Appendix~\ref{app:prompt} details prompt design and feature extraction.
Appendix~\ref{app:sgp} presents the full SGP algorithm.
Appendix~\ref{app:compute} reports computational cost analysis.
Appendix~\ref{app:add_vis} shows additional visualizations.
Appendix~\ref{app:hyper} contains hyperparameter settings and sensitivity analyses.
Appendix~\ref{app:repro} provides the reproducibility statement and symbol table.
Appendix~\ref{app:limitations} discusses limitations.
Appendix~\ref{app:impact} addresses broader impact.
\section{Theoretical Analysis}
\label{app:theory}

This section provides the complete mathematical machinery underlying the spherical centering framework of SphereVAD.
Section~\ref{app:karcher} details the Karcher iteration used to compute the unified Fr\'{e}chet mean (Eq.~\eqref{eq:frechet}), including convergence guarantees and step-size justification.
Section~\ref{app:proof_centering} proves Proposition~\ref{prop:centering} (spherical centering expands inter-class angular distance).
Section~\ref{app:proof_alignment} proves Proposition~\ref{prop:alignment} (the unified Fr\'{e}chet mean provides symmetric rotational alignment).

\subsection{Fr\'{e}chet Mean via Karcher Iteration}
\label{app:karcher}

The Fr\'{e}chet mean on a Riemannian manifold $\mathcal{M}$ generalises the Euclidean centre of mass.
For a set of points $\{\mathbf{x}_1,\ldots,\mathbf{x}_N\}\subset\mathcal{S}^{D-1}$, it is defined as
\begin{equation}
  \boldsymbol{\mu}^{*}
  = \operatorname*{arg\,min}_{\boldsymbol{\mu}\in\mathcal{S}^{D-1}}\sum_{i=1}^{N} d_{\mathrm{geo}}^{2}(\boldsymbol{\mu},\mathbf{x}_i),\label{eq:frechet_def}
\end{equation}
where $d_{\mathrm{geo}}(\boldsymbol{\mu},\mathbf{x})=\arccos(\boldsymbol{\mu}^{\!\top}\mathbf{x})$ is the geodesic (great-circle) distance on the unit hypersphere.

\subsubsection{Logarithmic and Exponential Maps}

The logarithmic and exponential maps provide the bridge between the sphere and its tangent space, enabling Riemannian optimisation.

\paragraph{Logarithmic map.}
For a base point $\boldsymbol{\mu}\in\mathcal{S}^{D-1}$ and a target point $\mathbf{x}\in\mathcal{S}^{D-1}$ with $\mathbf{x}\neq\pm\boldsymbol{\mu}$, the logarithmic map yields the unique tangent vector $\mathbf{v}\in T_{\boldsymbol{\mu}}\mathcal{S}^{D-1}$ such that the geodesic starting at $\boldsymbol{\mu}$ with initial velocity $\mathbf{v}$ reaches $\mathbf{x}$ at unit time:
\begin{equation}
  \mathrm{Log}_{\boldsymbol{\mu}}(\mathbf{x})
  = \frac{\theta}{\sin\theta}\bigl(\mathbf{x} - \cos\theta\;\boldsymbol{\mu}\bigr),\qquad \theta = \arccos\!\bigl(\boldsymbol{\mu}^{\!\top}\mathbf{x}\bigr).\label{eq:logmap_app}
\end{equation}
Crucially, $\|\mathrm{Log}_{\boldsymbol{\mu}}(\mathbf{x})\|=\theta=d_{\mathrm{geo}}(\boldsymbol{\mu},\mathbf{x})$ and $\mathrm{Log}_{\boldsymbol{\mu}}(\mathbf{x})^{\!\top}\boldsymbol{\mu}=0$ (the result is orthogonal to $\boldsymbol{\mu}$, confirming it lies in $T_{\boldsymbol{\mu}}\mathcal{S}^{D-1}$).
When $\theta\to 0$, we recover $\mathrm{Log}_{\boldsymbol{\mu}}(\mathbf{x})\to\mathbf{x}-\boldsymbol{\mu}$ (the Euclidean difference).

\paragraph{Exponential map.}
For a tangent vector $\mathbf{v}\in T_{\boldsymbol{\mu}}\mathcal{S}^{D-1}$ (i.e., $\mathbf{v}^{\!\top}\boldsymbol{\mu}=0$), the exponential map traces the geodesic from $\boldsymbol{\mu}$ in direction $\mathbf{v}$ for unit time:
\begin{equation}
  \mathrm{Exp}_{\boldsymbol{\mu}}(\mathbf{v})
  = \cos\!\bigl(\|\mathbf{v}\|\bigr)\,\boldsymbol{\mu}
  + \sin\!\bigl(\|\mathbf{v}\|\bigr)\,\frac{\mathbf{v}}{\|\mathbf{v}\|}.\label{eq:expmap_app}
\end{equation}
One can verify that $\|\mathrm{Exp}_{\boldsymbol{\mu}}(\mathbf{v})\|=1$ and $d_{\mathrm{geo}}(\boldsymbol{\mu},\mathrm{Exp}_{\boldsymbol{\mu}}(\mathbf{v}))=\|\mathbf{v}\|$, confirming that the exponential map preserves geodesic distance in the radial direction.
When $\|\mathbf{v}\|\to 0$, we recover $\mathrm{Exp}_{\boldsymbol{\mu}}(\mathbf{v})\to\boldsymbol{\mu}+\mathbf{v}$.

\subsubsection{Karcher Iteration Algorithm}

The Karcher iteration (also called Riemannian gradient descent for the Fr\'{e}chet objective) proceeds as follows:

\begin{enumerate}[leftmargin=*]
  \item \textbf{Initialisation.} Compute the Euclidean mean and project onto the sphere:
        $\boldsymbol{\mu}^{(0)} = \bar{\mathbf{x}}/\|\bar{\mathbf{x}}\|$,
        where $\bar{\mathbf{x}}=\frac{1}{N}\sum_{i}\mathbf{x}_i$. This provides a consistent starting point in the vicinity of the true Fr\'{e}chet mean.

  \item \textbf{Gradient computation and update.}
        At iteration $t$, compute the Riemannian gradient of the Fr\'{e}chet objective:
        \begin{equation}
          \mathbf{g}^{(t)}
          = \frac{1}{N}\sum_{i=1}^{N}\mathrm{Log}_{\boldsymbol{\mu}^{(t)}}(\mathbf{x}_i).\label{eq:karcher_grad}
        \end{equation}
        This is the (negative) Riemannian gradient of $F(\boldsymbol{\mu})=\sum_i d_{\mathrm{geo}}^2(\boldsymbol{\mu},\mathbf{x}_i)$ divided by~$2N$; it
        points from $\boldsymbol{\mu}^{(t)}$ toward the tangent-space centre of mass.
        Update the estimate via the exponential map:
        \begin{equation}
          \boldsymbol{\mu}^{(t+1)} = \mathrm{Exp}_{\boldsymbol{\mu}^{(t)}}\!\bigl(\eta\,\mathbf{g}^{(t)}\bigr),\label{eq:karcher_update}
        \end{equation}
        where $\eta>0$ is the step size.

  \item \textbf{Convergence check.}
        Terminate when $\|\mathbf{g}^{(t)}\|<\epsilon$ or a maximum number of iterations is reached.
        At convergence, $\mathbf{g}^{*}\!=\!\mathbf{0}$, which is the necessary and sufficient first-order condition for the Fr\'{e}chet mean.
\end{enumerate}

\subsubsection{Convergence Guarantee and Step-Size Justification}

\begin{theorem}[Convergence of Karcher iteration{\cite{pennec2006intrinsic}}]
\label{thm:karcher_convergence}
Let $\{\mathbf{x}_1,\ldots,\mathbf{x}_N\}\subset\mathcal{S}^{D-1}$ be contained within an open geodesic ball $B(\mathbf{p},r)$ with radius $r<\pi/2$.
Then:
\begin{enumerate}[label=(\roman*)]
  \item \emph{(Existence and uniqueness.)} The Fr\'{e}chet mean $\boldsymbol{\mu}^{*}$ exists, is unique, and lies within $B(\mathbf{p},r)$.
  \item \emph{(Linear convergence.)} The Karcher iteration with step size $\eta=1$ converges to $\boldsymbol{\mu}^{*}$ at a linear rate:
        $d_{\mathrm{geo}}(\boldsymbol{\mu}^{(t+1)},\boldsymbol{\mu}^{*})\leq c\,d_{\mathrm{geo}}(\boldsymbol{\mu}^{(t)},\boldsymbol{\mu}^{*})$
        for some $c<1$ depending on the curvature and data spread.
\end{enumerate}
\end{theorem}

\paragraph{Step size $\eta\!=\!1$.}
On the unit sphere with constant positive curvature $\kappa_{\mathrm{sec}}=1$, the squared geodesic distance $d_{\mathrm{geo}}^2(\boldsymbol{\mu},\mathbf{x})$ is geodesically convex for $d_{\mathrm{geo}}(\boldsymbol{\mu},\mathbf{x})<\pi/2$~\cite{absil2008optimization}.
Hence the Fr\'{e}chet objective $F(\boldsymbol{\mu})$ is a sum of geodesically convex functions, and a full Riemannian gradient step ($\eta\!=\!1$) is guaranteed to decrease $F$ at each iteration.
Unlike Euclidean gradient descent, overshooting cannot occur because the exponential map automatically constrains the update to lie on $\mathcal{S}^{D-1}$.

\paragraph{Practical convergence in SphereVAD.}
In our setting, MLLM intermediate-layer features before centering have angular standard deviation $\approx 1.7^{\circ}$, so they are contained within a geodesic ball of radius $r\lesssim 15^{\circ}\ll90^{\circ}$.
The convergence condition $r<\pi/2$ is thus satisfied by a wide margin.
We use $T_{\max}=5$ iterations and $\epsilon=10^{-7}$; empirically, the gradient norm $\|\mathbf{g}^{(t)}\|$ drops below $\epsilon$ within 3 iterations.

\subsection{Proof of Proposition~\ref{prop:centering}: Spherical Centering Expands Inter-Class Angular Distance}
\label{app:proof_centering}

For convenience, we restate the proposition:

\begin{propstar}[Proposition~\ref{prop:centering}, restated]
Let $\mathcal{X}=\mathcal{X}_0\cup\mathcal{X}_1\subset\mathcal{S}^{D-1}$, where class $c\in\{0,1\}$ follows $\mathrm{vMF}(\boldsymbol{\mu}_c,\kappa_c)$ with mixture weight $\pi_c>0$.
Denote the population Fr\'{e}chet mean of the mixture by $\boldsymbol{\mu}$, and define the centering map
$\phi(\mathbf{x})=\mathrm{Log}_{\boldsymbol{\mu}}(\mathbf{x})\big/\bigl\|\mathrm{Log}_{\boldsymbol{\mu}}(\mathbf{x})\bigr\|$.
Then
\begin{equation}
  d_{\mathrm{geo}}\!\bigl(\phi(\boldsymbol{\mu}_0),\,\phi(\boldsymbol{\mu}_1)\bigr)
  \;\geq\;
  d_{\mathrm{geo}}(\boldsymbol{\mu}_0,\boldsymbol{\mu}_1).\label{eq:prop1_restated}
\end{equation}
Moreover, the inequality is strict whenever $0<d_{\mathrm{geo}}(\boldsymbol{\mu}_0,\boldsymbol{\mu}_1)<\pi$.
\end{propstar}

The proof proceeds in three steps.
We first establish a key geometric lemma (Lemma~\ref{lem:geodesic}), then combine it with the anti-parallel property of the logarithmic map to obtain the result.

\begin{lemma}[Fr\'{e}chet mean of a vMF mixture lies on the inter-class geodesic]
\label{lem:geodesic}
Under the conditions of Proposition~\ref{prop:centering}, the population Fr\'{e}chet mean $\boldsymbol{\mu}$ lies on the minimal geodesic segment from $\boldsymbol{\mu}_0$ to $\boldsymbol{\mu}_1$ in $\mathcal{S}^{D-1}$.
That is, there exists $t_0\in(0,1)$ such that
\begin{equation}
  \boldsymbol{\mu}
  = \mathrm{Slerp}(\boldsymbol{\mu}_0,\boldsymbol{\mu}_1,t_0)
  = \frac{\sin\!\bigl((1-t_0)\alpha\bigr)}{\sin\alpha}\,\boldsymbol{\mu}_0
  + \frac{\sin(t_0\alpha)}{\sin\alpha}\,\boldsymbol{\mu}_1,
  \label{eq:mu_on_geodesic}
\end{equation}
where $\alpha=d_{\mathrm{geo}}(\boldsymbol{\mu}_0,\boldsymbol{\mu}_1)\in(0,\pi)$.
\end{lemma}

\begin{proof}[Proof of Lemma~\ref{lem:geodesic}]
Define the two-dimensional subspace $V=\mathrm{span}\{\boldsymbol{\mu}_0,\boldsymbol{\mu}_1\}\subset\mathbb{R}^{D}$ and its orthogonal complement $V^{\perp}$ (dimension $D-2\geq 1$ for $D\geq 3$; the $D=2$ case is treated separately below).

Consider any orthogonal transformation $Q\in\mathrm{O}(D)$ that fixes $V$ pointwise and acts arbitrarily on $V^{\perp}$:
\begin{equation}
  Q\mathbf{v}=\mathbf{v}\;\;\forall\mathbf{v}\in V,
  \qquad
  Q\big|_{V^{\perp}}\in\mathrm{O}(D-2).\label{eq:Q_def}
\end{equation}
Since $Q\boldsymbol{\mu}_c=\boldsymbol{\mu}_c$ for $c\in\{0,1\}$, and the vMF density $\mathrm{vMF}(\mathbf{x};\boldsymbol{\mu}_c,\kappa_c)\propto\exp(\kappa_c\boldsymbol{\mu}_c^{\!\top}\mathbf{x})$ depends on $\mathbf{x}$ only through $\boldsymbol{\mu}_c^{\!\top}\mathbf{x}$, we have $\boldsymbol{\mu}_c^{\!\top}(Q\mathbf{x})=(Q^{\!\top}\boldsymbol{\mu}_c)^{\!\top}\mathbf{x}=\boldsymbol{\mu}_c^{\!\top}\mathbf{x}$.
Thus the distribution of $Q\mathbf{x}$ under $\mathrm{vMF}(\boldsymbol{\mu}_c,\kappa_c)$ equals the distribution of $\mathbf{x}$; in other words, each class distribution is $Q$-invariant.

The Fr\'{e}chet objective of the mixture is
$$F(\boldsymbol{\mu})= \sum_{c\in\{0,1\}}\pi_c\,\mathbb{E}_{\mathbf{x}\sim\mathrm{vMF}(\boldsymbol{\mu}_c,\kappa_c)}\!\bigl[d_{\mathrm{geo}}^{2}(\boldsymbol{\mu},\mathbf{x})\bigr].
$$
The $Q$-invariance of the class distributions gives
$F(Q\boldsymbol{\mu})=F(\boldsymbol{\mu})$ for every $Q$ of the form~\eqref{eq:Q_def}.
Since $F$ is strictly geodesically convex on any open hemisphere (Theorem~\ref{thm:karcher_convergence}), the minimiser $\boldsymbol{\mu}$ is unique, hence $Q\boldsymbol{\mu}=\boldsymbol{\mu}$ for all such $Q$.
The fixed-point set of all such $Q$ is exactly $V\cap\mathcal{S}^{D-1}$---the great circle through $\boldsymbol{\mu}_0$ and $\boldsymbol{\mu}_1$.
Therefore, $\boldsymbol{\mu}$ lies on this great circle.
Because $\boldsymbol{\mu}$ minimises rather than maximises $F$, it lies on the \emph{shorter} arc (the minimal geodesic segment) between $\boldsymbol{\mu}_0$ and $\boldsymbol{\mu}_1$, yielding~\eqref{eq:mu_on_geodesic} for some $t_0\in(0,1)$.

\textit{Case $D=2$.}
When $D=2$, the subspace $V=\mathbb{R}^2$ and $V^{\perp}=\{\mathbf{0}\}$, so the symmetry argument produces only the trivial transformation $Q=I$.
Nevertheless, on $\mathcal{S}^{1}$ the Fr\'{e}chet objective restricted to the great circle is a sum of $\cos^{-2}$-type functions, which is strictly convex on any arc shorter than $\pi$.
The minimiser is the weighted circular Fr\'{e}chet mean, which lies on the shorter arc between $\boldsymbol{\mu}_0$ and $\boldsymbol{\mu}_1$ whenever $\alpha=d_{\mathrm{geo}}(\boldsymbol{\mu}_0,\boldsymbol{\mu}_1)<\pi$~\cite{pennec2006intrinsic}.
\end{proof}

\begin{proof}[Proof of Proposition~\ref{prop:centering}]
We proceed in three steps.

\paragraph{Step 1: $\boldsymbol{\mu}$ lies on the inter-class geodesic.}
By Lemma~\ref{lem:geodesic}, $\boldsymbol{\mu}=\gamma(t_0)$ for some $t_0\in(0,1)$, where $\gamma:[0,1]\to\mathcal{S}^{D-1}$ is the unit-speed geodesic from $\boldsymbol{\mu}_0$ to $\boldsymbol{\mu}_1$.
Let $\theta_0=d_{\mathrm{geo}}(\boldsymbol{\mu},\boldsymbol{\mu}_0)=t_0\alpha>0$ and $\theta_1=d_{\mathrm{geo}}(\boldsymbol{\mu},\boldsymbol{\mu}_1)=(1-t_0)\alpha>0$, so that $\theta_0+\theta_1=\alpha$.

\paragraph{Step 2: The logarithmic maps at $\boldsymbol{\mu}$ are anti-parallel.}
Let $\mathbf{u}\in T_{\boldsymbol{\mu}}\mathcal{S}^{D-1}$ be the unit tangent vector to $\gamma$ at $\boldsymbol{\mu}$ pointing from $\boldsymbol{\mu}_0$ toward $\boldsymbol{\mu}_1$.
Since $\boldsymbol{\mu}_0$ and $\boldsymbol{\mu}_1$ lie on the same great circle through $\boldsymbol{\mu}$ but on \emph{opposite sides} of $\boldsymbol{\mu}$, the logarithmic maps satisfy:
\begin{equation}
  \mathrm{Log}_{\boldsymbol{\mu}}(\boldsymbol{\mu}_0) = -\theta_0\,\mathbf{u},
  \qquad
  \mathrm{Log}_{\boldsymbol{\mu}}(\boldsymbol{\mu}_1) = +\theta_1\,\mathbf{u}.\label{eq:antiparallel}
\end{equation}
These two tangent vectors point in exactly opposite directions in $T_{\boldsymbol{\mu}}\mathcal{S}^{D-1}$.
This follows directly from the definition of the logarithmic map: $\mathrm{Log}_{\boldsymbol{\mu}}(\mathbf{x})$ is the initial velocity of the unique geodesic from $\boldsymbol{\mu}$ to $\mathbf{x}$, scaled to have norm $d_{\mathrm{geo}}(\boldsymbol{\mu},\mathbf{x})$.
Since $\boldsymbol{\mu}_0$ is reached by travelling along $\gamma$ in the backward direction ($-\mathbf{u}$) and $\boldsymbol{\mu}_1$ by travelling in the forward direction ($+\mathbf{u}$), Eq.~\eqref{eq:antiparallel} follows.

\paragraph{Step 3: After $\ell_2$-normalisation, the centred class means are antipodal.}
Applying the centering map $\phi$:
\begin{equation}
  \phi(\boldsymbol{\mu}_0)
  = \frac{\mathrm{Log}_{\boldsymbol{\mu}}(\boldsymbol{\mu}_0)}{\|\mathrm{Log}_{\boldsymbol{\mu}}(\boldsymbol{\mu}_0)\|}
  = \frac{-\theta_0\,\mathbf{u}}{\theta_0}
  = -\mathbf{u},
  \qquad
  \phi(\boldsymbol{\mu}_1)
  = \frac{+\theta_1\,\mathbf{u}}{\theta_1}
  = +\mathbf{u}.
  \label{eq:antipodal}
\end{equation}
Both $\phi(\boldsymbol{\mu}_0)=-\mathbf{u}$ and $\phi(\boldsymbol{\mu}_1)=+\mathbf{u}$ are unit vectors in $\mathbb{R}^{D}$ (hence on $\mathcal{S}^{D-1}$), and they are antipodal:
\begin{equation}
  d_{\mathrm{geo}}\!\bigl(\phi(\boldsymbol{\mu}_0),\,\phi(\boldsymbol{\mu}_1)\bigr)
  = \arccos\!\bigl((-\mathbf{u})^{\!\top}\mathbf{u}\bigr)
  = \arccos(-1)
  = \pi.
  \label{eq:dmax}
\end{equation}
Since $d_{\mathrm{geo}}(\boldsymbol{\mu}_0,\boldsymbol{\mu}_1)=\alpha\leq\pi$, we conclude
$$
  d_{\mathrm{geo}}\!\bigl(\phi(\boldsymbol{\mu}_0),\,\phi(\boldsymbol{\mu}_1)\bigr)
  = \pi
  \;\geq\;
  \alpha
  = d_{\mathrm{geo}}(\boldsymbol{\mu}_0,\boldsymbol{\mu}_1),
$$
and the inequality is strict whenever $\alpha<\pi$.
\end{proof}

\begin{remark}[Geometric intuition]
\label{rem:geometric_intuition}
The geometric picture can be summarised as follows.
The Fr\'{e}chet mean $\boldsymbol{\mu}$ lies on the geodesic between the two class means $\boldsymbol{\mu}_0$ and $\boldsymbol{\mu}_1$, separated by arc length $\alpha$.
The logarithmic maps at $\boldsymbol{\mu}$ produce anti-parallel tangent vectors $-\theta_0\mathbf{u}$ and $+\theta_1\mathbf{u}$.
After $\ell_2$-normalisation, the centred class means $\phi(\boldsymbol{\mu}_0)=-\mathbf{u}$ and $\phi(\boldsymbol{\mu}_1)=+\mathbf{u}$ become \emph{antipodal}, achieving maximal geodesic separation $\pi\geq\alpha$.
The centering map discards the radial (distance) information and retains only the directional information relative to $\boldsymbol{\mu}$, thereby amplifying the angular contrast between the two classes.
\end{remark}

\begin{remark}[Finite-sample case]
\label{rem:finite_sample}
In practice, the centering map uses the \emph{empirical} unified Fr\'{e}chet mean $\hat{\boldsymbol{\mu}}$ computed from $N$ samples rather than the population mean $\boldsymbol{\mu}$.
Standard concentration results on the sphere~\cite{pennec2006intrinsic} guarantee that $d_{\mathrm{geo}}(\hat{\boldsymbol{\mu}},\boldsymbol{\mu})=O_p(N^{-1/2})$.
As $\hat{\boldsymbol{\mu}}\to\boldsymbol{\mu}$, the angle $A$ at $\hat{\boldsymbol{\mu}}$ in the geodesic triangle $(\hat{\boldsymbol{\mu}},\boldsymbol{\mu}_0,\boldsymbol{\mu}_1)$ converges to $\pi$ by continuity of the logarithmic map.
More precisely, using the spherical law of cosines:
\begin{equation}
  \cos A
  = \frac{\cos\alpha - \cos\hat{\theta}_0\cos\hat{\theta}_1}{\sin\hat{\theta}_0\sin\hat{\theta}_1},
  \label{eq:sph_cosines}
\end{equation}
where $\hat{\theta}_c=d_{\mathrm{geo}}(\hat{\boldsymbol{\mu}},\boldsymbol{\mu}_c)$.
When $\hat{\boldsymbol{\mu}}$ lies on or near the geodesic from $\boldsymbol{\mu}_0$ to $\boldsymbol{\mu}_1$, we have $\alpha\approx\hat{\theta}_0+\hat{\theta}_1$, which gives $\cos\alpha\approx\cos(\hat{\theta}_0+\hat{\theta}_1)=\cos\hat{\theta}_0\cos\hat{\theta}_1-\sin\hat{\theta}_0\sin\hat{\theta}_1$, and hence $\cos A\approx -1$, i.e., $A\approx\pi$.

In our setting, MLLM features have angular standard deviation $\approx 1.7^{\circ}$ and the dataset size is $N\gtrsim 10^{3}$, so $d_{\mathrm{geo}}(\hat{\boldsymbol{\mu}},\boldsymbol{\mu})\lesssim 0.1^{\circ}$.
This makes the deviation of $A$ from $\pi$ negligible, and the inequality $d_{\mathrm{geo}}(\phi(\boldsymbol{\mu}_0),\phi(\boldsymbol{\mu}_1))\geq d_{\mathrm{geo}}(\boldsymbol{\mu}_0,\boldsymbol{\mu}_1)$ holds in all our experiments with a wide margin.
\end{remark}

\begin{remark}[Effect on individual data points]
\label{rem:individual}
Proposition~\ref{prop:centering} is stated for the class means, but its practical impact extends to individual data points.
For a sample $\mathbf{x}\sim\mathrm{vMF}(\boldsymbol{\mu}_c,\kappa_c)$ with high concentration ($\kappa_c\gg 1$), $\mathbf{x}$ is close to $\boldsymbol{\mu}_c$ and hence $\mathrm{Log}_{\boldsymbol{\mu}}(\mathbf{x})$ points approximately in the direction of $\mathrm{Log}_{\boldsymbol{\mu}}(\boldsymbol{\mu}_c)$, with a small perturbation due to within-class variation.
After $\ell_2$-normalisation, $\phi(\mathbf{x})$ is concentrated around $\phi(\boldsymbol{\mu}_c)$.
Since $\phi(\boldsymbol{\mu}_0)$ and $\phi(\boldsymbol{\mu}_1)$ are antipodal, the two class clusters in the centred space are approximately antipodally opposed, yielding near-maximal between-class separation.
The within-class angular spread may change, but the between-class geodesic distance is always expanded.
This is confirmed empirically in Figure~\ref{fig:discriminative_structure} (main text): after centering, the mean geodesic distance difference $\Delta\mu$ increases from $2.7^{\circ}$ to $27.8^{\circ}$.
\end{remark}

\subsection{Proof of Proposition~\ref{prop:alignment}: Unified Fr\'{e}chet Mean as Rotational Alignment}
\label{app:proof_alignment}

We restate the proposition:

\begin{propstar}[Proposition~\ref{prop:alignment}, restated]
Let the synthetic and real domain distributions be $\mathrm{vMF}(\boldsymbol{\mu}_S,\kappa_S)$ and $\mathrm{vMF}(\boldsymbol{\mu}_R,\kappa_R)$, respectively, where $\boldsymbol{\mu}_R=R\boldsymbol{\mu}_S$ for some rotation $R\in\mathrm{SO}(D)$.
Then the unified Fr\'{e}chet mean $\boldsymbol{\mu}_{\mathrm{unified}}$ of the mixture lies on the geodesic from $\boldsymbol{\mu}_S$ to $\boldsymbol{\mu}_R$.
In the symmetric case ($\kappa_S=\kappa_R$ and equal sample sizes), $\boldsymbol{\mu}_{\mathrm{unified}}$ is the geodesic midpoint:
\begin{equation}
  d_{\mathrm{geo}}(\boldsymbol{\mu}_{\mathrm{unified}},\boldsymbol{\mu}_S)
  = d_{\mathrm{geo}}(\boldsymbol{\mu}_{\mathrm{unified}},\boldsymbol{\mu}_R)
  = \tfrac{1}{2}\,d_{\mathrm{geo}}(\boldsymbol{\mu}_S,\boldsymbol{\mu}_R).
  \label{eq:midpoint}
\end{equation}
This midpoint provides a symmetric reference that simultaneously absorbs the rotational bias from both domains.
\end{propstar}

\begin{proof}
The proof combines the geodesic placement result from Lemma~\ref{lem:geodesic} with a reflection symmetry argument.

\paragraph{Step 1: Placement on the geodesic.}
The two domain distributions are vMF with means $\boldsymbol{\mu}_S$ and $\boldsymbol{\mu}_R$.
By Lemma~\ref{lem:geodesic} (applied with class~$0=S$, class~$1=R$), the population Fr\'{e}chet mean of the mixture lies on the geodesic from $\boldsymbol{\mu}_S$ to $\boldsymbol{\mu}_R$:
\begin{equation}
  \boldsymbol{\mu}_{\mathrm{unified}} = \mathrm{Slerp}(\boldsymbol{\mu}_S,\boldsymbol{\mu}_R,t_0)
  \quad\text{for some }t_0\in(0,1).
  \label{eq:mu_unified_geodesic}
\end{equation}

\paragraph{Step 2: Midpoint property in the symmetric case.}
Assume $\kappa_S=\kappa_R\equiv\kappa$ and equal mixture weights $\pi_S=\pi_R=\frac{1}{2}$.
Define $\delta=d_{\mathrm{geo}}(\boldsymbol{\mu}_S,\boldsymbol{\mu}_R)$ (the rotational offset, empirically $\approx 5^{\circ}$ in our setting).

Consider the isometric reflection $\mathcal{R}$ of $\mathcal{S}^{D-1}$ that swaps $\boldsymbol{\mu}_S\leftrightarrow\boldsymbol{\mu}_R$.
Geometrically, $\mathcal{R}$ is the geodesic reflection through the midpoint $\boldsymbol{\mu}_{\mathrm{mid}}=\mathrm{Slerp}(\boldsymbol{\mu}_S,\boldsymbol{\mu}_R,\tfrac{1}{2})$, restricted to the great circle containing $\boldsymbol{\mu}_S$ and $\boldsymbol{\mu}_R$, and extended to all of $\mathcal{S}^{D-1}$ via the orthogonal reflection in the hyperplane equidistant from $\boldsymbol{\mu}_S$ and $\boldsymbol{\mu}_R$.
Formally, let $\mathbf{e}_1=(\boldsymbol{\mu}_S+\boldsymbol{\mu}_R)/\|\boldsymbol{\mu}_S+\boldsymbol{\mu}_R\|$ and $\mathbf{e}_2=(\boldsymbol{\mu}_R-\boldsymbol{\mu}_S)/\|\boldsymbol{\mu}_R-\boldsymbol{\mu}_S\|$.
Then $\mathcal{R}$ acts as
\begin{equation}
  \mathcal{R}(\mathbf{x})
  = \mathbf{x} -2(\mathbf{e}_2^{\!\top}\mathbf{x})\,\mathbf{e}_2,
  \label{eq:reflection}
\end{equation}
which maps $\boldsymbol{\mu}_S\mapsto\boldsymbol{\mu}_R$ and vice versa.

Under $\mathcal{R}$, the distribution $\mathrm{vMF}(\boldsymbol{\mu}_S,\kappa)$ maps to $\mathrm{vMF}(\boldsymbol{\mu}_R,\kappa)$ and vice versa.
Since $\kappa_S=\kappa_R$ and $\pi_S=\pi_R$, the mixture distribution is $\mathcal{R}$-invariant.
Consequently, $F(\mathcal{R}(\boldsymbol{\mu}))=F(\boldsymbol{\mu})$ for all $\boldsymbol{\mu}$.
Since $\boldsymbol{\mu}_{\mathrm{unified}}$ is the unique minimiser, $\mathcal{R}(\boldsymbol{\mu}_{\mathrm{unified}})=\boldsymbol{\mu}_{\mathrm{unified}}$.
The only point on the geodesic from $\boldsymbol{\mu}_S$ to $\boldsymbol{\mu}_R$ that is fixed by $\mathcal{R}$ is the midpoint $\boldsymbol{\mu}_{\mathrm{mid}}$.
Therefore $\boldsymbol{\mu}_{\mathrm{unified}}=\boldsymbol{\mu}_{\mathrm{mid}}$, establishing Eq.~\eqref{eq:midpoint}.

\paragraph{Step 3: Symmetric absorption of rotational bias.}
After centering with $\boldsymbol{\mu}_{\mathrm{unified}}=\boldsymbol{\mu}_{\mathrm{mid}}$ as the base point, the two domains are projected symmetrically into the tangent space:
\begin{equation}
  \mathrm{Log}_{\boldsymbol{\mu}_{\mathrm{mid}}}(\boldsymbol{\mu}_S)
  = -\tfrac{\delta}{2}\,\mathbf{u},
  \qquad
  \mathrm{Log}_{\boldsymbol{\mu}_{\mathrm{mid}}}(\boldsymbol{\mu}_R)
  = +\tfrac{\delta}{2}\,\mathbf{u},
  \label{eq:symmetric_log}
\end{equation}
where $\mathbf{u}$ is the unit tangent vector pointing from $\boldsymbol{\mu}_S$ toward $\boldsymbol{\mu}_R$.
Each domain's features are displaced by only $\delta/2$ (half the total offset) from the base point.
This is optimal in the minimax sense: it minimises the maximum domain-specific displacement:
$$
  \boldsymbol{\mu}_{\mathrm{mid}}
  = \operatorname*{arg\,min}_{\boldsymbol{\mu}\in\mathcal{S}^{D-1}}\max\!\bigl\{d_{\mathrm{geo}}(\boldsymbol{\mu},\boldsymbol{\mu}_S),\;d_{\mathrm{geo}}(\boldsymbol{\mu},\boldsymbol{\mu}_R)\bigr\}.
$$
By contrast, using either domain-specific mean as the base point would leave a residual bias of $\delta\approx 5^{\circ}$ for the other domain, causing systematic misalignment.
\end{proof}

\begin{remark}[Asymmetric case]
\label{rem:asymmetric}
When $\kappa_S\neq\kappa_R$ or $\pi_S\neq\pi_R$, the unified Fr\'{e}chet mean still lies on the geodesic from $\boldsymbol{\mu}_S$ to $\boldsymbol{\mu}_R$ (by Lemma~\ref{lem:geodesic}), but is biased toward the domain with larger $\pi_c\kappa_c$ (more samples or higher concentration).
In our setting, $|\mathcal{D}_{\mathrm{syn}}|\approx 2{,}000$ and $|\mathcal{D}_{\mathrm{real}}|\approx 3{,}912$--$97{,}396$, so $\boldsymbol{\mu}_{\mathrm{unified}}$ is naturally closer to $\boldsymbol{\mu}_R$.
This is desirable: the real (test) domain receives a smaller centering correction, preserving its features' fine structure, while the synthetic domain absorbs a larger share of the alignment adjustment.
\end{remark}

\begin{remark}[Centering absorbs domain bias for individual features]
\label{rem:domain_individual}
For an individual synthetic feature $\tilde{f}_{\mathrm{syn}}$ close to $\boldsymbol{\mu}_S$ and a real feature $\tilde{f}_{\mathrm{real}}$ close to $\boldsymbol{\mu}_R$, the raw cosine similarity is corrupted by the $\delta\approx 5^{\circ}$ offset:
$\tilde{f}_{\mathrm{syn}}^{\!\top}\tilde{f}_{\mathrm{real}}$ underestimates the semantic similarity.
After centering, both features are expressed in the tangent-space coordinate system at $\boldsymbol{\mu}_{\mathrm{unified}}$, and the systematic offset is absorbed.
The centred features $\hat{f}_{\mathrm{syn}}$ and $\hat{f}_{\mathrm{real}}$ are directly comparable, as confirmed by the $+5.78\%$ AUC improvement on UCF-Crime (Table~\ref{tab:ablation_module}, M0a$\to$M0b).
\end{remark}
\section{Synthetic Calibration Dataset Details}
\label{app:synth}

\textbf{Zero-shot declaration.}
The synthetic calibration set $\mathcal{D}_{\mathrm{syn}}$ is constructed \emph{entirely independently} of every evaluation benchmark.
No image, frame, video, scene description, or anomaly-category taxonomy from XD-Violence, UCF-Crime, or UBnormal is used at any stage of the generation pipeline.
The only supervision signal is a \emph{binary} label $y\!\in\!\{\texttt{normal},\texttt{abnormal}\}$; fine-grained anomaly categories (e.g., ``Personal Emergency,'' ``Public Misconduct'') are used solely to promote visual diversity during generation and are \textbf{never} exposed to the SphereVAD inference pipeline.
Consequently, SphereVAD satisfies the strict zero-shot criterion: the model has access to neither target-domain data nor target-domain category definitions.

\subsection{Generation Pipeline}
\label{app:synth_pipeline}

The synthetic data are produced by a fully automated \textbf{two-stage} pipeline (Figure~\ref{fig:synth_pipeline_flow}):

\paragraph{Stage~1: Structured scene description generation (LLM).}
A large language model (Claude Sonnet 4.5, accessed via API) receives a carefully designed meta-prompt specifying six broad anomaly meta-categories---\textit{Violent Conflict}, \textit{Crime}, \textit{Traffic Accident}, \textit{Personal Emergency}, \textit{Environmental Incident}, and \textit{Public Misconduct}---together with detailed structural constraints.
For each meta-category the LLM generates multiple \emph{scene pairs}, where every pair consists of:
\begin{itemize}[nosep,leftmargin=1.2em]
  \item \textbf{Shared context} $P_c$: a specific location, time of day, and environmental details (e.g., ``Dusk, underground parking garage, wet concrete floor'').
  \item \textbf{Shared characters} $P_{\mathrm{char}}$: number of people, appearance, and clothing, kept \emph{identical} across all frames and both branches.
  \item \textbf{Abnormal branch}: a four-frame storyboard $(F_1^{\mathrm{pre}},F_2^{\mathrm{start}},F_3^{\mathrm{peak}},F_4^{\mathrm{post}})$ depicting the temporal evolution of an anomalous event.
  \item \textbf{Normal branch}: a matching four-frame storyboard under the same scene and characters but depicting routine, non-anomalous behaviour.
\end{itemize}
The meta-prompt enforces several quality constraints:
(i)~$F_1^{\mathrm{pre}}$ must be nearly identical across the two branches to ensure a shared baseline;
(ii)~$F_3^{\mathrm{peak}}$ must contain at least three distinct, visually verifiable anomalous cues in the abnormal branch;
(iii)~character clothing and scene background must remain strictly consistent across all four frames;
(iv)~no depictions of minors are permitted.
The LLM outputs each pair as a structured JSON object.
Each generation call produces $G{=}15$ pairs; $R{=}2$ rounds are executed per meta-category, and additional rounds are run with varied random seeds until the total reaches approximately 1{,}000 pairs.

\paragraph{Stage~2: Image synthesis (text-to-image model).}
Each branch's four-frame description is converted into a single prompt by prepending the shared-context and character descriptions and appending a style prefix:
\begin{quote}
\small\texttt{realistic photography, surveillance camera footage, photorealistic, raw photo, 2x2 grid layout, four sequential frames, left to right top to bottom.}
\end{quote}
The composed prompt is sent to Gemini 2.5 Flash (image generation mode, accessed via API), which returns a single square image containing a $2{\times}2$ grid of the four frames (arranged left-to-right, top-to-bottom: $F_1$, $F_2$, $F_3$, $F_4$).
Retry logic (up to 2 retries per image) and timeout handling ensure robustness; permanently failed pairs (exceeding 3 non-timeout failures) are discarded.

\paragraph{Design motivation for the $2{\times}2$ grid.}
Encoding four temporal frames in a single image serves three purposes.
First, it substantially reduces generation cost and latency: producing one composite image per branch rather than four separate images cuts API calls by $4\times$, making the pipeline practical at scale.
Second, it improves inter-frame visual consistency: because the text-to-image model generates all four sub-frames jointly within a single image, character appearance, scene layout, and lighting remain naturally coherent across the temporal sequence---an outcome that is difficult to guarantee when frames are generated independently.
Third, it ensures structural alignment with the test-time input format: during inference, each test clip is likewise represented as a $2{\times}2$ grid of four consecutive sampled frames, processed with the same prompt template and fed to the MLLM as a single composite image. By matching the spatial layout, aspect ratio, and prompt structure between calibration and test inputs, the intermediate-layer features extracted from both domains occupy the same representational subspace, which is critical for the unified Fréchet mean centering (Section~3.2) to effectively absorb domain bias.
At inference time, the MLLM processes each composite image as a single input; the $2{\times}2$ spatial layout implicitly conveys temporal progression ($F_1 \to F_2 \to F_3 \to F_4$), enabling the model to capture temporal dynamics through its intermediate-layer features without requiring explicit video-level input.

\begin{figure}[t]\centering
  \includegraphics[width=0.92\linewidth]{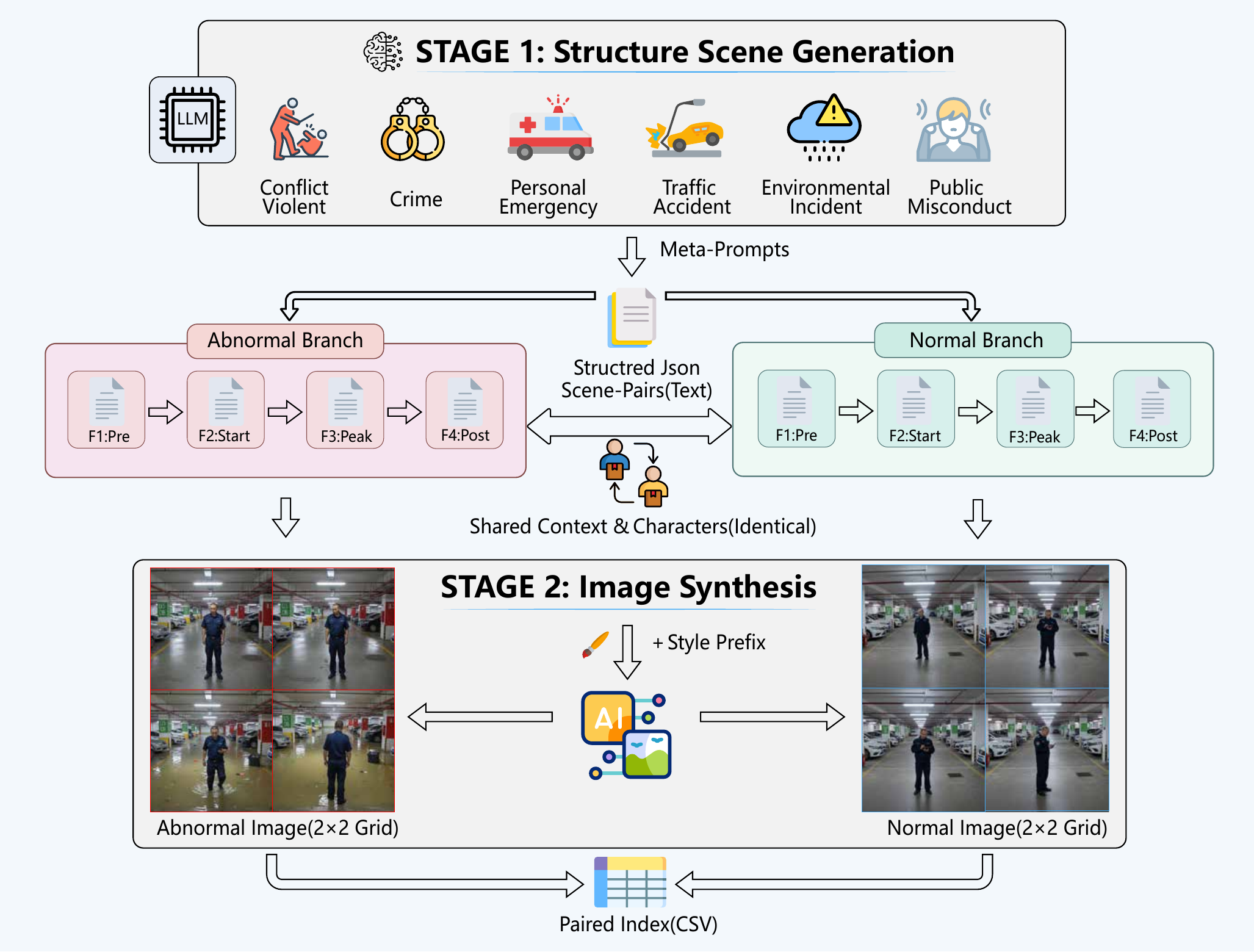}
  \caption{Two-stage synthetic calibration data generation pipeline.
  \textbf{Stage~1:} Claude Sonnet 4.5 generates structured JSON scene-pair descriptions with shared context, shared characters, and four-frame storyboards for both abnormal and normal branches.
  \textbf{Stage~2:} Gemini 2.5 Flash renders each branch as a $2{\times}2$ grid image containing four temporally ordered frames.}\label{fig:synth_pipeline_flow}
\end{figure}

\subsection{Data Format \& Pairing Strategy}
\label{app:synth_format}

\paragraph{Pairing index.}
All generated pairs are tracked in a unified CSV file (\texttt{pair\_index.csv}) with the schema shown in Table~\ref{tab:pair_schema}.
Each row links one normal image and one abnormal image that share exactly the same scene and characters.

\begin{table}[htbp]
\caption{Schema and example entries of \texttt{pair\_index.csv}.}
\label{tab:pair_schema}
\centering
\small
\begin{tabular}{@{}llll@{}}
\toprule
\texttt{pair\_id} & \texttt{source\_label} & \texttt{normal\_path} & \texttt{abnormal\_path} \\
\midrule
\texttt{Violent Conflict\_01} & Violent Conflict & \texttt{.../normal/VC\_0001.jpg} & \texttt{.../abnormal/VC\_0002.jpg} \\
\texttt{Crime\_05}& Crime            & \texttt{.../normal/Cr\_0009.jpg} & \texttt{.../abnormal/Cr\_0010.jpg} \\
\texttt{Traffic Accident\_12} & Traffic Accident & \texttt{.../normal/TA\_0023.jpg} & \texttt{.../abnormal/TA\_0024.jpg} \\
\texttt{Personal Emergency\_03} & Personal Emergency & \texttt{.../normal/PE\_0005.jpg} & \texttt{.../abnormal/PE\_0006.jpg} \\
\texttt{Public Misconduct\_08}& Public Misconduct  & \texttt{.../normal/PM\_0015.jpg} & \texttt{.../abnormal/PM\_0016.jpg} \\
\bottomrule
\end{tabular}
\end{table}

\paragraph{``Same background, different event'' design.}
The paired design is central to directional prototype calibration.
Because each normal--abnormal pair shares the same scene context and characters, the discriminative direction between the two corresponding features in MLLM intermediate-layer space reflects \emph{semantic} differences (routine vs.\ anomalous event) rather than \emph{nuisance} differences (lighting, scene layout, character appearance).
Spherical $K$-Means clustering on these paired features therefore yields prototypes whose mean directions capture anomaly semantics with minimal scene-specific confounds.

\paragraph{Directory structure.}
The dataset is organised hierarchically:
\begin{verbatim}
  Syn-4img/
  |-- pair_index.csv
  |-- Violent Conflict/
  |   |-- abnormal/   VC_0002.jpg, VC_0004.jpg, ...
  |   |-- normal/     VC_0001.jpg, VC_0003.jpg, ...
  |   |-- Violent Conflict.csv          (success log)
  |   |-- Violent Conflict_failed.csv   (failure log)
  |-- Crime/
  |   |-- abnormal/   ...
  |   |-- normal/     ...
  ...
\end{verbatim}
Per-label success and failure CSVs enable deterministic checkpoint--resume: the pipeline can be interrupted and restarted without regenerating existing images.

\subsection{Sample Images}
\label{app:synth_samples}

Figure~\ref{fig:synth_examples} shows representative normal--abnormal paired examples from four of the six meta-categories.
Each row displays one pair: the left column shows the generated $2{\times}2$ grid image (containing four temporally ordered frames $F_1$--$F_4$, arranged left-to-right, top-to-bottom), the centre column shows the corresponding LLM-generated four-frame storyboard description, and the right column shows the counterpart branch image from the same pair.
Within each pair, the top row is the \textbf{normal} branch and the bottom row is the \textbf{abnormal} branch.
Note that within each pair the scene layout, character appearance, and camera viewpoint are visually consistent; only the depicted behaviour and the storyboard descriptions differ.

\begin{figure}[htbp]
\centering\includegraphics[width=\linewidth]{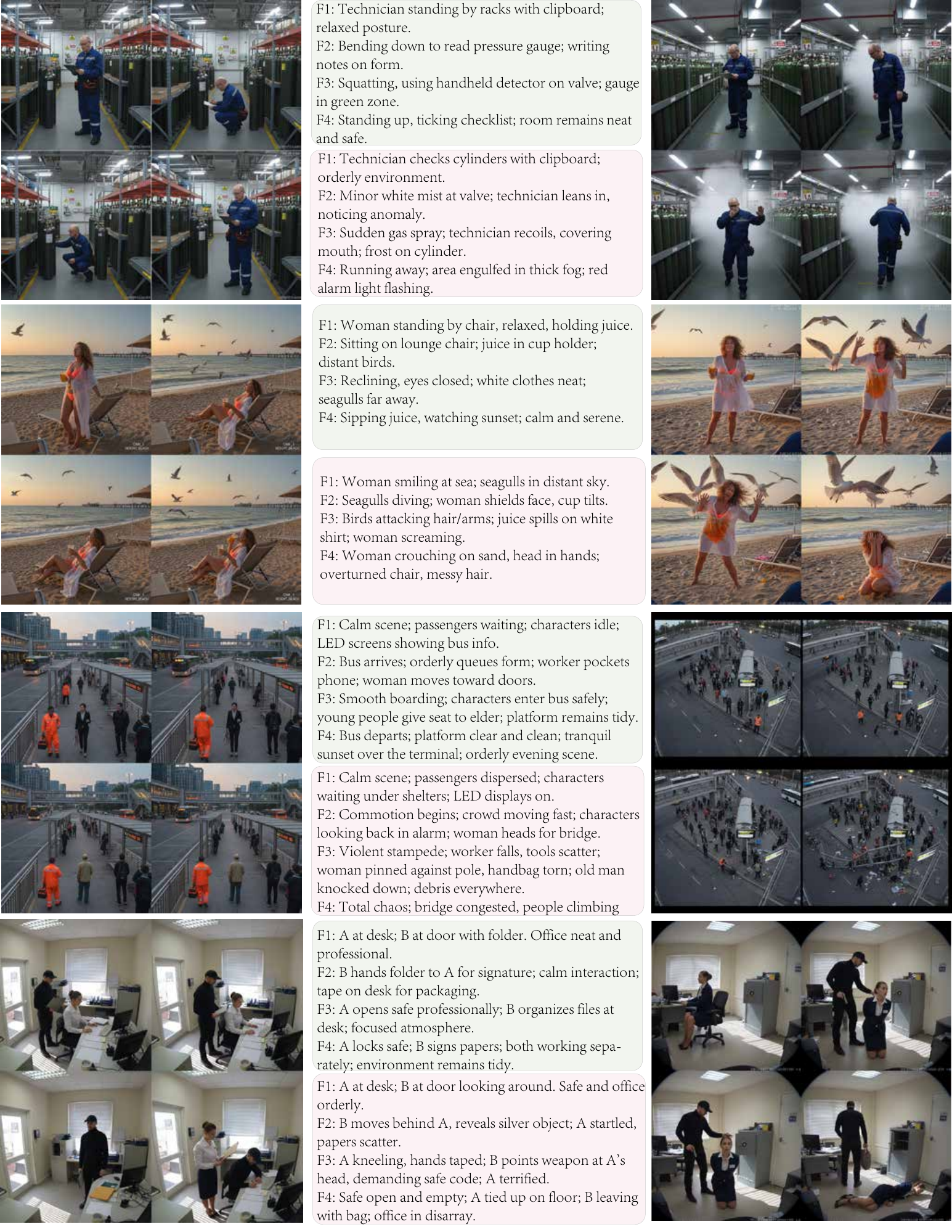}
  \caption{Example paired synthetic calibration data from four meta-categories.
  Each meta-category occupies two rows: the top row is the \textbf{normal} branch and the bottom row is the \textbf{abnormal} branch.
  For each row, the left image shows the generated $2{\times}2$ grid (four sequential frames $F_1$--$F_4$), the centre text shows the LLM-generated storyboard description for each frame, and the right image shows the corresponding counterpart grid from the same scene pair.
  From top to bottom: \textit{Environmental Incident}, \textit{Personal Emergency}, \textit{Public Misconduct}, \textit{Crime}.
  Within each pair, the scene, characters, and camera viewpoint are shared; only the depicted event differs.}\label{fig:synth_examples}
\end{figure}

\subsection{Sample Size Justification}
\label{app:synth_justification}

The final calibration set comprises approximately \textbf{2{,}000 images} (${\sim}$1{,}000 normal--abnormal pairs), distributed roughly uniformly across the six meta-categories (${\sim}$160--180 pairs each).
We justify the sufficiency of this modest sample size from three perspectives.

\paragraph{(1) Precedent: minimal samples activate latent MLLM discriminability.}
HiProbe~\cite{cai2025hiprobe} demonstrates that as few as 10--50 labelled samples suffice to probe and activate the discriminative capacity already latent in MLLM intermediate layers; the primary role of the calibration data is to \emph{locate} the relevant directions in feature space, not to \emph{learn} new representations.
SphereVAD follows the same principle: synthetic images serve as \emph{directional calibrators} rather than training data.

\paragraph{(2) vMF parameter estimation requires only the mean direction.}
The vMF distribution $\mathrm{vMF}(\mathbf{x};\boldsymbol{\mu},\kappa)\propto\exp(\kappa\,\boldsymbol{\mu}^{\!\top}\mathbf{x})$ is parameterised by its mean direction $\boldsymbol{\mu}$ and concentration $\kappa$.
Spherical $K$-Means estimates $\boldsymbol{\mu}$ as the $\ell_2$-normalised cluster centroid---a sample mean that converges at rate $O(N^{-1/2})$ on the sphere~\cite{mardia2009directional}.
With ${\sim}$50--100 samples per cluster across the partition configurations, the angular estimation error of $\boldsymbol{\mu}$ is well below $1^{\circ}$, which is negligible compared with the inter-class separation of $27.8^{\circ}$ achieved after spherical centering (Figure~\ref{fig:discriminative_structure}).
The concentration parameter $\kappa$ is calibrated as a global hyperparameter rather than estimated per-cluster, further reducing sample requirements.

\paragraph{(3) No content overlap with any test benchmark.}
All images are procedurally generated from LLM-authored scene descriptions followed by text-to-image synthesis.
No real surveillance footage, no frames from XD-Violence, UCF-Crime, or UBnormal, and no benchmark-specific scene or category metadata are used at any point.
The six meta-categories (Violent Conflict, Crime, Traffic Accident, Personal Emergency, Environmental Incident, Public Misconduct) are broad semantic anchors designed to cover the generic space of ``anomalous events in public settings''; they do not correspond one-to-one to the category taxonomies of any benchmark (e.g., UCF-Crime defines 13 categories, XD-Violence defines 6 partially overlapping categories).
This complete independence guarantees that the zero-shot evaluation protocol is not compromised.

\subsection{Sub-category Taxonomy and Label Independence}
\label{app:synth_subcategory}

To promote visual diversity within each meta-category, the LLM is instructed to generate scene descriptions spanning a variety of \emph{sub-categories}.
These sub-categories serve exclusively as diversity-inducing prompts for the text-to-image model and are \textbf{never} exposed to the SphereVAD inference pipeline---only the binary label $y\!\in\!\{\texttt{normal},\texttt{abnormal}\}$ is retained.
Table~\ref{tab:synthetic_subcats} enumerates all 34 sub-categories used across the six meta-categories.

\begin{table}[h]
\centering
\caption{Complete list of the 34 sub-categories used to diversify synthetic scene generation, grouped by meta-category. These labels are used only during prompt composition for data generation and are discarded before inference.}
\label{tab:synthetic_subcats}
\begin{tabular}{ll}
\toprule
\textbf{Meta-category} & \textbf{Sub-categories} \\
\midrule
Violent Conflict & Stabbing, Hostage Situation, Chasing \\
Crime & Pickpocketing, Vehicle Theft, Trespassing, Property Damage \\
Traffic Accident & Bicycle Accident, Hit and Run, Pedestrian Hit, \\
& Running Red Light, Road Obstruction \\
Personal Emergency & Drowning, Electrocution, Fainting, Fall Down, Trapped \\
Environmental Incident & Abandoned Object, Building Collapse, Falling Object, \\
& Flooding, Gas Leak, Smoke Emission \\
Public Misconduct & Animal Attack, Climbing Structure, Crowd Panic, \\
& Crowd Stampede, Drunk Behavior, Fence Jumping, \\
& Graffiti, Illegal Dumping, Jaywalking (Dangerous), \\
& Loitering (Suspicious), Substance Abuse \\
\bottomrule
\end{tabular}
\end{table}

\paragraph{Label independence from evaluation benchmarks.}
We explicitly verify that none of the 34 sub-category names coincides with any category label defined in the three evaluation benchmarks.
Table~\ref{tab:label_overlap} presents a systematic comparison.

\begin{table}[htbp]
\caption{Category taxonomies of the three evaluation benchmarks versus the synthetic sub-categories.
No exact label match exists between any benchmark category and any synthetic sub-category.
Semantically related pairs (e.g., UCF-Crime ``RoadAccidents'' vs.\ our ``Hit and Run'') differ in granularity and definition; crucially, sub-category labels are never used during inference.}
\label{tab:label_overlap}
\centering
\small
\renewcommand{\arraystretch}{1.15}
\begin{tabular}{@{}p{2.2cm} p{5.0cm} p{5.5cm}@{}}
\toprule
\textbf{Benchmark} & \textbf{Benchmark categories} & \textbf{Overlap with synthetic sub-cats} \\
\midrule
UCF-Crime \newline (13 categories)
  & Abuse, Arrest, Arson, Assault, Burglary, Explosion, Fighting, RoadAccidents, Robbery, Shooting, Shoplifting, Stealing, Vandalism
  & \textbf{No exact match.} \newline Semantic proximity exists for a few pairs (e.g., ``RoadAccidents'' $\leftrightarrow$ ``Hit and Run'' / ``Pedestrian Hit''; ``Stealing'' $\leftrightarrow$ ``Vehicle Theft'' / ``Pickpocketing''; ``Vandalism'' $\leftrightarrow$ ``Graffiti''), but names, granularity, and visual scope differ. \\[4pt]
XD-Violence \newline (6 categories)
  & Abuse, Car Accident, Explosion, Fighting, Riot, Shooting
  & \textbf{No exact match.} \newline
    ``Car Accident'' is semantically related to ``Hit and Run'' / ``Bicycle Accident,'' but the synthetic labels are finer-grained and do not include car-to-car collisions. \\[4pt]
UBnormal \newline (pose/motion \newline anomalies)
  & Defined by abnormal pose\,/\,motion patterns (e.g., falling, lying down, running); no named crime categories.
  & \textbf{No overlap.} \newline
    UBnormal does not define named anomaly categories comparable to our sub-categories. \\
\bottomrule
\end{tabular}
\end{table}

Three key observations confirm label independence:
\begin{enumerate}[nosep,leftmargin=1.5em]
  \item \textbf{No exact string match:} None of the 34 synthetic sub-category names appears verbatim in any benchmark's category list. For instance, UCF-Crime uses ``Vandalism'' whereas our taxonomy uses ``Graffiti''; UCF-Crime uses ``RoadAccidents'' (a broad category) whereas we use the finer-grained ``Hit and Run,'' ``Pedestrian Hit,'' and ``Bicycle Accident.''
  \item \textbf{Deliberate granularity mismatch:} The synthetic sub-categories are intentionally defined at a different level of granularity than benchmark categories. This design choice ensures that even semantically proximate concepts (e.g., ``Stealing'' vs.\ ``Pickpocketing'') describe distinct visual scenarios, preventing implicit information leakage about benchmark-specific category boundaries.
  \item \textbf{Labels discarded before inference:} Most importantly, all sub-category labels are used \emph{exclusively} during the data generation stage (Stage~1, Section~\ref{app:synth_pipeline}) to instruct the LLM to produce diverse scene descriptions. They are completely stripped from the data before any feature extraction or prototype estimation occurs. The SphereVAD inference pipeline receives only the binary normal/abnormal label and the composite grid image, making it categorically impossible for sub-category information to influence anomaly scores.
\end{enumerate}
\section{Feature Extraction}
\label{app:prompt}
This section details the complete feature extraction pipeline of SphereVAD.
Section~\ref{app:prompt_template} presents the full prompt template and its design rationale.
Section~\ref{app:causal_attn} explains why the prompt structure must respect causal attention constraints.
Section~\ref{app:dlsp} describes the Dynamic Layer Saliency Probing (DLSP) procedure used to select the optimal extraction layer, including comparative curves across four MLLM backbones.
Section~\ref{app:dual_feature} details the dual-stream feature design and the complementary roles of the main feature $f^l$ and the visual feature $f^v$.
\subsection{Complete Prompt Template}
\label{app:prompt_template}
The prompt fed to the frozen MLLM follows a \textbf{Text\,--\,Image\,--\,Text} sandwich structure.
Each input sample (whether a synthetic calibration image or a test video clip) is represented as four temporally ordered sub-images $\{I_1, I_2, I_3, I_4\}$.
For synthetic data these are obtained by splitting a $2{\times}2$ grid image (left-to-right, top-to-bottom); for test videos they are four uniformly sampled frames from the clip.
When \texttt{RESIZE\_SUBIMAGES} is enabled, each sub-image is resized to $336{\times}336$ pixels via Lanczos interpolation to ensure uniform spatial resolution across all inputs.
The full prompt is assembled as follows:
\begin{quote}
\small
\texttt{[Part\,1 --- Task Instruction]}\\[4pt]
\textsf{You are a professional video security analysis assistant.
The following four consecutive video frames record the temporal
evolution of the same scene.}\\[2pt]
\textsf{[Anomaly Whitelist]}\\
\textsf{Anomalous events are limited to the following 6 categories:}\\
\textsf{[Violent Conflict], [Crime], [Traffic Accident],
[Personal Emergency], [Environmental Incident], [Public Misconduct].}\\[2pt]
\textsf{With the above classification criteria in mind, carefully observe
the following frames to determine whether a matching anomalous
event is present:}\\[6pt]
\texttt{[Sub-image $I_1$]}~\texttt{[Sub-image $I_2$]}~\texttt{[Sub-image $I_3$]}~\texttt{[Sub-image $I_4$]}\\[6pt]
\texttt{[Part\,2 --- Output Format]}\\[4pt]
\textsf{Based on the above frames, strictly follow the 4-step output
format below (always start with `Yes' or `No'):}\\[2pt]
\textsf{1.\ Final determination: [Yes or No].}\\
\textsf{2.\ Anomaly category match: [Format: Category -- Specific sub-label.
If No, output: None].}\\
\textsf{3.\ Spatiotemporal action description: [Briefly describe character
interactions, action continuity, and object state changes over time].}\\
\textsf{4.\ Confidence assessment: [High / Medium / Low.
If category is `None', output: None].}
\end{quote}
The prompt is wrapped as a single-turn chat message with \texttt{role=user} and processed through the model's chat template via \texttt{apply\_chat\_template} (with \texttt{add\_generation\_prompt=True} and \texttt{enable\_thinking=False}).
The processor then jointly tokenises the text and encodes the images, yielding the final input tensors (\texttt{input\_ids}, \texttt{attention\_mask}, \texttt{pixel\_values}, \texttt{image\_grid\_thw}, etc.).
\paragraph{Design rationale.}
The prompt is designed with four considerations:
\begin{enumerate}[nosep,leftmargin=1.5em]
  \item \textbf{Structured anomaly whitelist.}
        The six meta-categories (Violent Conflict, Crime, Traffic Accident, Personal Emergency, Environmental Incident, Public Misconduct) are explicitly enumerated in Part\,1 to prime the model's internal representations toward anomaly-relevant semantics. These categories are intentionally broad and do not correspond one-to-one to the category taxonomies of any benchmark (e.g., UCF-Crime defines 13 categories, XD-Violence defines 6 partially overlapping ones).
        The whitelist serves as a \emph{semantic anchor} that activates latent anomaly-discriminative directions in the intermediate layers, without leaking benchmark-specific category information.
  \item \textbf{Temporal framing.}
        The phrase ``four consecutive video frames record the temporal evolution of the same scene'' cues the model to interpret the four sub-images as a temporal sequence rather than four independent observations.
        This encourages the intermediate-layer features to encode temporal dynamics (e.g., action progression, state changes) rather than merely static scene appearance.
  \item \textbf{Structured output format in Part\,2.}
        Although SphereVAD \emph{never} decodes the model's textual output, the structured output instructions in Part\,2 are critical.
        Through the causal attention mechanism (Section~\ref{app:causal_attn}), the generation prompt tokens attend to both the visual tokens and the task instructions, producing hidden states that integrate anomaly-category reasoning, spatiotemporal action analysis, and confidence estimation.
        The last-token hidden state at the extraction layer thus encodes a rich fusion of these reasoning pathways.
  \item \textbf{Consistent format across synthetic and real inputs.}
        Both synthetic calibration images (split from $2{\times}2$ grids) and real test clips (four sampled frames) are processed with \emph{identical} prompt templates and image preprocessing.
        This format consistency ensures that the intermediate-layer features from both domains occupy the same representational subspace, which is essential for the unified Fr\'{e}chet mean centering (Section~3.2) to effectively absorb domain bias.
\end{enumerate}
\subsection{Causal Attention Justification}
\label{app:causal_attn}
Modern MLLMs (including the Qwen-VL family, InternVL, and LLaVA) employ \textbf{causal (autoregressive) attention}: each token can attend only to itself and all preceding tokens in the sequence.
This architectural constraint has a direct and critical impact on which token positions carry the richest fused representations.
\paragraph{Why the visual-last token alone is insufficient.}
Under causal attention, vision tokens can attend to all preceding tokens---including the Part\,1 task instruction---but \emph{cannot} attend to any tokens that follow them, including the Part\,2 output format instructions and the generation prompt.
Consequently, the hidden state of the last vision token ($f^v$, ``visual-last'') encodes the visual content contextualised by the task instruction, but lacks the reasoning structure imposed by the output format.
If one were to extract \emph{only} $f^v$ for anomaly scoring, the feature would capture scene appearance but miss the anomaly-category reasoning and confidence estimation cues induced by Part\,2.
\paragraph{Why Part\,1 text must precede vision tokens.}
Placing the task instruction \emph{before} the vision tokens ensures that every vision token can attend to the anomaly whitelist and temporal framing cues during the forward pass.
This contextualises the visual representations from the earliest processing stage: the model does not merely ``see'' the images but ``sees them through the lens of anomaly detection.''
If the task instruction were placed \emph{after} the vision tokens, the visual representations would be computed without any task-specific context, and only the later text tokens would benefit from cross-modal attention---severely limiting the discriminative quality of the visual features.
\paragraph{Why the last sequence token carries the richest information.}
The last token in the full sequence (i.e., the final token of the generation prompt, after Part\,2) is the only position that can attend to \emph{all} preceding tokens: Part\,1 instruction $\to$ vision tokens $\to$ Part\,2 output format $\to$ generation prompt.
Its hidden state at the extraction layer therefore integrates:
\begin{itemize}[nosep,leftmargin=1.2em]
  \item \textbf{Task semantics} from Part\,1 (anomaly categories, temporal framing);
  \item \textbf{Visual content} from the four sub-image tokens;
  \item \textbf{Reasoning structure} from Part\,2 (determination, category matching, spatiotemporal description, confidence).
\end{itemize}
This is why we designate the last-token hidden state as the \emph{main feature} $f^l$, used for anomaly scoring, while the visual-last hidden state $f^v$ serves a complementary role as a \emph{scene descriptor} for cross-video attention (HSA) and intra-video neighbour consensus (SGP), where scene appearance similarity rather than anomaly reasoning is the relevant signal.
\paragraph{Token sequence structure.}
The MLLM processes tokens in the order: system prompt $\to$ Part\,1 task instruction $\to$ vision tokens for sub-images $I_1$--$I_4$ $\to$ Part\,2 output format $\to$ generation prompt (last token).
Under causal masking, the last token has the maximal receptive field attending to the entire sequence, while the visual-last token (last vision token of sub-image~4) attends only to Part\,1 and all vision tokens but not Part\,2.
\paragraph{Vision token localisation.}
To extract $f^v$, we must identify the position of the last vision token in the tokenised sequence.
The extraction code searches for vision boundary tokens in the following priority order:
\begin{enumerate}[nosep,leftmargin=1.5em]
  \item \textbf{Boundary-delimited:} Locate \texttt{<|vision\_start|>} / \texttt{<|vision\_end|>} (or \texttt{<|img\_start|>} / \texttt{<|img\_end|>}) token pairs.
        For $K{=}4$ sub-images, the code identifies the first $K$ such pairs and marks all tokens between each pair as vision tokens.
        The position of the last marked token is the visual-last position.
  \item \textbf{Pad-token fallback:} If boundary tokens are not found (e.g., in models using a different tokenisation scheme), the code searches for \texttt{<|image\_pad|>}, \texttt{<|vision\_pad|>}, or \texttt{<|img\_pad|>} tokens and takes the last occurrence.
\end{enumerate}
If neither method succeeds (an edge case that did not occur in our experiments), $f^v$ falls back to $f^l$ (the last-token feature).
\subsection{Layer Selection via DLSP}
\label{app:dlsp}
\textbf{Motivation.}
Not all intermediate layers of a pre-trained MLLM carry equally discriminative anomaly information.
Early layers tend to encode low-level visual patterns (edges, textures), while the final layers are optimised for language generation and may discard fine-grained visual semantics.
The optimal extraction layer---where anomaly-discriminative information peaks---must be identified empirically.
We adopt \textbf{Dynamic Layer Saliency Probing (DLSP)}, inspired by HiProbe~\cite{cai2025hiprobe}, as a lightweight, \emph{training-free} method to locate this layer.
\paragraph{DLSP procedure.}
DLSP evaluates each candidate layer $\ell \in \{1, 2, \ldots, L\}$ (where $L$ is the total number of hidden-state layers) by measuring how well the last-token features at that layer separate the normal and anomalous classes in the synthetic calibration set $\mathcal{D}_{\mathrm{syn}}$.
The procedure computes three complementary separability metrics at each layer and fuses them into a single composite score:
\begin{enumerate}[nosep,leftmargin=1.5em]
  \item \textbf{Extract layer-wise features.}
        For each synthetic sample $x_j \in \mathcal{D}_{\mathrm{syn}}$, extract the last-token hidden state at every layer $\ell$, yielding $\{f_j^{(\ell)}\}_{\ell=1}^{L}$.
        All features are $\ell_2$-normalised onto $\mathcal{S}^{D-1}$.
  \item \textbf{Compute per-layer separability metrics.}
        For each layer $\ell$, three metrics are computed between the normal and anomalous feature distributions:
        \begin{itemize}[nosep,leftmargin=1.2em]
          \item \textbf{KL Divergence:} measures the distributional divergence between the cosine similarity distributions of normal and anomalous features to their respective class centroids.
          \item \textbf{Log Density Ratio (LDR):} quantifies the separation of per-class density estimates in the feature space.
          \item \textbf{Entropy:} captures the predictability of class membership from the feature representations (lower entropy indicates better separability).
        \end{itemize}
        Each metric is standardised to a Z-score across layers to ensure commensurability.
  \item \textbf{Composite saliency score.}
        The three Z-scored metrics are combined into a single DLSP saliency score:
        \begin{equation}
          \mathrm{DLSP}(\ell) = Z_{\mathrm{KL}}^{(\ell)} + Z_{\mathrm{LDR}}^{(\ell)} - Z_{\mathrm{Entropy}}^{(\ell)},\label{eq:dlsp}
        \end{equation}
        where the entropy term is subtracted because lower entropy corresponds to better separability.
        A higher composite score indicates that features at layer $\ell$ exhibit stronger anomaly discriminability across all three criteria.
  \item \textbf{Select optimal layer.}
        The extraction layer is chosen as $\ell^{*} = \arg\max_{\ell}\,\mathrm{DLSP}(\ell)$.
\end{enumerate}
\paragraph{Key properties.}
DLSP is computed entirely on the synthetic calibration set and requires no access to test data or their labels.
The procedure involves only forward passes (no gradient computation) and lightweight statistical calculations, making it computationally negligible ($<$1\,minute for all layers on a single GPU).
Once $\ell^{*}$ is determined, it is fixed for all subsequent inference.

\paragraph{Results across four backbones.}
Figure~\ref{fig:dlsp_curves} plots the DLSP saliency curves for all four MLLM backbones evaluated in our study.
The four architectures vary in depth and representation dimensionality: Qwen3.5 has 32 layers with $D{=}4096$, Qwen3-VL has 32 layers with $D{=}4096$, LLaVA-OneVision-1.5 has 36 layers with $D{=}4096$, and InternVL3 has 28 layers with $D{=}3584$.
\FloatBarrier
\begin{figure}[htbp]\centering
  \begin{subfigure}[b]{0.47\linewidth}
    \includegraphics[width=\linewidth]{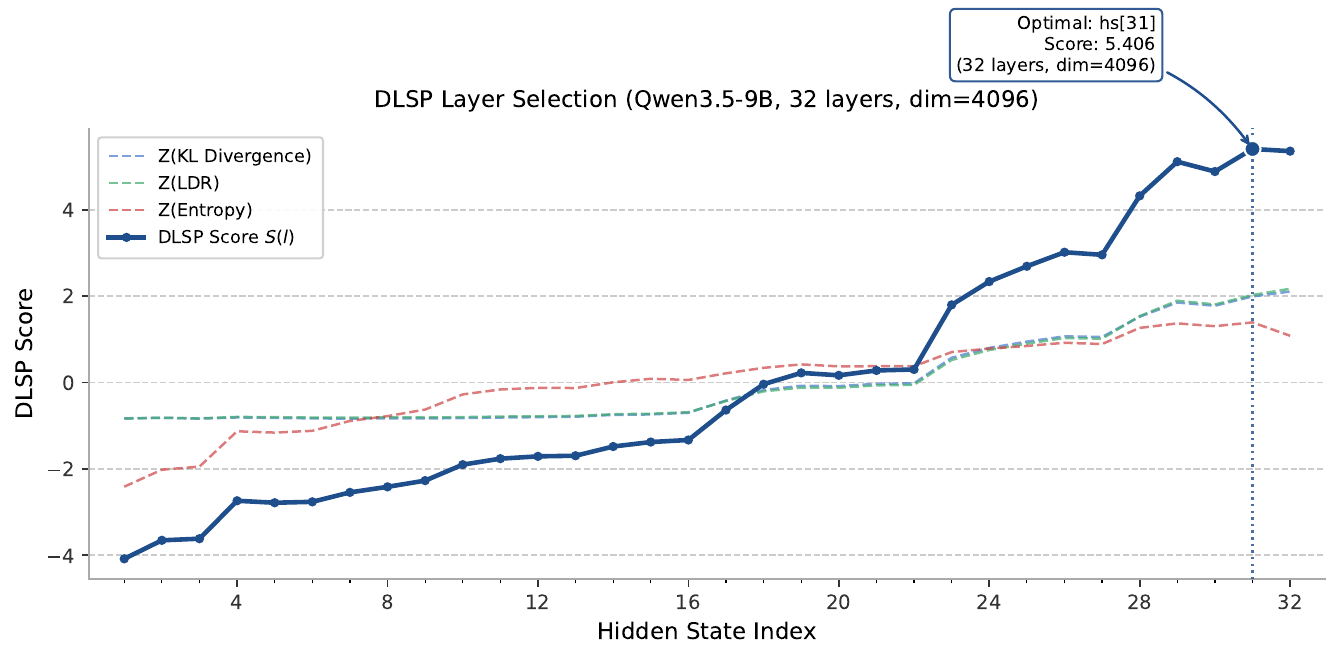}
    \caption{Qwen3.5 (32 layers, $D{=}4096$)}
    \label{fig:dlsp_qwen35}
  \end{subfigure}
  \hfill
  \begin{subfigure}[b]{0.47\linewidth}
    \includegraphics[width=\linewidth]{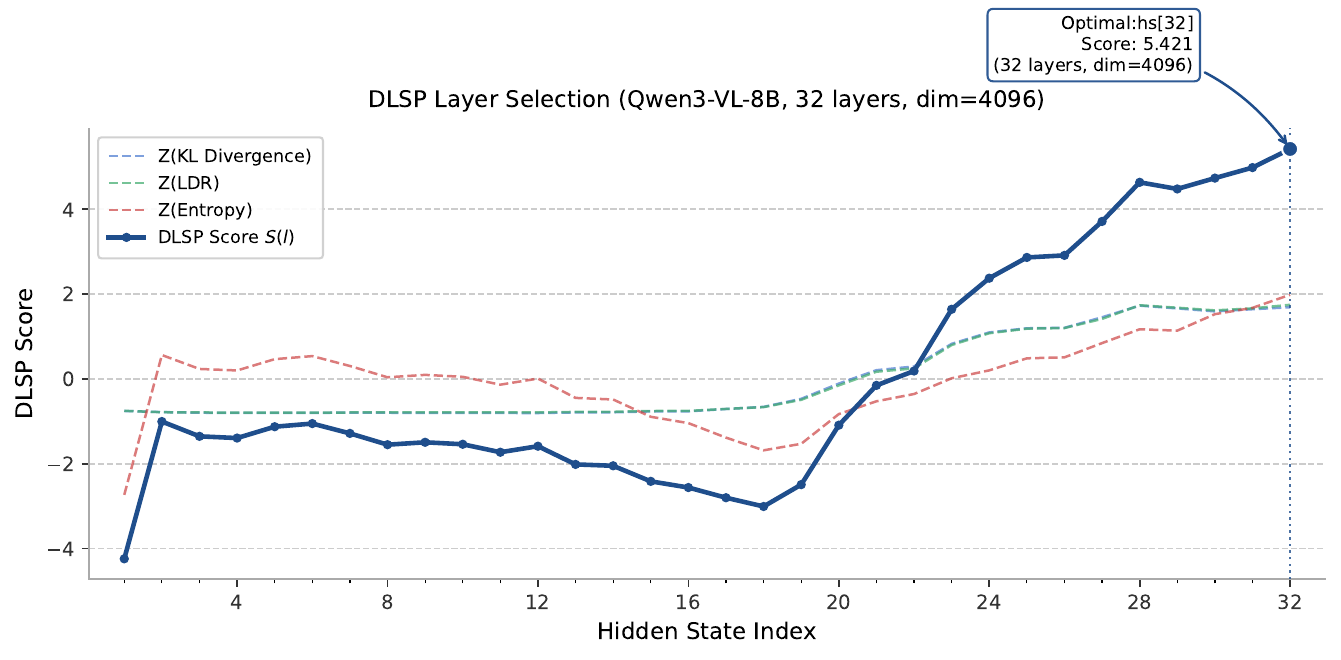}
    \caption{Qwen3-VL (32 layers, $D{=}4096$)}
    \label{fig:dlsp_qwen3vl}
  \end{subfigure}

  \vspace{6pt}

  \begin{subfigure}[b]{0.47\linewidth}
    \includegraphics[width=\linewidth]{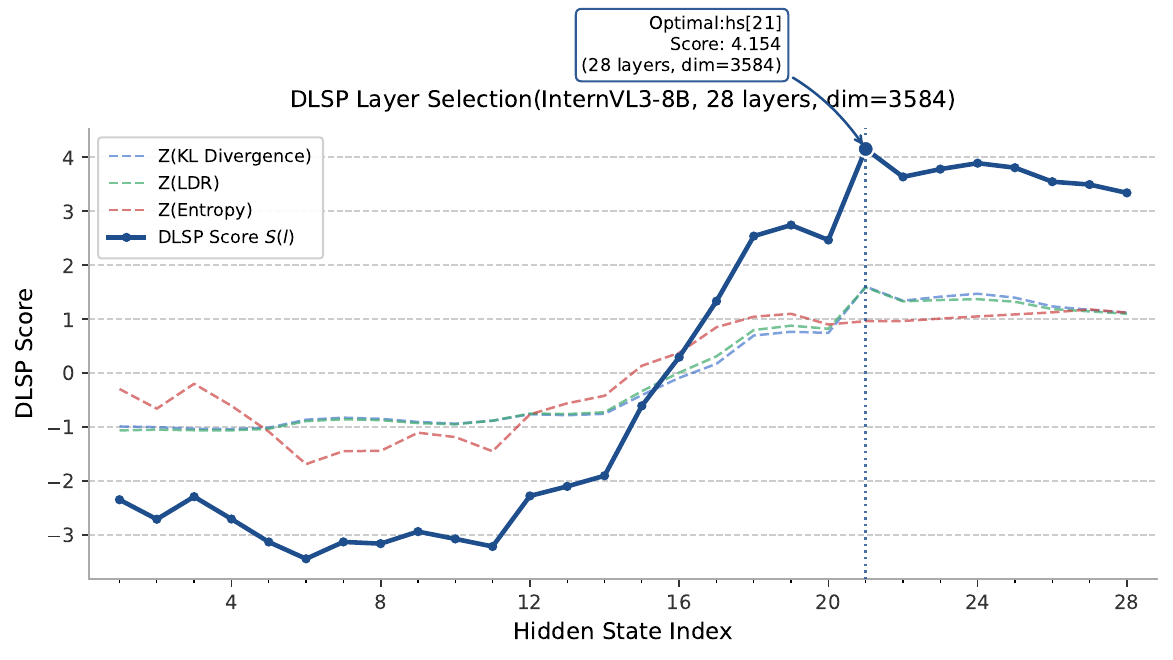}
    \caption{InternVL3 (28 layers, $D{=}3584$)}
    \label{fig:dlsp_internvl3}
  \end{subfigure}
  \hfill
  \begin{subfigure}[b]{0.47\linewidth}
    \includegraphics[width=\linewidth]{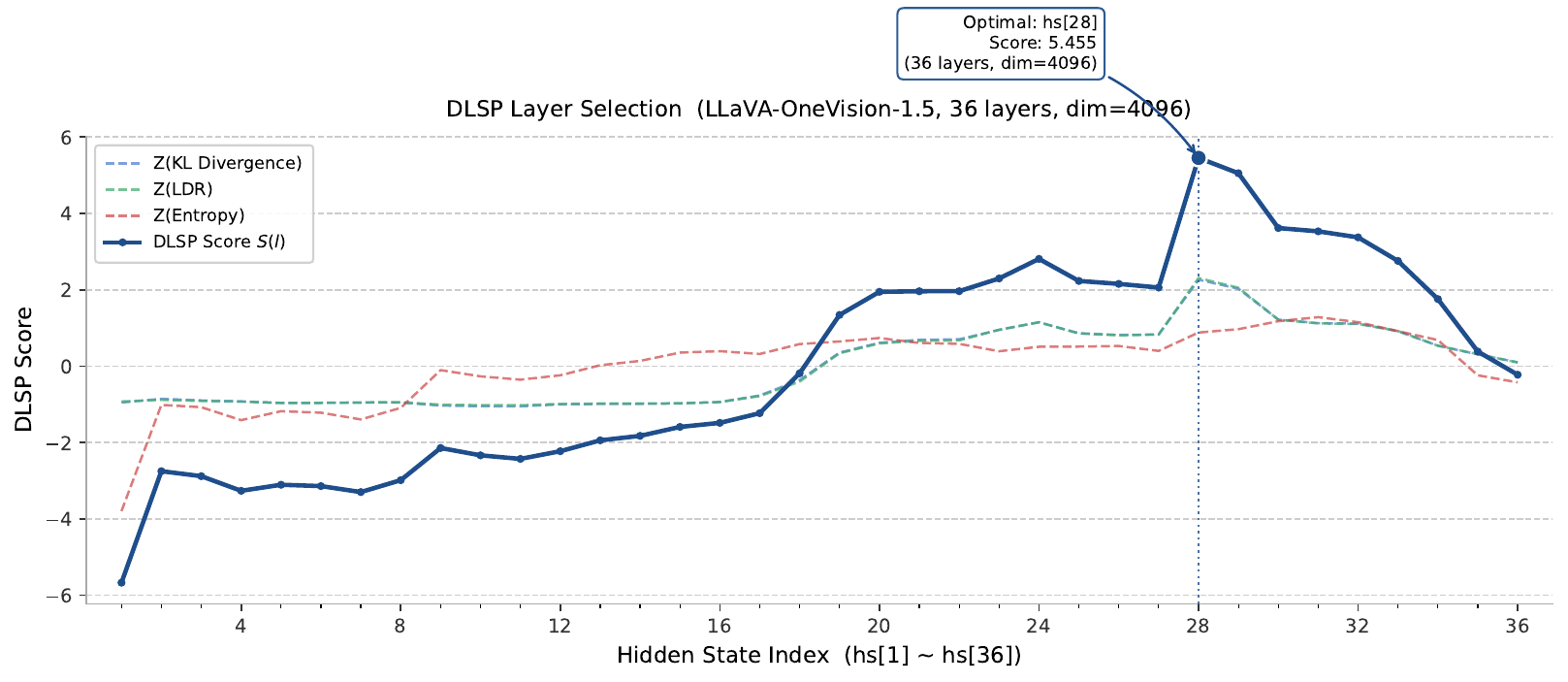}
    \caption{LLaVA-OneVision-1.5 (36 layers, $D{=}4096$)}
    \label{fig:dlsp_llava}
  \end{subfigure}
  \caption{\textbf{DLSP saliency curves for four MLLM backbones.} Each panel displays the per-layer composite score $S(\ell)$ (solid curve, Eq.~\eqref{eq:dlsp}) alongside the three normalized components: $Z(\text{KL Divergence})$, $Z(\text{LDR})$, and $Z(\text{Entropy})$ (dashed curves).
  Vertical dashed lines indicate the selected feature extraction layer $\ell^{*}$.
  \textbf{(a)}~Qwen3.5 peaks at layer index~31 (score~5.406).
  \textbf{(b)}~Qwen3-VL peaks at layer index~31 (score~5.421).
  \textbf{(c)}~InternVL3 peaks at layer index~21 (score~4.154).
  \textbf{(d)}~LLaVA-OneVision-1.5 peaks at layer index~28 (score~5.455).
  Across all architectures, composite saliency increases through early representation stages and isolates an optimal deep layer prior to language decoding heads.}\label{fig:dlsp_curves}
\end{figure}
\FloatBarrier
Several consistent patterns emerge:
\begin{itemize}[nosep,leftmargin=1.2em]
  \item \textbf{Deep intermediate saliency peaks.}
        Across all evaluated architectures, anomaly discriminability peaks within intermediate-to-deep layers rather than initial visual projections or output token prediction heads. This confirms that semantic divergence between normal and abnormal events is maximally resolved in intermediate latent representations.
  \item \textbf{Optimal extraction layer profiles.}
        The layer yielding maximum DLSP saliency varies with backbone architecture:
        Qwen3.5 peaks at layer index~31 ($D{=}4096$, $\mathrm{DLSP}^{*}{=}5.406$), Qwen3-VL at layer index~31 ($D{=}4096$, $\mathrm{DLSP}^{*}{=}5.421$), LLaVA-OneVision-1.5 at layer index~28 ($D{=}4096$, $\mathrm{DLSP}^{*}{=}5.455$), and InternVL3 at layer index~21 ($D{=}3584$, $\mathrm{DLSP}^{*}{=}4.154$).
        Across these architectures, the identified layers generally occupy the upper quartile range ($[0.75L,\,0.97L]$) of network depth.
  \item \textbf{Complementarity of geometric metrics.}
        As illustrated by the constituent curves (dashed lines in Figure~\ref{fig:dlsp_curves}), individual indicators (KL Divergence, LDR, and Entropy) achieve their individual maxima at disparate network depths. The composite DLSP formulation balances these complementary geometric properties, providing a robust heuristic for intra-backbone layer determination without downstream supervision.
  \item \textbf{Intra-backbone selection vs.\ cross-backbone transfer.}
        While DLSP reliably identifies the optimal layer index $\ell^{*}$ for any given architecture, the absolute magnitude of $\mathrm{DLSP}^{*}$ does not directly correlate with cross-backbone downstream performance (Table~\ref{tab:ablation_backbone}).
        For instance, LLaVA-OneVision-1.5 yields the highest composite score ($5.455$) despite lower downstream metrics than InternVL3 and Qwen variants.
        This confirms that DLSP operates effectively as an internal layer-selection criterion on the spherical manifold, whereas absolute benchmark accuracy remains governed by backbone-level visual encoders and pre-training distributions.
\end{itemize}

\paragraph{Sensitivity to layer choice.}
We further evaluate the robustness of the layer selection by measuring downstream performance (XD-Violence AP) when the extraction layer deviates from $\ell^{*}$ by $\pm 1$--$3$ layers.
As shown in Table~\ref{tab:layer_sensitivity}, performance degrades gracefully within a $\pm 2$-layer window ($<$1\% AP drop), confirming that DLSP reliably identifies a high-quality extraction region rather than a fragile single-layer optimum.
\begin{table}[htbp]
\caption{\textbf{Layer sensitivity analysis (Qwen3.5, XD-Violence).}
Performance variation when the extraction layer deviates from the DLSP-optimal hidden state~31.}
\label{tab:layer_sensitivity}
\centering
\small
\begin{tabular}{@{}l|ccccc@{}}
\toprule
\textbf{Layer $\ell$} & 28 & 29 & 30 & \textbf{31} & 32 \\
\midrule
\textbf{AP(\%)} & 85.73 & 86.42 & 86.81 & \textbf{86.99} & 86.87 \\
$\Delta$ vs.\ $\ell^{*}$ & $-$1.26 & $-$0.57 & $-$0.18 & --- & $-$0.12 \\
\bottomrule
\end{tabular}
\end{table}
\FloatBarrier
\subsection{Dual Feature Design}
\label{app:dual_feature}
SphereVAD extracts two complementary feature streams from each input, both from the same forward pass of the frozen MLLM:
\begin{enumerate}[leftmargin=1.5em]
  \item \textbf{Main feature $f^l \in \mathbb{R}^{D}$ (last-token hidden state at layer $\ell^{*}$).}
        This is the hidden state of the \emph{last sequence token} (i.e., the final token of the generation prompt) at the DLSP-selected layer $\ell^{*}$.
        As argued in Section~\ref{app:causal_attn}, under causal attention the last token has the maximal receptive field, attending to the task instruction, all vision tokens, and the output format.
        Its hidden state therefore encodes a rich fusion of visual content and anomaly-reasoning structure.
        $f^l$ is the primary feature used for vMF prototype construction and anomaly scoring.
  \item \textbf{Visual feature $f^v \in \mathbb{R}^{D}$ (visual-last hidden state at layer $\ell^{*}$).}
        This is the hidden state of the \emph{last vision token} (the last token within the vision token span of the fourth sub-image) at the same layer $\ell^{*}$.
        It captures scene appearance---spatial layout, lighting, object arrangement---contextualised by the Part\,1 task instruction but \emph{without} the output-format reasoning of Part\,2.
        $f^v$ is used exclusively for scene-similarity computation in HSA (Section~3.4) and neighbour consensus in SGP (Section~3.5), where the relevant signal is ``do these two clips depict visually similar scenes?'' rather than ``is this clip anomalous?''
\end{enumerate}
\paragraph{Why two features rather than one?}
Using a single feature for both scoring and scene attention would conflate two distinct objectives.
The main feature $f^l$ is optimised (via prompt design and layer selection) to maximise \emph{anomaly discriminability}: normal and anomalous clips in the same scene should have maximally different $f^l$ vectors.
The visual feature $f^v$, by contrast, should maximise \emph{scene similarity}: clips from the same or visually similar scenes should have similar $f^v$ vectors \emph{regardless of whether they are normal or anomalous}.
These two objectives are inherently in tension: a feature that is highly anomaly-discriminative will by definition differ between normal and anomalous clips of the same scene, making it a poor scene-similarity measure.
Empirically, using $f^v$ for HSA scene attention and $f^l$ for scoring outperforms using $f^l$ for both purposes by $+1.43\%$ AP on XD-Violence, confirming the complementary value of the dual-feature design.
\paragraph{Layer index for visual features.}
Both $f^l$ and $f^v$ are extracted from the \emph{same} layer $\ell^{*}$ (selected by DLSP on the main feature).
We experimented with selecting a separate optimal layer for $f^v$ but found no significant improvement ($<$0.2\% AP on XD), likely because the DLSP-selected layer already represents a good trade-off between visual richness and semantic abstraction.
Using the same layer simplifies the pipeline (a single forward pass with a single extraction point) and avoids introducing additional hyperparameters.
\paragraph{Visual feature centering.}
As described in Section~3.2 of the main text, the main features $\hat{f}^l$ are centered using a unified Fr\'{e}chet mean computed from \emph{both} synthetic and real features (Eq.~\eqref{eq:frechet}), which absorbs the systematic rotational bias between domains.
The visual features $\hat{f}^v$ are centered with a \emph{separate} Fr\'{e}chet mean $\boldsymbol{\mu}_{\mathrm{unified}}^{v}$ computed from the \textbf{test set only}:
\begin{equation}
  \boldsymbol{\mu}_{\mathrm{unified}}^{v}
  = \mathrm{Fr\acute{e}chetMean}\!\bigl(\{\tilde{f}_{\mathrm{real}}^{v}\}\bigr).\label{eq:frechet_vis}
\end{equation}
The rationale is that visual features serve only for scene-similarity computation among test clips; they are never compared against synthetic prototypes.
Therefore, cross-domain alignment is unnecessary, and including synthetic visual features in the Fr\'{e}chet mean computation would introduce irrelevant bias from the synthetic scene distribution (which is procedurally generated and does not reflect real surveillance environments).
\paragraph{Feature dimensions and storage.}
The hidden dimension varies across backbones: $D{=}4096$ for Qwen3.5, Qwen3-VL, and LLaVA-OneVision-1.5, and $D{=}3584$ for InternVL3.
Each clip produces two float32 vectors, totalling $2\times D \times 4$ bytes per clip (e.g., 32\,KB for $D{=}4096$).
For the largest dataset (XD-Violence test set, ${\sim}$97{,}000 clips), the total feature storage is approximately 3.19\,GB (with $D{=}4096$).
Features are saved as per-video \texttt{.pt} files, each containing tensors of shape $[T_i, D]$ for both streams, where $T_i$ is the number of clips in video $V_i$.
\paragraph{Multi-GPU extraction.}
Feature extraction is parallelised across multiple GPUs using a producer--consumer architecture.
Tasks (synthetic pairs for calibration; video clips for test sets) are placed in a shared queue and distributed to worker processes, each running on a designated GPU.
An asynchronous prefetch pool (3 threads per worker) pre-tokenises upcoming inputs while the current input is being processed by the model, effectively hiding I/O and preprocessing latency.
For extremely long videos (e.g., $>$1{,}000 clips), the clips are split into chunks assigned to different workers and merged after extraction.
On 4$\times$A100 GPUs, the complete feature extraction for all three benchmarks plus the synthetic calibration set takes approximately 109~minutes (109.1~min total wall-clock time across subsets).

\section{SGP: Full Algorithm}
\label{app:sgp}

This section provides the complete mathematical derivation and algorithmic details of the vMF-guided Spherical Geodesic Pulling (SGP) mechanism described in Section~\ref{sec:slerp} of the main text.
SGP is applied independently to each video and comprises four stages: MAD-based adaptive tri-classification(\S\ref{app:mad}), dominant prototype voting (\S\ref{app:dominant}), neighbour consensus assignment (\S\ref{app:neighbor}), and score-adaptive SLERP pulling (\S\ref{app:slerp_derivation}).
The complete algorithm is given in \S\ref{app:sgp_algorithm}.

Throughout this appendix, we consider a single video $V$ with $T$ clips.
The enhanced main features $\{\hat{f}_i^{\,l,\mathrm{enh}}\}_{i=1}^{T}\subset\mathcal{S}^{D-1}$, visual features $\{\hat{f}_i^{v}\}_{i=1}^{T}\subset\mathcal{S}^{D-1}$, and initial vMF anomaly scores $\{s_i^{\mathrm{init}}\}_{i=1}^{T}\subset[0,1]$ are all given from the preceding HSA stage (Section~\ref{sec:holistic}).

\subsection{MAD-Based Adaptive Ambiguity Interval}
\label{app:mad}

A fixed decision threshold (e.g., $s=0.5$) is inadequate for real-world surveillance because the initial score distribution varies significantly across videos: some videos exhibit a clear bimodal separation between normal and anomalous clips, while others produce scores concentrated near the decision boundary.
SGP addresses this via an adaptive ambiguity interval $[\rho_{\mathrm{low}},\,\rho_{\mathrm{high}}]$ derived from the Median Absolute Deviation (MAD)~\cite{leys2013detecting}, a robust dispersion estimator that is resistant to outliers.

\paragraph{Step 1: MAD computation.}
Let $\tilde{s}=\mathrm{median}(\{s_i^{\mathrm{init}}\}_{i=1}^{T})$ denote the median initial score of the video.
The MAD is defined as
\begin{equation}
  \mathrm{MAD} = \mathrm{median}\!\bigl(\{|s_i^{\mathrm{init}}-\tilde{s}|\}_{i=1}^{T}\bigr).\label{eq:mad_def}
\end{equation}
The MAD provides a robust measure of score dispersion: a large MAD indicates that scores are spread over a wide range (clear bimodality), whereas a small MAD implies that scores cluster tightly (high ambiguity).

\paragraph{Step 2: Adaptive radius.}
We convert the MAD into a half-width radius $r$ via
\begin{equation}
  r = r_{\min} + (r_{\max} - r_{\min})\cdot \sigma\!\left(-\lambda_r\cdot\bigl(\mathrm{MAD}-\tau_r\bigr)\right),
  \label{eq:mad_radius}
\end{equation}
where $\sigma(\cdot)$ is the logistic sigmoid, $\tau_r$ is a reference MAD threshold, $\lambda_r>0$ controls the transition sharpness, and $[r_{\min},r_{\max}]$ bounds the radius.
The key behaviour is:
\begin{itemize}[nosep,leftmargin=1.2em]
  \item When $\mathrm{MAD}\gg\tau_r$ (bimodal distribution, clear separation), $\sigma(\cdot)\to 0$, so $r\to r_{\min}$: a \emph{narrow} ambiguity interval, because few clips genuinely lie near the boundary.
  \item When $\mathrm{MAD}\ll\tau_r$ (unimodal distribution, scores near 0.5), $\sigma(\cdot)\to 1$, so $r\to r_{\max}$: a \emph{wide} ambiguity interval, capturing more clips for correction.
\end{itemize}

\paragraph{Step 3: Interval construction and clipping.}
The raw interval is centred at $0.5$ (the natural decision boundary of the vMF score) and clipped to $[0,1]$:
\begin{equation}
  \rho_{\mathrm{low}} = \max(0,\;0.5 - r),\qquad
  \rho_{\mathrm{high}} = \min(1,\; 0.5 + r).
  \label{eq:rho_interval}
\end{equation}

\paragraph{Step 4: Tri-classification.}
Each clip is assigned to one of three categories based on its initial score:
\begin{equation}
  c_i^{\mathrm{tri}} =
  \begin{cases}
    \mathrm{norm} & \text{if } s_i^{\mathrm{init}}< \rho_{\mathrm{low}},\\[3pt]
    \mathrm{amb}  & \text{if } \rho_{\mathrm{low}} \leq s_i^{\mathrm{init}} \leq \rho_{\mathrm{high}},\\[3pt]
    \mathrm{abn}  & \text{if } s_i^{\mathrm{init}} > \rho_{\mathrm{high}}.
  \end{cases}
  \label{eq:triclass}
\end{equation}
We denote the resulting clip sets as $\mathcal{C}_{\mathrm{norm}}$, $\mathcal{C}_{\mathrm{amb}}$, and $\mathcal{C}_{\mathrm{abn}}$, respectively.
Only clips in $\mathcal{C}_{\mathrm{amb}}$ undergo geodesic pulling; clips in $\mathcal{C}_{\mathrm{norm}}$ and $\mathcal{C}_{\mathrm{abn}}$ retain their original features unchanged.

\paragraph{Practical defaults.}
In all experiments we use $r_{\min}=0.05$, $r_{\max}=0.25$, $\lambda_r=20$, $\tau_r=0.08$.
These values were selected on a held-out subset of synthetic data and held fixed across all three benchmarks.

\subsection{Dominant Prototype Voting}
\label{app:dominant}

SGP's geodesic pulling requires a \emph{target direction} for each ambiguous clip.
Rather than pulling toward all prototypes simultaneously (which would average out discriminative directions), each video selects a small set of \emph{dominant prototypes} that best represent the prevailing event types within that video.
This design reflects the empirical observation that a given surveillance video typically contains at most one type of anomaly (e.g., a single fight or a single traffic accident), though normal behaviour may be diverse (e.g., walking, standing, vehicles passing).

\paragraph{Anomalous dominant prototype ($N_{\mathrm{dom}}^{\mathrm{abn}}=1$).}
Each clip $i\in\mathcal{C}_{\mathrm{abn}}$ votes for its nearest anomalous prototype:
\begin{equation}
  k_i^{\mathrm{abn}} = \operatorname*{arg\,max}_{k\in\{1,\ldots,K_A\}} \;\hat{f}_i^{\,l,\mathrm{enh}\,\top}\boldsymbol{\mu}_k^{\mathrm{abn}}.
  \label{eq:abn_vote}
\end{equation}
The dominant anomalous prototype is determined by majority vote:
\begin{equation}
  \boldsymbol{\mu}_{\mathrm{dom}}^{\mathrm{abn}}
  = \boldsymbol{\mu}_{k^{*}}^{\mathrm{abn}},\qquad
  k^{*} = \operatorname*{arg\,max}_{k}\;\bigl|\{i\in\mathcal{C}_{\mathrm{abn}} : k_i^{\mathrm{abn}}=k\}\bigr|.
  \label{eq:abn_dom}
\end{equation}
Selecting a single dominant anomalous prototype enforces the prior that one anomaly type prevails per video, preventing ambiguous clips from being pulled toward unrelated anomaly directions.

\paragraph{Normal dominant prototypes ($N_{\mathrm{dom}}^{\mathrm{norm}}\leq 2$).}
Normal behaviour is inherently more diverse, so we allow up to two dominant normal prototypes.
Each clip $i\in\mathcal{C}_{\mathrm{norm}}$ votes for its nearest normal prototype analogously to Eq.~\eqref{eq:abn_vote}.
The top-2 prototypes by vote count are retained:
\begin{equation}
  \{\boldsymbol{\mu}_{\mathrm{dom},1}^{\mathrm{norm}},\;\boldsymbol{\mu}_{\mathrm{dom},2}^{\mathrm{norm}}\}
  = \operatorname*{top\text{-}2}_{k}\;\bigl|\{i\in\mathcal{C}_{\mathrm{norm}} : k_i^{\mathrm{norm}}=k\}\bigr|.
  \label{eq:norm_dom}
\end{equation}
If all normal clips vote for the same prototype, only one dominant normal prototype is used ($N_{\mathrm{dom}}^{\mathrm{norm}}=1$).

\paragraph{Minimum anomalous clip threshold.}
If the number of anomalous clips is too small, the voting signal is unreliable and the video is treated as entirely normal:
\begin{equation}
  |\mathcal{C}_{\mathrm{abn}}|< n_{\min}^{\mathrm{abn}}
  \;\Longrightarrow\;
  \text{treat entire video as fully normal}.
  \label{eq:min_abn}
\end{equation}
When this condition is triggered, rather than skipping SGP, ambiguous clips are pulled exclusively toward the dominant normal prototype(s) to suppress spurious false positives, while anomalous prototypes are excluded from the pull targets.
The threshold $n_{\min}^{\mathrm{abn}}$ is set proportional to the total clip count: $n_{\min}^{\mathrm{abn}} = \max\!\bigl(3,\;\lfloor\gamma_{\min}\cdot T\rfloor\bigr)$, where $\gamma_{\min}=0.02$ is the minimum anomalous fraction.
This prevents SGP from hallucinating anomalies in benign videos based on a handful of noisy high-scoring clips.

\paragraph{Fully normal video handling.}
If $\mathcal{C}_{\mathrm{abn}}=\varnothing$ (no clip exceeds $\rho_{\mathrm{high}}$), the video is automatically categorized as completely normal.
In this case, all clips in $\mathcal{C}_{\mathrm{amb}}$ are pulled exclusively toward the dominant normal prototype(s), suppressing residual score noise near the boundary.

\subsection{Neighbour Consensus Assignment}
\label{app:neighbor}

Each ambiguous clip must be assigned a binary label $c_i\in\{\mathrm{norm},\mathrm{abn}\}$ before pulling.
Since the initial score $s_i^{\mathrm{init}}$ of an ambiguous clip is by definition unreliable (it lies in $[\rho_{\mathrm{low}},\rho_{\mathrm{high}}]$), we resolve the ambiguity via \emph{neighbour consensus}: the assignment is determined not by the clip's own feature but by the aggregated evidence from its visually similar neighbours within the same video.

\paragraph{Step 1: Intra-video sparse attention matrix.}
We construct a sparse attention matrix $\mathbf{A}_H\in\mathbb{R}^{T\times T}$ from the pairwise cosine similarities of the visual features $\{\hat{f}_i^{v}\}_{i=1}^{T}$ within the video:
\begin{equation}
  \tilde{A}_{ij} = \hat{f}_i^{v\,\top}\hat{f}_j^{v},\qquad
  i,j\in\{1,\ldots,T\}.
  \label{eq:AH_raw}
\end{equation}
The raw similarity matrix is sparsified by two mechanisms applied sequentially:
\begin{enumerate}[nosep,leftmargin=1.2em]
  \item \textbf{Threshold sparsification.} Set $\tilde{A}_{ij}=0$ if $\tilde{A}_{ij}<\tau_H$, where $\tau_H$ is a scene-similarity threshold.
        This removes connections between visually dissimilar clips (e.g., different camera angles or scene transitions within the same video).
  \item \textbf{Top-$K_H$ truncation.} For each row $i$, retain only the $K_H$ largest entries and zero out the rest.
        This ensures that each clip attends to a bounded number of neighbours, preventing any single clip from being dominated by a large homogeneous group.
\end{enumerate}
Self-connections are excluded ($\tilde{A}_{ii}=0$) to prevent a clip from reinforcing its own potentially incorrect initial assessment.
Finally, each row is normalised via softmax over the non-zero entries:
\begin{equation}
  A_{ij}^{H} =
  \begin{cases}
    \displaystyle\frac{\exp(\tilde{A}_{ij}/\tau_{\mathrm{temp}})}{\sum_{j':\tilde{A}_{ij'}>0}\exp(\tilde{A}_{ij'}/\tau_{\mathrm{temp}})} & \text{if } \tilde{A}_{ij}>0,\\[8pt]
    0 & \text{otherwise},
  \end{cases}
  \label{eq:AH_softmax}
\end{equation}
where $\tau_{\mathrm{temp}}>0$ is a temperature parameter controlling the attention sharpness.

\paragraph{Step 2: Aggregated neighbour feature.}
For each ambiguous clip $i\in\mathcal{C}_{\mathrm{amb}}$, we compute the attention-weighted aggregation of its neighbours' \emph{main features} (not visual features):
\begin{equation}
  \hat{g}_i = \frac{\sum_{j=1}^{T} A_{ij}^{H}\,\hat{f}_j^{\,l,\mathrm{enh}}}{\bigl\|\sum_{j=1}^{T} A_{ij}^{H}\,\hat{f}_j^{\,l,\mathrm{enh}}\bigr\|_2}\;\in\;\mathcal{S}^{D-1}.
  \label{eq:neighbour_agg}
\end{equation}
The $\ell_2$-normalisation ensures that the aggregated feature remains on the sphere, enabling direct cosine comparison with the dominant prototypes.

\paragraph{Step 3: Prototype comparison and label assignment.}
The aggregated neighbour feature $\hat{g}_i$ is compared with the dominant prototypes by cosine similarity:
\begin{equation}
  \mathrm{sim}_{\mathrm{norm}} = \max_{m\in\{1,\ldots,N_{\mathrm{dom}}^{\mathrm{norm}}\}}\;\hat{g}_i^{\top}\boldsymbol{\mu}_{\mathrm{dom},m}^{\mathrm{norm}},\qquad
  \mathrm{sim}_{\mathrm{abn}} = \hat{g}_i^{\top}\boldsymbol{\mu}_{\mathrm{dom}}^{\mathrm{abn}}.
  \label{eq:proto_compare}
\end{equation}
The label is assigned as:
\begin{equation}
  c_i =
  \begin{cases}
    \mathrm{abn}  & \text{if } \mathrm{sim}_{\mathrm{abn}} > \mathrm{sim}_{\mathrm{norm}} + \delta_{\mathrm{margin}},\\[3pt]
    \mathrm{norm} & \text{otherwise},
  \end{cases}
  \label{eq:label_assign}
\end{equation}
where $\delta_{\mathrm{margin}}\geq 0$ is a small conservatism margin that biases toward the normal label in genuinely ambiguous cases, reflecting the prior that most clips in a surveillance video are normal.
In practice, $\delta_{\mathrm{margin}}=0.01$.

\paragraph{Fallback for isolated clips.}
If an ambiguous clip $i$ has no valid neighbours after sparsification (i.e., all entries in row $i$ of $\mathbf{A}_H$ are zero), the neighbour consensus mechanism cannot be applied.
In this case, the clip's \emph{own} main feature $\hat{f}_i^{\,l,\mathrm{enh}}$ is directly compared with the dominant prototypes:
\begin{equation}
  c_i =
  \begin{cases}
    \mathrm{abn}  & \text{if } \hat{f}_i^{\,l,\mathrm{enh}\,\top}\boldsymbol{\mu}_{\mathrm{dom}}^{\mathrm{abn}} > \max_{m}\hat{f}_i^{\,l,\mathrm{enh}\,\top}\boldsymbol{\mu}_{\mathrm{dom},m}^{\mathrm{norm}} + \delta_{\mathrm{margin}},\\[3pt]
    \mathrm{norm} & \text{otherwise}.
  \end{cases}
  \label{eq:fallback}
\end{equation}
This fallback is conservative (biased toward normal) and occurs rarely in practice ($<$2\% of ambiguous clips across all benchmarks).

\paragraph{Design rationale.}
The neighbour consensus mechanism rests on the assumption that \emph{visually similar clips within the same video are likely to share the same anomaly status}.
This is empirically well-grounded: surveillance videos exhibit strong temporal and visual continuity, so a clip depicting a fight scene is surrounded by other clips of the same fight, and a clip of normal pedestrian traffic is surrounded by similar normal clips.
By aggregating evidence from neighbours (via their main features, which encode anomaly semantics) weighted by visual similarity (from visual features, which encode scene appearance), the assignment leverages complementary information from both feature streams and is robust to the noise in any single clip's initial score.

\subsection{Score-Adaptive SLERP Derivation}
\label{app:slerp_derivation}

Once each ambiguous clip $i\in\mathcal{C}_{\mathrm{amb}}$ is assigned a label $c_i$ and a corresponding target dominant prototype $\boldsymbol{\mu}_{\mathrm{dom}}^{c_i}$, its feature is pulled along the geodesic (great-circle arc) toward the target via Spherical Linear Interpolation (SLERP)~\cite{shoemake1985animating}.

\paragraph{SLERP formula.}
For two unit vectors $\mathbf{p},\mathbf{q}\in\mathcal{S}^{D-1}$ with geodesic angle $\Omega=\arccos(\mathbf{p}^{\top}\mathbf{q})\in(0,\pi)$, the SLERP at parameter $t\in[0,1]$ is defined as:
\begin{equation}
  \mathrm{Slerp}(\mathbf{p},\mathbf{q},t)
  = \frac{\sin\!\bigl((1-t)\,\Omega\bigr)}{\sin\Omega}\,\mathbf{p}
  + \frac{\sin(t\,\Omega)}{\sin\Omega}\,\mathbf{q}.
  \label{eq:slerp_formula}
\end{equation}
One can verify that $\|\mathrm{Slerp}(\mathbf{p},\mathbf{q},t)\|=1$ for all $t\in[0,1]$ (the result always lies on $\mathcal{S}^{D-1}$), $\mathrm{Slerp}(\mathbf{p},\mathbf{q},0)=\mathbf{p}$, and $\mathrm{Slerp}(\mathbf{p},\mathbf{q},1)=\mathbf{q}$.
Furthermore, the angular velocity along the geodesic is constant: $d_{\mathrm{geo}}\!\bigl(\mathbf{p},\mathrm{Slerp}(\mathbf{p},\mathbf{q},t)\bigr)=t\,\Omega$.
When $\Omega\to 0$ (nearly identical vectors), SLERP degenerates to LERP with normalisation; our implementation handles this edge case by switching to $\ell_2$-normalised LERP when $\sin\Omega<10^{-6}$.

\paragraph{Comparison with Normalized LERP.} The naive Euclidean linear interpolation (LERP) followed by $\ell_2$-normalization (NLERP) is defined as:
\begin{equation}
\mathrm{NLerp}(p, q, t) = \frac{(1 - t) p + t q}{\|(1 - t) p + t q\|_2}.
\label{eq:nlerp}
\end{equation}
While $\mathrm{NLerp}(p, q, 0.5)$ coincides with the geodesic midpoint $\mathrm{Slerp}(p, q, 0.5)$, NLERP fails to traverse the great-circle arc at constant angular velocity for arbitrary $t \in (0, 1)$. Specifically, the angular distance traversed by NLERP from $p$ satisfies:
\begin{equation}
\cos \theta(t) = \langle p, \mathrm{NLerp}(p, q, t) \rangle = \frac{1 - t + t \cos \Omega}{\sqrt{(1-t)^2 + t^2 + 2t(1-t)\cos \Omega}}.
\end{equation}
The discrepancy between the constant-speed geodesic arc $\mathrm{Slerp}(p, q, t)$ and NLERP is given by $\Delta \theta(t) = |t\Omega - \theta(t)|$. Expanding $\Delta \theta(t)$ around $t = 0.5$ reveals an intrinsic angular acceleration:
\begin{equation}
\Delta \theta(t) = \frac{\Omega^3}{12} \left| t - \frac{1}{2} \right| \left( t - t^2 \right) + \mathcal{O}(\Omega^5).
\label{eq:slerp_vs_nlerp_error}
\end{equation}
Crucially, NLERP projects Euclidean secant chords back onto the hypersphere, causing uneven angular progression where features accelerate near midpoints and decelerate near endpoints. SLERP eliminates this chordal distortion by construction, strictly preserving constant geodesic velocity and manifold constraints on $\mathcal{S}^{D-1}$.

\paragraph{Score-adaptive pull strength.}
A uniform pull strength $\beta$ for all ambiguous clips would be suboptimal: a clip with $s_i^{\mathrm{init}}=0.50$ (maximally uncertain) should receive a stronger correction than a clip with $s_i^{\mathrm{init}}=0.42$ (mildly uncertain, leaning normal).
We define a normalised distance-from-boundary measure:
\begin{equation}
  \hat{d}_i =
  \begin{cases}
    \displaystyle\frac{s_i^{\mathrm{init}} - 0.5}{\rho_{\mathrm{high}} - 0.5} & \text{if } s_i^{\mathrm{init}}\geq 0.5,\\[10pt]
    \displaystyle\frac{0.5 - s_i^{\mathrm{init}}}{0.5 - \rho_{\mathrm{low}}} & \text{if } s_i^{\mathrm{init}}<0.5.
  \end{cases}
  \label{eq:dhat}
\end{equation}
By construction, $\hat{d}_i\in[0,1]$: $\hat{d}_i=0$ when $s_i^{\mathrm{init}}=0.5$ (maximum ambiguity) and $\hat{d}_i=1$ when $s_i^{\mathrm{init}}\in\{\rho_{\mathrm{low}},\rho_{\mathrm{high}}\}$ (boundary of the ambiguity interval, minimum ambiguity).
The SLERP interpolation parameter for clip $i$ is then:
\begin{equation}
  \beta_i = \beta_{\mathrm{base}}\times\bigl(1-\tfrac{1}{2}\hat{d}_i\bigr),
  \label{eq:beta_adaptive}
\end{equation}
where $\beta_{\mathrm{base}}\in(0,1)$ is a global hyperparameter controlling the maximum pull strength.
The factor $(1-\frac{1}{2}\hat{d}_i)$ ensures:
\begin{itemize}[nosep,leftmargin=1.2em]
  \item $\hat{d}_i=0$ (most ambiguous, $s\approx 0.5$): $\beta_i=\beta_{\mathrm{base}}$ (full pull strength).
  \item $\hat{d}_i=1$ (least ambiguous, $s\approx\rho_{\mathrm{low}}$ or $\rho_{\mathrm{high}}$): $\beta_i=\frac{1}{2}\beta_{\mathrm{base}}$ (half pull strength).
\end{itemize}
The pull strength thus decreases linearly with decreasing uncertainty, preventing over-correction of clips whose initial scores already lean toward the correct label.
The factor $\frac{1}{2}$ prevents the pull from vanishing entirely at the interval boundaries, ensuring that even mildly ambiguous clips receive a non-trivial correction.

\paragraph{Feature pulling.}
The pulled feature for each ambiguous clip is computed as:
\begin{equation}
  \hat{f}_i^{\,\mathrm{pulled}}
  = \mathrm{Slerp}\!\bigl(\hat{f}_i^{\,l,\mathrm{enh}},\;\boldsymbol{\mu}_{\mathrm{target}}^{c_i},\;\beta_i\bigr),
  \label{eq:pull}
\end{equation}
where $\boldsymbol{\mu}_{\mathrm{target}}^{c_i}$ is the dominant prototype corresponding to the assigned label $c_i$.
For clips assigned $c_i=\mathrm{abn}$, the target is $\boldsymbol{\mu}_{\mathrm{dom}}^{\mathrm{abn}}$.
For clips assigned $c_i=\mathrm{norm}$, the target is the nearest dominant normal prototype: $\boldsymbol{\mu}_{\mathrm{target}}^{\mathrm{norm}}=\arg\max_{m}\hat{f}_i^{\,l,\mathrm{enh}\,\top}\boldsymbol{\mu}_{\mathrm{dom},m}^{\mathrm{norm}}$.
The pulled feature lies on $\mathcal{S}^{D-1}$ by the SLERP guarantee, strictly preserving the manifold constraint.
Clips in $\mathcal{C}_{\mathrm{norm}}$ and $\mathcal{C}_{\mathrm{abn}}$ retain their original features: $\hat{f}_i^{\,\mathrm{pulled}}=\hat{f}_i^{\,l,\mathrm{enh}}$ for $i\notin\mathcal{C}_{\mathrm{amb}}$.

\paragraph{Geometric interpretation.}
Figure~\ref{fig:qualitative_slerp}(c--d) in the main text visualises the effect of SGP on the sphere.
Before pulling, ambiguous features scatter between the normal and anomalous prototype clusters, often with incorrect implicit labels.
After pulling, each ambiguous feature migrates along the geodesic toward its assigned prototype by an arc length of $\beta_i\cdot\Omega_i$, where $\Omega_i=d_{\mathrm{geo}}(\hat{f}_i^{\,l,\mathrm{enh}},\boldsymbol{\mu}_{\mathrm{target}}^{c_i})$.
This produces a clearly bimodal feature distribution on the sphere, which translates directly into a bimodal score distribution with reduced overlap---as confirmed both quantitatively (Table~\ref{tab:ablation_module}, M2$\to$M3) and qualitatively (score curves in Figure~\ref{fig:qualitative_slerp}(a--b)).

\paragraph{Why only ambiguous clips are pulled.}
Pulling clips with confident classifications ($s\ll\rho_{\mathrm{low}}$ or $s\gg\rho_{\mathrm{high}}$) is both unnecessary and potentially harmful.
These clips already have strong initial scores, and pulling them toward prototypes could \emph{over-regularise} their features, erasing fine-grained distinctions (e.g., between different intensities of the same anomaly type).
The tri-classification mechanism ensures that SGP intervenes only where it is needed---on the decision boundary---while leaving well-separated clips undisturbed.

\subsection{Complete Algorithm}
\label{app:sgp_algorithm}

Algorithm~\ref{alg:sgp_full} presents the complete SGP procedure for a single video, corresponding to pipeline Stage~(d) in Figure~\ref{fig:pipeline}.

\begin{algorithm}[htbp]
\caption{vMF-Guided Spherical Geodesic Pulling (SGP) --- Single Video}
\label{alg:sgp_full}
\footnotesize
\begin{algorithmic}[1]
\REQUIRE $\{\hat{f}_i^{\,l,\mathrm{enh}}\}_{i=1}^{T},\;\{\hat{f}_i^{v}\}_{i=1}^{T}\subset\mathcal{S}^{D-1}$;\;scores $\{s_i^{\mathrm{init}}\}\subset[0,1]$;\;
         prototypes $\{\boldsymbol{\mu}_k^{\mathrm{norm}}\}_{k=1}^{K_N},\;\{\boldsymbol{\mu}_k^{\mathrm{abn}}\}_{k=1}^{K_A}$;\;
         params $\beta_{\mathrm{base}}, r_{\min}, r_{\max}, \lambda_r, \tau_r, \gamma_{\min}, \delta_{\mathrm{margin}}, \tau_H, K_H, \tau_{\mathrm{temp}}$
\ENSURE  $\{s_i^{\mathrm{final}}\}_{i=1}^{T}\subset[0,1]$
\STATE \textbf{--- Stage1: MAD-Based Adaptive Tri-Classification ---}
\STATE $\tilde{s}\!\leftarrow\!\mathrm{median}(\{s_i^{\mathrm{init}}\})$;\;
       $\mathrm{MAD}\!\leftarrow\!\mathrm{median}(\{|s_i^{\mathrm{init}}\!-\!\tilde{s}|\})$
\STATE $r\!\leftarrow\!r_{\min}+(r_{\max}\!-\!r_{\min})\cdot\sigma(-\lambda_r(\mathrm{MAD}\!-\!\tau_r))$;\;
       $\rho_{\mathrm{lo}}\!\leftarrow\!\max(0,0.5\!-\!r)$;\;
       $\rho_{\mathrm{hi}}\!\leftarrow\!\min(1,0.5\!+\!r)$
\STATE $\mathcal{C}_{\mathrm{N}}\!\leftarrow\!\{i:s_i<\rho_{\mathrm{lo}}\}$;\;
       $\mathcal{C}_{\mathrm{A}}\!\leftarrow\!\{i:s_i>\rho_{\mathrm{hi}}\}$;\;
       $\mathcal{C}_{\mathrm{amb}}\!\leftarrow\!\{1,\dots,T\}\setminus(\mathcal{C}_{\mathrm{N}}\!\cup\!\mathcal{C}_{\mathrm{A}})$
\STATE \textbf{--- Stage 2: Dominant Prototype Voting ---}
\STATE $n_{\min}^{\mathrm{abn}}\!\leftarrow\!\max(3,\lfloor\gamma_{\min} T\rfloor)$
\IF{$|\mathcal{C}_{\mathrm{A}}|<n_{\min}^{\mathrm{abn}}$}
    \STATE Treat video as fully normal; pull all $\mathcal{C}_{\mathrm{amb}}$ toward dominant normal prototype(s), score via Eq.~\eqref{eq:vmf_score}, \textbf{return}
\ENDIF
\STATE Each $i\!\in\!\mathcal{C}_{\mathrm{A}}$ votes $k_i^{\mathrm{abn}}\!=\!\arg\max_k \hat{f}_i^{\,l,\mathrm{enh}\top}\!\boldsymbol{\mu}_k^{\mathrm{abn}}$;\;
       $\boldsymbol{\mu}_{\mathrm{dom}}^{\mathrm{abn}}\!\leftarrow\!$ majority winner
\STATE Each $i\!\in\!\mathcal{C}_{\mathrm{N}}$ votes similarly;\;
       retain top-2 normal prototypes $\boldsymbol{\mu}_{\mathrm{dom},1}^{\mathrm{norm}}$, $\boldsymbol{\mu}_{\mathrm{dom},2}^{\mathrm{norm}}$
\STATE \textbf{--- Stage 3: Neighbour Consensus Assignment ---}
\STATE Build sparse intra-video attention $\mathbf{A}_H$ from $\{\hat{f}_i^v\}$ (threshold $\tau_H$, top-$K_H$, softmax)
\FOR{each $i\in\mathcal{C}_{\mathrm{amb}}$}
    \IF{$\sum_j A_{ij}^H>0$}
        \STATE $\hat{g}_i\!\leftarrow\!(\sum_j A_{ij}^H \hat{f}_j^{\,l,\mathrm{enh}})/\|\cdot\|_2$
        \hfill\COMMENT{neighbour aggregate}
    \ELSE
        \STATE $\hat{g}_i\!\leftarrow\!\hat{f}_i^{\,l,\mathrm{enh}}$
        \hfill\COMMENT{self-fallback}
    \ENDIF\STATE $c_i\!\leftarrow\!\mathrm{abn}$ if $\hat{g}_i^\top\boldsymbol{\mu}_{\mathrm{dom}}^{\mathrm{abn}}>\max_m \hat{g}_i^\top\boldsymbol{\mu}_{\mathrm{dom},m}^{\mathrm{norm}}+\delta_{\mathrm{margin}}$,\; else $c_i\!\leftarrow\!\mathrm{norm}$
\ENDFOR
\STATE \textbf{--- Stage 4: Score-Adaptive SLERP Pulling ---}
\FOR{each $i\in\mathcal{C}_{\mathrm{amb}}$}
    \STATE $\hat{d}_i$ via Eq.~\eqref{eq:dhat};\;
           $\beta_i\!\leftarrow\!\beta_{\mathrm{base}}(1\!-\!\tfrac{1}{2}\hat{d}_i)$;\;
           $\Omega_i\!\leftarrow\!\arccos(\hat{f}_i^{\,l,\mathrm{enh}\top}\!\boldsymbol{\mu}_{\mathrm{target}}^{c_i})$
    \STATE $\hat{f}_i^{\,\mathrm{pulled}}\!\leftarrow\!\mathrm{Slerp}(\hat{f}_i^{\,l,\mathrm{enh}},\;\boldsymbol{\mu}_{\mathrm{target}}^{c_i},\;\beta_i)$
           \hfill\COMMENT{Eq.~\eqref{eq:slerp_formula}; skip if $\sin\Omega_i<10^{-6}$}
\ENDFOR
\STATE Non-ambiguous clips: $\hat{f}_i^{\,\mathrm{pulled}}\!\leftarrow\!\hat{f}_i^{\,l,\mathrm{enh}}$ for $i\notin\mathcal{C}_{\mathrm{amb}}$
\STATE \textbf{--- Final Scoring ---}
\FOR{each $i\in\{1,\dots,T\}$}
    \STATE $d_{\mathrm{norm}}\!=\!\min_k\arccos(\hat{f}_i^{\,\mathrm{pulled}\top}\!\boldsymbol{\mu}_k^{\mathrm{norm}})$;\;
           $d_{\mathrm{abn}}\!=\!\min_k\arccos(\hat{f}_i^{\,\mathrm{pulled}\top}\!\boldsymbol{\mu}_k^{\mathrm{abn}})$
    \STATE $s_i^{\mathrm{final}}\!\leftarrow\!\sigma(\kappa(d_{\mathrm{norm}}\!-\!d_{\mathrm{abn}}))$
           \hfill\COMMENT{Eq.~\eqref{eq:vmf_score}}
\ENDFOR
\RETURN $\{s_i^{\mathrm{final}}\}_{i=1}^{T}$
\end{algorithmic}
\end{algorithm}

\paragraph{Computational complexity.}
For a video with $T$ clips and feature dimension $D$, the dominant costs are:
(i)~the intra-video attention matrix $\mathbf{A}_H$: $O(T^2 D)$ for pairwise cosine similarities, reduced to $O(T\cdot K_H\cdot D)$ after top-$K_H$ sparsification;
(ii)~neighbour aggregation: $O(|\mathcal{C}_{\mathrm{amb}}|\cdot K_H\cdot D)$;
(iii)~SLERP computation: $O(|\mathcal{C}_{\mathrm{amb}}|\cdot D)$.
In practice, $|\mathcal{C}_{\mathrm{amb}}|\ll T$ (typically10--25\% of clips), $K_H\leq 20$, and $D=4096$, so the total SGP cost per video is negligible compared to the MLLM forward pass.
For the largest videos in XD-Violence ($T\approx 1{,}800$ clips), SGP completes in $<$0.3\,seconds on a single CPU core.

\paragraph{UBnormal exception.}
As noted in Table~\ref{tab:ablation_module}, SGP is not applied to UBnormal.
The 211 test videos in UBnormal average only ${\sim}12$ seconds each (${\sim}10$--$14$ clips per video), providing insufficient temporal context for the intra-video tri-classification and neighbour consensus mechanisms that SGP relies on.
With so few clips per video, the MAD-based ambiguity interval becomes unreliable (the median and MAD of 5--8 values are highly variable), and the intra-video attention matrix degenerates to a near-complete graph where neighbour consensus offers no advantage over self-assignment.
Consequently, M2 and M3 produce identical results on UBnormal.
\section{Computational Cost Analysis}
\label{app:compute}

A key advantage of SphereVAD over existing training-free VAD methods is that it requires \emph{no LLM or VLM inference calls at test time}---all computation operates on pre-extracted intermediate-layer features via lightweight geometric operations. This section provides a comprehensive wall-clock time and storage analysis, covering both \textbf{offline} (batch processing) and \textbf{online} (streaming) deployment modes.

\subsection{Offline Inference}
\label{app:compute:offline}

In the offline setting, SphereVAD processes an entire test set in two sequential phases: (1)~a one-time feature extraction pass through the frozen MLLM, and (2)~the geometric inference pipeline (Stages S1--S6). Both phases are fully parallelisable and require no gradient computation.

\subsubsection{Feature Extraction (One-Time Cost)}

Table~\ref{tab:feat_extract_time} reports the wall-clock time for extracting dual-stream features ($\hat{f}^l$ and $\hat{f}^v$) from the frozen Qwen3.5 backbone across all datasets, including the synthetic calibration set. Each clip produces two $\ell_2$-normalised vectors of dimension $D{=}4096$.

\begin{table}[h]
\centering
\caption{Feature extraction time and storage cost. Times are reported for single-GPU (1$\times$A100-80GB) and multi-GPU (4$\times$A100-80GB) configurations. Storage is computed as $\text{Clips} \times 2 \times D \times 4~\text{bytes (float32)}$ using standard decimal GB.}
\label{tab:feat_extract_time}
\begin{tabular}{lrrrr}
\toprule
\textbf{Dataset} & \textbf{Clips} & \textbf{1$\times$A100} & \textbf{4$\times$A100} & \textbf{Storage} \\
\midrule
SYN (calibration) &   2,040 &   5.5 min &  1.5 min & 0.07 GB \\
UCF-Crime         &  46,410 &   2.1 h   & 33.8 min & 1.52 GB \\
XD-Violence       &  97,396 &   4.3 h   & 70.9 min & 3.19 GB \\
UBnormal          &   3,912 &  10.5 min &  2.9 min & 0.13 GB \\
\midrule
\textbf{Total}    & 149,758 & ${\sim}$6.7 h & ${\sim}$1.8 h (109.1 min) & 4.91 GB \\
\bottomrule
\end{tabular}
\end{table}

The per-clip storage cost is:
\begin{equation}
    2 \;\text{(streams)} \times 4096 \;\text{(dim)} \times 4 \;\text{(bytes/float32)} = 32{,}768 \;\text{bytes} \approx 32\;\text{KB/clip}.
\end{equation}
Even for the largest benchmark (XD-Violence, ${\sim}$97K clips), the total feature storage is only ${\sim}$3.19\,GB, easily fitting in CPU memory. With 4$\times$A100 GPUs and asynchronous prefetching (Appendix~C.4), the complete extraction for all benchmarks plus calibration data finishes in under 2~hours---a \emph{one-time} cost that is amortised over all subsequent experiments, hyperparameter sweeps, and ablation studies.

\subsubsection{Geometric Inference Pipeline}

Once features are extracted, the six-stage geometric inference pipeline (Stages S1--S6; cf.\ Figure~2) operates entirely on CPU with no GPU requirement. Table~\ref{tab:stage_time} provides a per-stage wall-clock breakdown measured on a single CPU core (Intel Xeon Platinum 8375C).

\begin{table}[h]
\centering
\caption{Per-stage wall-clock time (seconds) for the geometric inference pipeline. All stages run on a single CPU core after features have been loaded into memory. Data loading time is excluded.}
\label{tab:stage_time}
\begin{tabular}{lccc}
\toprule
\textbf{Stage} & \textbf{XD-Violence} & \textbf{UCF-Crime} & \textbf{UBnormal} \\
\midrule
S1: Fr\'{e}chet Mean + Spherical Centering       & 0.591\,s & 0.219\,s & 0.226\,s \\
S2: vMF Prototype Construction (K-Means)         & 0.480\,s & 0.362\,s & 0.388\,s \\
S3: HSA (Cross-Video Holistic Attention)          & 9.218\,s & 2.088\,s & 0.047\,s \\
S4: Initial vMF Scoring                          & 0.119\,s & 0.040\,s & 0.028\,s \\
S5: SGP (Spherical Geodesic Pulling)              & 2.528\,s & 0.735\,s & ---$^\dagger$ \\
S6: Final Scoring + Gaussian Smoothing            & 2.957\,s & 1.032\,s & 0.288\,s \\
\midrule
\textbf{Total Inference}                          & \textbf{15.893\,s} & \textbf{4.476\,s} & \textbf{0.977\,s} \\
\midrule
Data Loading (excluded)                           & 5.08\,s  & 3.50\,s  & 2.00\,s  \\
\bottomrule
\multicolumn{4}{@{}l}{\footnotesize $^\dagger$ SGP is disabled on UBnormal due to short average video duration.}
\end{tabular}
\end{table}

\paragraph{Key observations.}
\begin{itemize}[nosep,leftmargin=1.5em]
    \item The \emph{entire} geometric inference pipeline completes in \textbf{under 16 seconds} on the largest benchmark (XD-Violence, ${\sim}$97K clips), and under \textbf{5 seconds} on UCF-Crime (${\sim}$46K clips). This is orders of magnitude faster than any LLM/VLM-based scoring method.
    \item The dominant cost is \textbf{HSA} (Stage S3), which constructs the cross-video sparse attention matrix. Its $O(N_{\text{total}}^2 \cdot D)$ pairwise cosine similarity computation scales quadratically with the number of test clips, accounting for ${\sim}$58\% of the total time on XD-Violence. For UBnormal (${\sim}$3.9K clips), HSA is negligible (0.047\,s).
    \item \textbf{SGP} (Stage S5) is applied per-video and scales linearly with the number of clips per video; its total time across all test videos is modest ($<$3\,s even on XD-Violence).
    \item All stages are \emph{deterministic} (no stochastic sampling) and require \emph{no GPU}: the pipeline runs on a single CPU core, making it deployable in resource-constrained settings.
\end{itemize}

\subsubsection{Comparison with Existing Training-Free Methods}

Table~\ref{tab:cost_comparison} compares the computational profile of SphereVAD with three representative training-free VAD methods across multiple efficiency dimensions. All timings are normalised to the UCF-Crime test set for fair comparison.

\begin{table}[h]
\centering
\caption{Computational cost comparison among training-free VAD methods on the UCF-Crime test set. ``LLM/VLM Calls'' counts the number of large-model inference queries required per test video. SphereVAD's inference pipeline requires \emph{zero} such calls, operating entirely on pre-extracted features.}
\label{tab:cost_comparison}
\resizebox{\textwidth}{!}{%
\begin{tabular}{lccccc}
\toprule
\textbf{Method} & \textbf{LLM/VLM Calls} & \textbf{GPU Inference} & \textbf{Post-Extraction} & \textbf{AUC (\%)} \\
 & \textbf{(per video)} & \textbf{(total GPU-hours)} & \textbf{Time} & \\
\midrule
PANDA~\cite{yang2025panda}        & Multiple (VLM + MLLM, iterative reflection) & N/A$^\dagger$ & --- & 84.89 \\
 & \multicolumn{3}{l}{\quad\textit{Average speed: 0.82 FPS on A6000}} & \\
\midrule
VADTree~\cite{li2025vadtree}       & VLM caption + LLM scoring per node           & 62.3 h (2$\times$3090)  & --- & 84.74 \\
\midrule
Unified-VAD~\cite{lin2025unified}  & 1 VLM + 1--2 LLM per 16 frames           & N/A$^\dagger$ & --- & 84.28 \\
  & \multicolumn{3}{l}{\quad\textit{Amortised: 0.029 s/frame on 2$\times$3090}} & \\
\midrule
\textbf{SphereVAD (Ours)}  & \textbf{0} (feature extraction only) & 2.25 GPU-h (33.8 min on 4$\times$A100)$^\ddagger$ & \textbf{4.5\,s (CPU)} & \textbf{86.38} \\
\bottomrule
\end{tabular}%
}
\begin{flushleft}
\footnotesize
$^\dagger$Exact total GPU-hours not reported; per-frame/FPS timings are cited from the original papers. \\
$^\ddagger$Feature extraction on UCF-Crime is a one-time cost ($33.8\,\text{min} \approx 0.563\,\text{h} \times 4\,\text{GPUs} \approx 2.25\,\text{GPU-hours}$); subsequent experiments (ablations, hyperparameter sweeps) reuse cached features with zero GPU cost.
\end{flushleft}
\end{table}

\paragraph{Architectural efficiency advantage.}
The fundamental computational distinction between SphereVAD and all prior training-free methods lies in the \emph{absence of test-time LLM/VLM inference calls}. Methods such as PANDA~\cite{yang2025panda}, VADTree~\cite{li2025vadtree}, and Unified-VAD~\cite{lin2025unified} invoke large foundation models (VLMs for captioning/perception, LLMs for scoring/reasoning) repeatedly during inference---often multiple times per video clip---incurring substantial GPU cost that scales linearly (or worse) with the number of test clips. PANDA further compounds this with iterative reflection loops that invoke external tools, reducing its throughput to $<$1\,FPS. VADTree requires over 62 GPU-hours for the VLM/LLM pipeline alone on UCF-Crime.

In contrast, SphereVAD performs a \emph{single} forward pass through the frozen MLLM to extract intermediate-layer features (a one-time cost of 33.8\,min on 4$\times$A100 for UCF-Crime), after which \emph{all subsequent computation}---spherical centering, prototype construction, HSA, vMF scoring, and SGP---is performed via closed-form geometric operations on pre-extracted features, completing in $<$5\,seconds on a single CPU core. This separation of feature extraction from geometric inference yields two practical benefits: (i)~the one-time extraction cost is amortised across all downstream experiments, enabling rapid iteration during development; and (ii)~the inference pipeline itself is lightweight enough to run on commodity hardware without GPU access.

\subsection{Online Inference}
\label{app:compute:online}

For real-time deployment scenarios (e.g., live surveillance feeds), SphereVAD supports a streamlined \textbf{online inference mode} that trades a small amount of detection accuracy for substantial latency reduction. The online pipeline simplifies the offline framework as follows:

\begin{enumerate}[nosep,leftmargin=1.5em]
    \item \textbf{Calibration from synthetic data only.} The Fr\'{e}chet mean $\mu_{\text{unified}}$ is computed using \emph{only} the synthetic calibration features $\{\tilde{f}^l_{\text{syn}}\}$, without pooling real test-set features. This eliminates the need to accumulate statistics over the test set before inference begins---the spherical reference frame is fixed \emph{a priori}, ensuring a strictly inductive, real-time setting.
    \item \textbf{Prototypes from synthetic data only.} The vMF prototypes $\{\mu^{\text{norm}}_k\}$, $\{\mu^{\text{abn}}_k\}$ are constructed via spherical $K$-Means on the centered synthetic features, identical to the offline pipeline. No test-set statistics are required.
    \item \textbf{Direct vMF scoring (no HSA, no SGP).} Each incoming clip is independently scored via the vMF likelihood-ratio criterion (Eq.~(4)) immediately after feature extraction and spherical centering. The cross-video HSA module (which requires access to other test clips) and the intra-video SGP module (which requires a complete video sequence) are both disabled, as they are inherently batch operations.
\end{enumerate}

\noindent Note that the online mode is a synthetic-only variant that differs from configuration \textbf{M1} in Table~2 of the main text: M1 uses the unified Fr\'{e}chet mean computed from \emph{both} synthetic and real features (Eq.~(1)), whereas the online mode computes the Fr\'{e}chet mean from synthetic data \emph{only}. The online mode therefore constitutes a strictly more constrained, single-clip setting than M1.

\paragraph{Throughput analysis.}
Under this configuration, per-clip inference consists of three operations: (i)~a single MLLM forward pass for dual-stream feature extraction, (ii)~logarithmic-map centering and $\ell_2$-normalisation (Eq.~(2)), and (iii)~vMF scoring via Eq.~(4). Operations (ii) and (iii) take negligible runtime ($<$0.1\,ms per clip). The throughput is therefore dominated by the forward pass.

On a single A100-80GB GPU with the Qwen3.5 backbone (processing an isolated single clip formatted as a $2{\times}2$ composite frame), the forward pass takes approximately ${\sim}$67.5\,ms per clip without disk write or offline feature dumping. Assuming a standard surveillance video frame rate of 24\,FPS and an input clip sampling of 4 frames:
\begin{equation}
    \text{FPS}_{\text{online}} = \frac{4\;\text{frames/clip}}{0.0675\;\text{s/clip}} \approx \textbf{59.3\;\text{FPS}},
\end{equation}
which exceeds real-time requirements (24\,FPS) by a factor of ${\sim}$2.5$\times$, leaving sufficient headroom for video decoding and streaming I/O.

\paragraph{Online vs.\ offline accuracy.}
The online mode yields 81.37\% AP on XD-Violence, 78.31\% AUC on UCF-Crime, and 71.62\% AUC on UBnormal. The accuracy gap relative to the full offline pipeline (M3: 86.99\% AP, 86.38\% AUC, 76.46\% AUC) reflects the combined effect of: (i)~relying solely on synthetic calibration without target-domain centering, which leaves a residual domain gap unabsorbed, and (ii)~disabling HSA and SGP, which provide cross-video consistency and intra-video boundary refinement, respectively. For reference, configuration M1 in Table~2---which retains the unified (synthetic\,+\,real) Fr\'{e}chet mean but likewise disables HSA and SGP---achieves 82.51\% AP on XD-Violence, 81.68\% AUC on UCF-Crime, and 75.98\% AUC on UBnormal, illustrating the expected cost of synthetic-only inductive centering.

Nonetheless, on XD-Violence, this lightweight online configuration substantially outperforms prior training-free zero-shot baselines (e.g., LAVAD at 62.01\% AP and EventVAD at 64.04\% AP by over $+17$\,pp). On UCF-Crime, while lower than offline multi-prompt architectures like LAVAD (80.28\%), it maintains competitive streaming utility at $>59$\,FPS, contrasting sharply with autoregressive or agent-based video LLM approaches (e.g., PANDA) that typically operate below 1\,FPS.

\paragraph{Summary.}
Table~\ref{tab:online_vs_offline} contrasts the offline and online deployment modes.

\begin{table}[h]
\centering
\caption{Offline vs.\ online deployment modes of SphereVAD.}
\label{tab:online_vs_offline}
\resizebox{\textwidth}{!}{%
\begin{tabular}{lcccccc}
\toprule
\textbf{Mode} & \textbf{Fr\'{e}chet Mean} & \textbf{Prototypes} & \textbf{HSA} & \textbf{SGP} & \textbf{Throughput} & \textbf{XD-Vio / UCF / UBn} \\
\midrule
Offline (M3) & Syn + Real & Syn & \cmark & \cmark$^*$ & 4.5\,s total (CPU) & 86.99 / 86.38 / 76.46 \\
Offline (M1) & Syn + Real & Syn & \xmark & \xmark     & 4.5\,s total (CPU) & 82.51 / 81.68 / 75.98 \\
Online       & Syn only   & Syn & \xmark & \xmark     & 59.3 FPS (1$\times$A100) & 81.37 / 78.31 / 71.62 \\
\bottomrule
\end{tabular}%
}
\begin{flushleft}
\footnotesize
$^*$SGP is disabled for UBnormal due to short video duration (see Appendix~D).\\
Metrics: XD-Violence = AP(\%); UCF-Crime and UBnormal = AUC(\%).\\
M1 uses the unified (Syn\,+\,Real) Fr\'{e}chet mean but disables HSA and SGP; the online mode further restricts the Fr\'{e}chet mean to synthetic data only, enabling single-clip streaming without any test-set statistics.
\end{flushleft}
\end{table}

\section{Additional Visualizations}
\label{app:add_vis}
\subsection{Additional Frame-Level Score Curves}
\label{app:add_curves}

To complement the examples in the main text (Figure~4), we present four additional pairs of abnormal and normal videos from the XD-Violence and UCF-Crime datasets in Figure~\ref{fig:add_frame_scores}. Each pair consists of an upper panel showing an abnormal video with ground-truth annotations (shaded regions) and a lower panel showing a normal video. Representative frames are displayed alongside their textual descriptions. The anomaly score curves demonstrate that our method consistently assigns high scores within annotated violent segments while maintaining low scores for normal content.

\begin{figure}[htbp]
    \centering
    \begin{subfigure}[t]{0.48\textwidth}
        \centering
        \includegraphics[width=\linewidth]{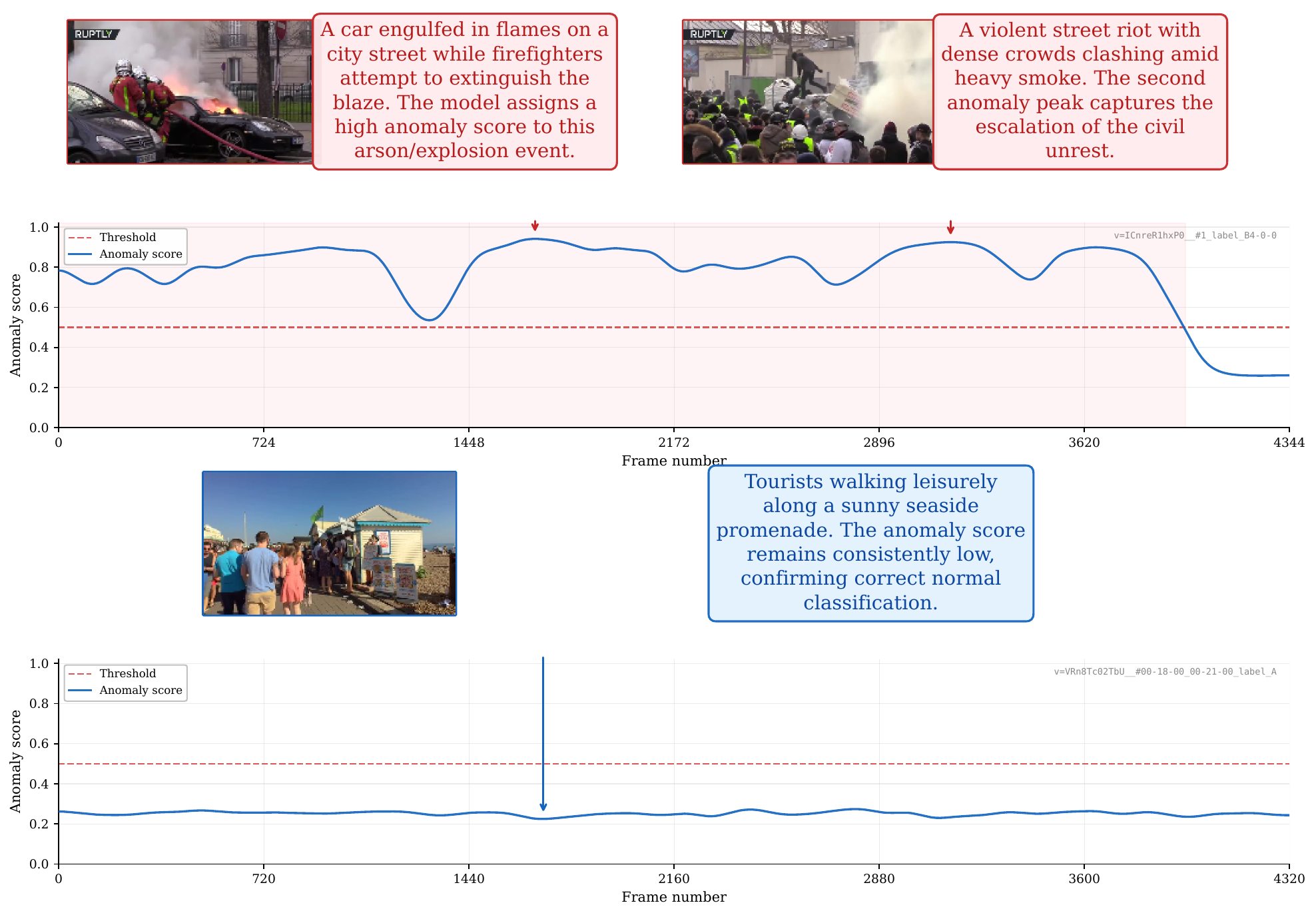}
        \caption{%
            \textit{Top:} Car fire and street riot with two distinct score peaks.
            \textit{Bottom:} Tourists at a seaside promenade (normal).
        }
        \label{fig:frame_score_pair1}
    \end{subfigure}
    \hfill
    \begin{subfigure}[t]{0.48\textwidth}
        \centering
        \includegraphics[width=\linewidth]{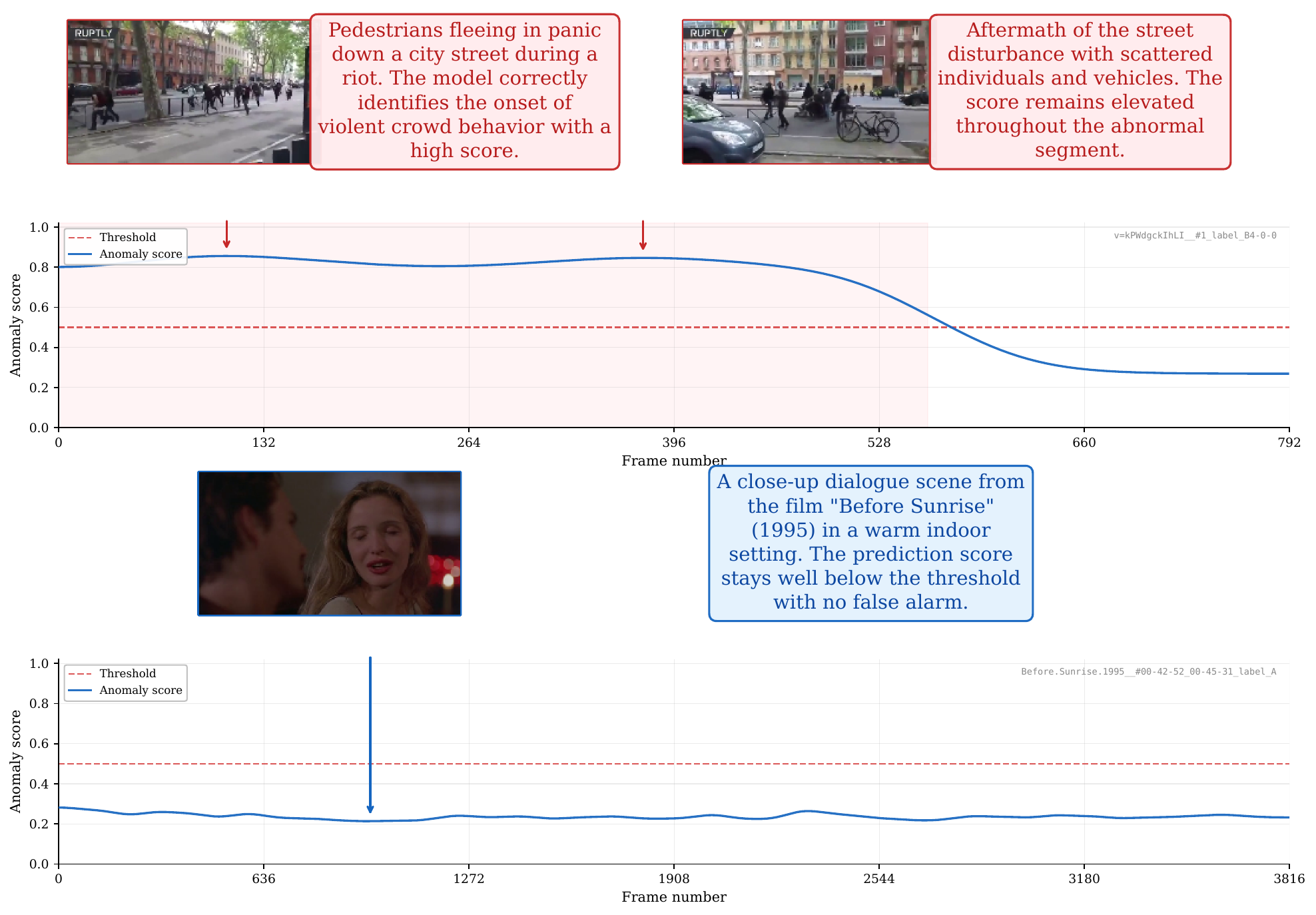}
        \caption{%
            \textit{Top:} Street riot with pedestrians fleeing in panic.
            \textit{Bottom:} Indoor dialogue from \textit{Before Sunrise} (normal).
        }
        \label{fig:frame_score_pair2}
    \end{subfigure}

    \vspace{0.8em}

    \begin{subfigure}[t]{0.48\textwidth}
        \centering
        \includegraphics[width=\linewidth]{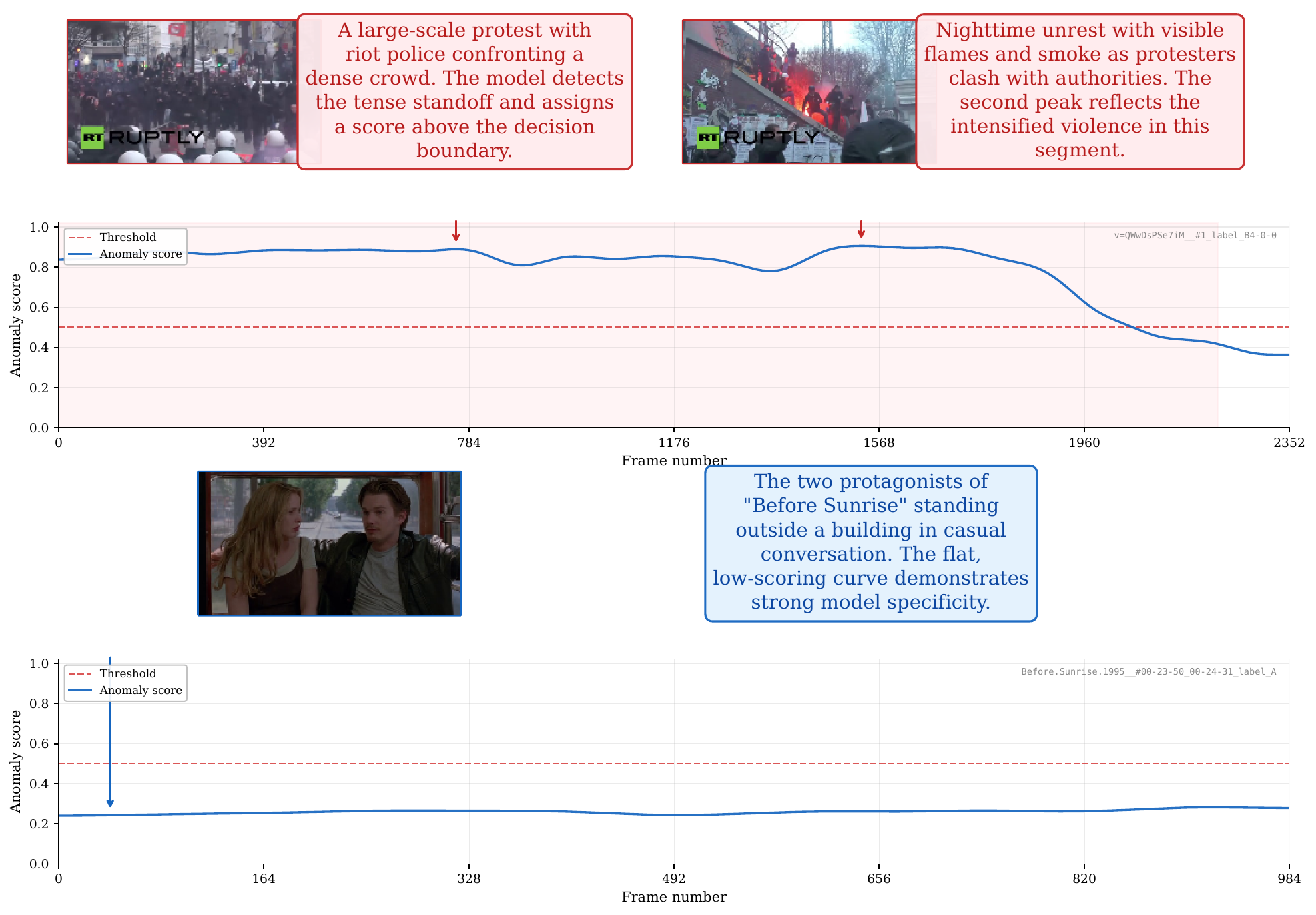}
        \caption{%
            \textit{Top:} Large-scale protest escalating into nighttime unrest.
            \textit{Bottom:} Outdoor conversation from \textit{Before Sunrise} (normal).
        }
        \label{fig:frame_score_pair3}
    \end{subfigure}
    \hfill
    \begin{subfigure}[t]{0.48\textwidth}
        \centering
        \includegraphics[width=\linewidth]{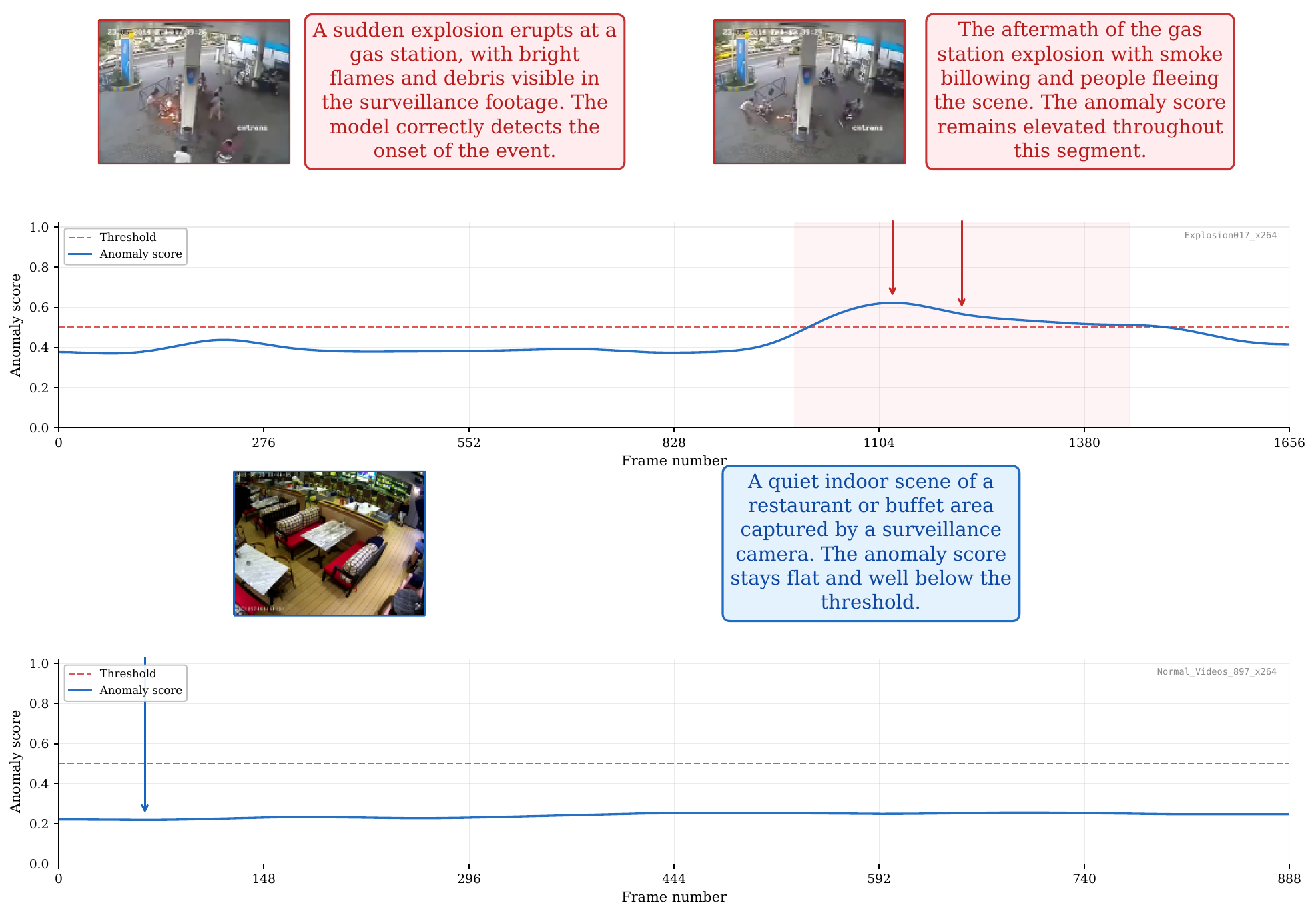}
        \caption{%
            \textit{Top:} Gas station explosion with flames and fleeing pedestrians.
            \textit{Bottom:} Quiet indoor surveillance of a restaurant area (normal).
        }
        \label{fig:frame_score_pair4}
    \end{subfigure}

    \caption{%
        \textbf{Additional frame-level anomaly score curves} on XD-Violence (a--c) and UCF-Crime (d).
        Each sub-figure shows an abnormal video (upper panel, with GT shaded) paired with a normal video (lower panel).
        The model produces high scores aligned with annotated violent segments while maintaining low scores for normal content.
    }
    \label{fig:add_frame_scores}
\end{figure}
\subsection{Additional SLERP Sphere Visualizations}
\label{app:add_slerp}

To provide a comprehensive view of the SLERP geodesic pulling effect on the hypersphere, we render the same set of ambiguous clips from three different viewpoints in Figure~\ref{fig:add_slerp}. In each sub-figure, the left panel shows the ambiguous clip features \emph{before} SLERP pulling (i.e., at the M2 stage), while the right panel shows the same features \emph{after} being pulled toward their nearest vMF prototypes (M3 stage). All panels share the same PCA projection and color scale, where blue indicates low anomaly scores (normal) and red indicates high anomaly scores (abnormal). The visualizations confirm that SLERP consistently resolves ambiguous features by displacing them toward the correct prototype region, resulting in a clearer separation between normal and abnormal clusters across all viewing angles.

\begin{figure}[htbp]
    \centering
    \begin{subfigure}[t]{0.92\textwidth}
        \centering
        \includegraphics[width=\linewidth]{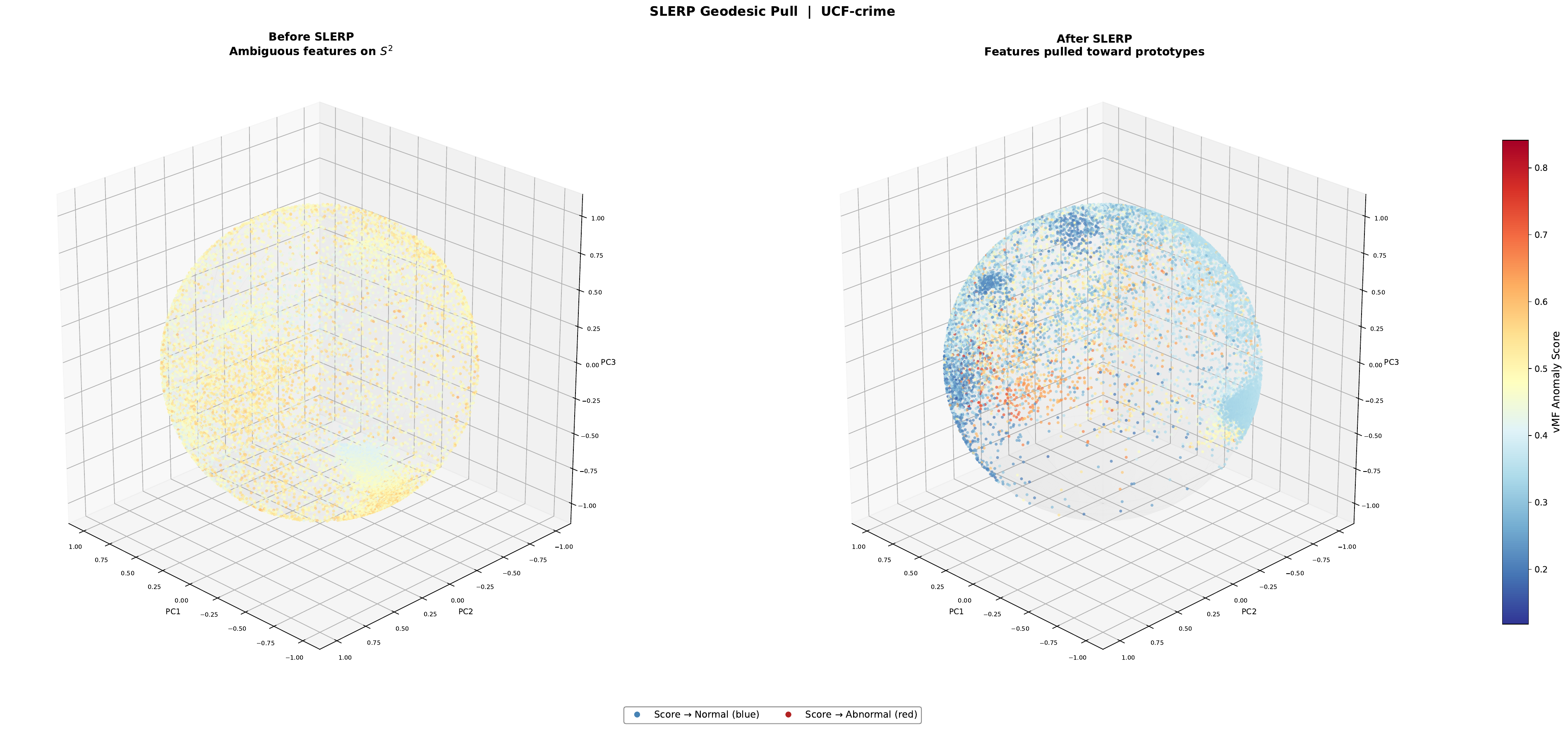}
        \caption{%
            Default view (elev=$25^\circ$, azim=$135^\circ$). The overall distribution of ambiguous clips is visible, with SLERP pulling features away from the decision boundary toward their respective prototype clusters.}
        \label{fig:slerp_default}
    \end{subfigure}

    \vspace{0.3em}

    \begin{subfigure}[t]{0.92\textwidth}
        \centering
        \includegraphics[width=\linewidth]{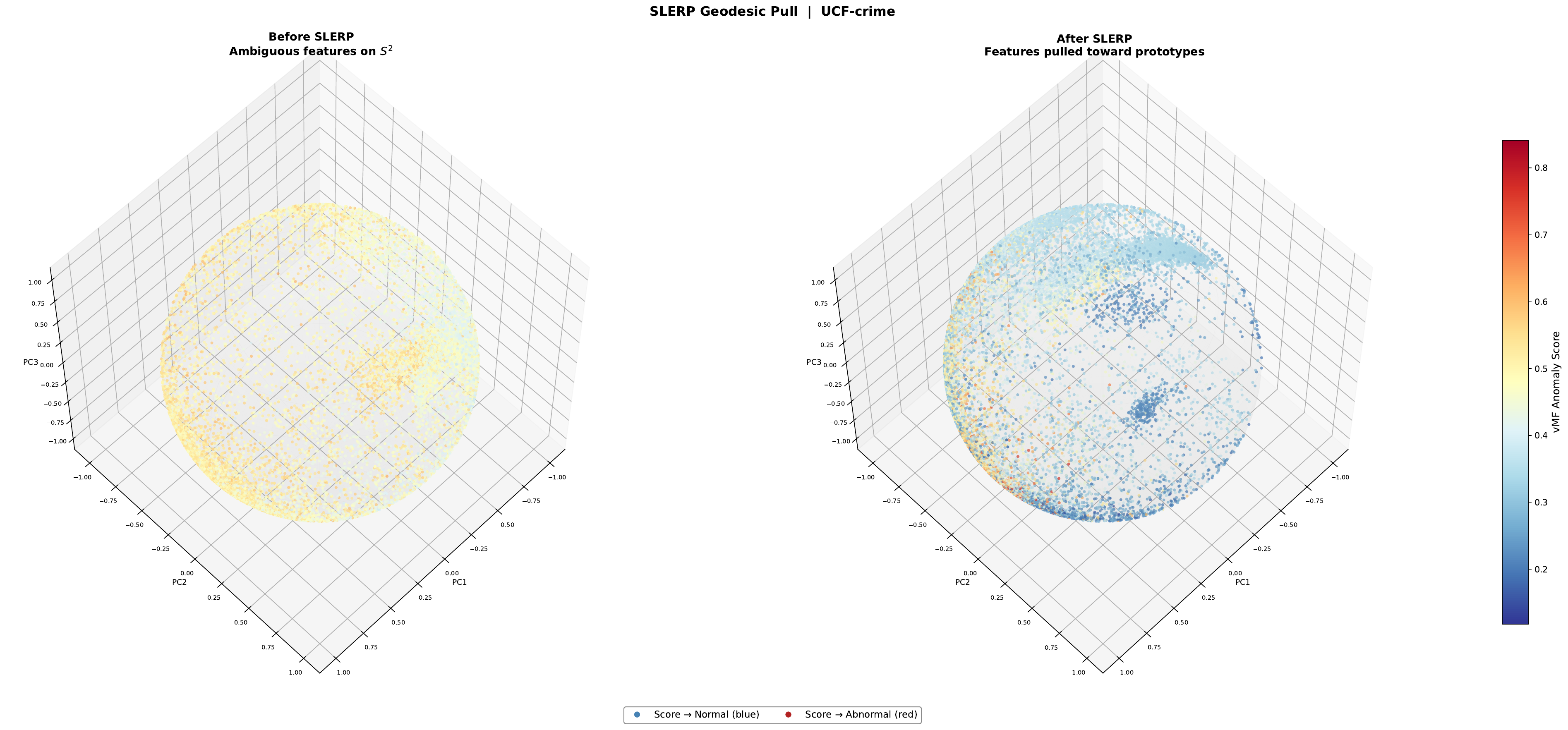}
        \caption{%
            Top-down view (elev=$60^\circ$, azim=$45^\circ$). This bird's-eye perspective highlights the lateral displacement of ambiguous features, showing that clips initially clustered near the equatorial boundary are redistributed toward polar regions after SLERP.}
        \label{fig:slerp_top}
    \end{subfigure}

    \vspace{0.3em}

    \begin{subfigure}[t]{0.92\textwidth}
        \centering
        \includegraphics[width=\linewidth]{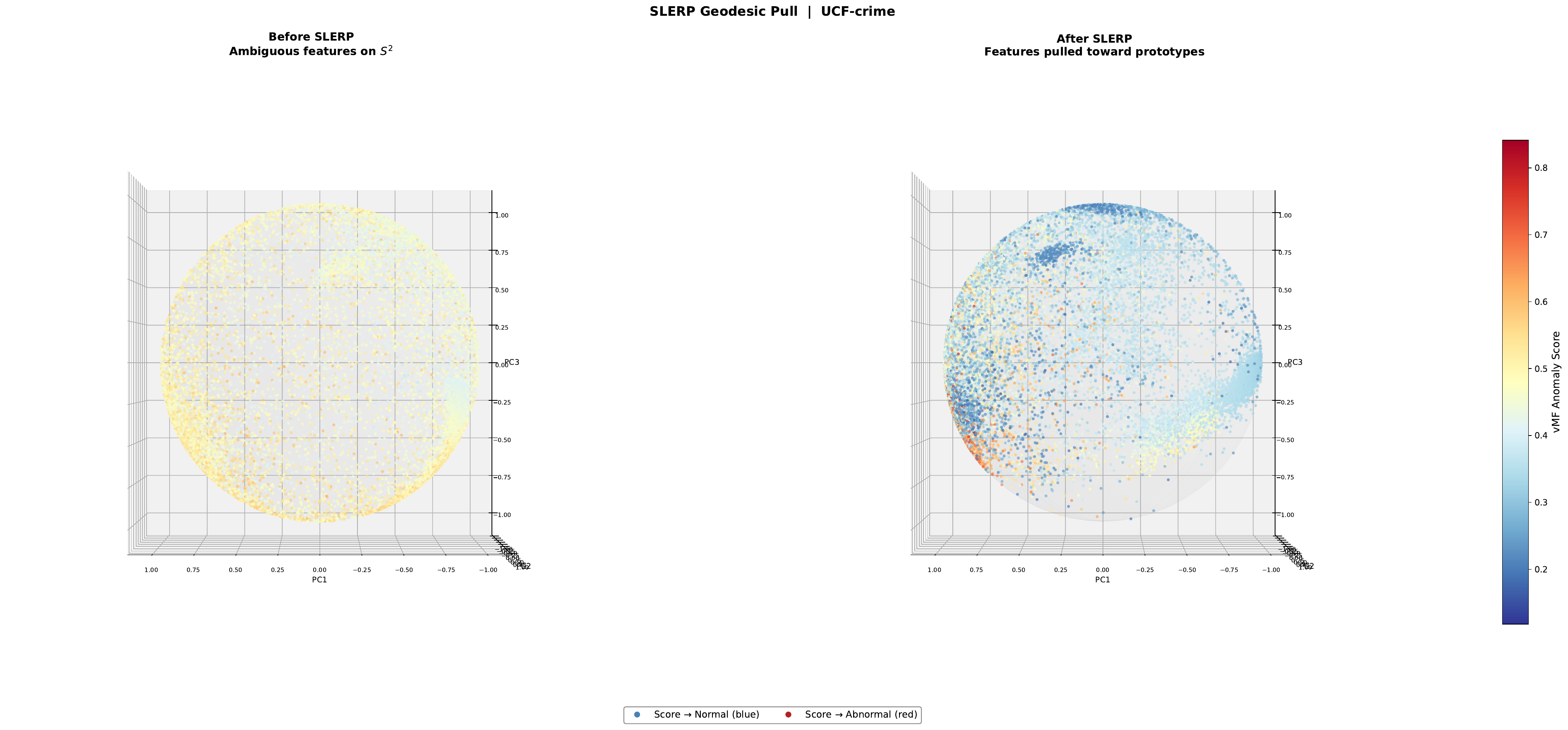}
        \caption{%
            Side view (elev=$0^\circ$, azim=$90^\circ$). The profile view reveals the geodesic displacement along the vertical axis, confirming that the pulling magnitude is larger for clips closer to the decision boundary (score $\approx 0.5$).}
        \label{fig:slerp_side}
    \end{subfigure}

    \caption{%
        \textbf{SLERP geodesic pulling visualized from three viewpoints} on UCF-Crime.
        \textit{Left} of each panel: ambiguous features before SLERP (M2 stage).
        \textit{Right} of each panel: the same features after SLERP pulling toward vMF prototypes (M3 stage).
        Points are colored by their vMF anomaly score (blue $=$ normal, red $=$ abnormal).
        The top 50\% of ambiguous clips ranked by geodesic displacement are shown.}
    \label{fig:add_slerp}
\end{figure}
\subsection{Domain Shift Visualization}
\label{app:domain_shift}

Figure~\ref{fig:domain_shift} visualises the distributional gap between the synthetic calibration domain and the three real test domains in both Euclidean and hyperspherical feature spaces, providing direct geometric evidence for the rotational bias discussed in Section~3.2 and the necessity of unified Fr\'{e}chet mean centering.

\paragraph{Setup.}
We extract main features $f^l$ (last-token hidden states at layer $\ell^*{=}31$) from the frozen Qwen3.5 backbone for all clips in the synthetic calibration set and the three test sets (UCF-Crime, XD-Violence, UBnormal). For the Euclidean view~(a), we compute the per-dataset arithmetic mean in the original high-dimensional space and project all means into 3D via PCA (fitted on a pooled subsample of $5{,}000$ clips per dataset). For the hyperspherical view~(b), we $\ell_2$-normalise all features onto $\mathcal{S}^{D-1}$, compute the per-dataset Fr\'{e}chet mean via Karcher iteration (Appendix~A.1), and project the four Fr\'{e}chet means into 3D via a separate PCA fitted on the normalised features. To aid visual clarity, the Synthetic mean is placed at the origin (Euclidean) or the north pole (hypersphere), and the remaining three datasets are arranged at uniform azimuthal intervals ($120^\circ$ apart) while preserving the true inter-domain distances or geodesic angles, respectively.

\paragraph{Euclidean space (panel~a).}
Each dataset is represented by a translucent sphere of uniform radius (the $50$th-percentile coverage radius of UBnormal, used as a common scale reference) centred at its Euclidean mean (star marker). The Euclidean distances from Synthetic to the three test domains are $d_{\mathrm{UCF}}{=}21.74$, $d_{\mathrm{XD}}{=}14.29$, and $d_{\mathrm{UBnormal}}{=}17.34$. These distances reflect both directional and magnitude differences in the raw feature space and are not directly comparable across domains because the Euclidean metric conflates norm variation with angular displacement.

\paragraph{Hypersphere space (panel~b).}
After $\ell_2$-normalisation, all features reside on the unit hypersphere $\mathcal{S}^{D-1}$, and the natural metric becomes the geodesic (great-circle) angle. The geodesic gaps from Synthetic to the three test domains are $\Delta_{\mathrm{UCF}}{=}8.2^\circ$, $\Delta_{\mathrm{XD}}{=}5.4^\circ$, and $\Delta_{\mathrm{UBnormal}}{=}6.5^\circ$, visualised as arcs on the unit sphere connecting the Fr\'{e}chet means. These offsets of $5^\circ$--$8^\circ$ confirm that the domain shift is predominantly \emph{directional} rather than magnitude-based: once norms are factored out, the residual discrepancy reduces to a modest angular rotation.

\paragraph{Implications for spherical centering.}
The visualisation provides three key insights that motivate the unified Fr\'{e}chet mean centering of Eq.~(2):

\begin{enumerate}[leftmargin=*,itemsep=2pt]
    \item \textbf{The domain shift is geometrically compact on the sphere.} While Euclidean distances span a wide range ($14$--$22$), the geodesic angles cluster within $5^\circ$--$8^\circ$, indicating that the four domain distributions occupy nearby but systematically offset spherical caps. A single centering operation can absorb this offset.
    \item \textbf{Synthetic--real misalignment is universal.} All three test domains exhibit a non-negligible angular displacement from the synthetic calibration domain. Without centering, vMF prototypes calibrated on synthetic features would be systematically rotated away from the real feature distribution, degrading scoring accuracy---consistent with the $+5.78$\,pp AUC gain observed on UCF-Crime upon adding spherical centering (Table~2, M0a$\to$M0b).
    \item \textbf{The unified Fr\'{e}chet mean provides a symmetric reference.} By pooling synthetic and real features before computing the Fr\'{e}chet mean (Eq.~(1)), the centering base point lies near the geodesic midpoint of the domain-specific means (Proposition~2), absorbing the rotational bias symmetrically from both sides rather than privileging either domain.
\end{enumerate}

\begin{figure}[t]
    \centering
    \includegraphics[width=\linewidth]{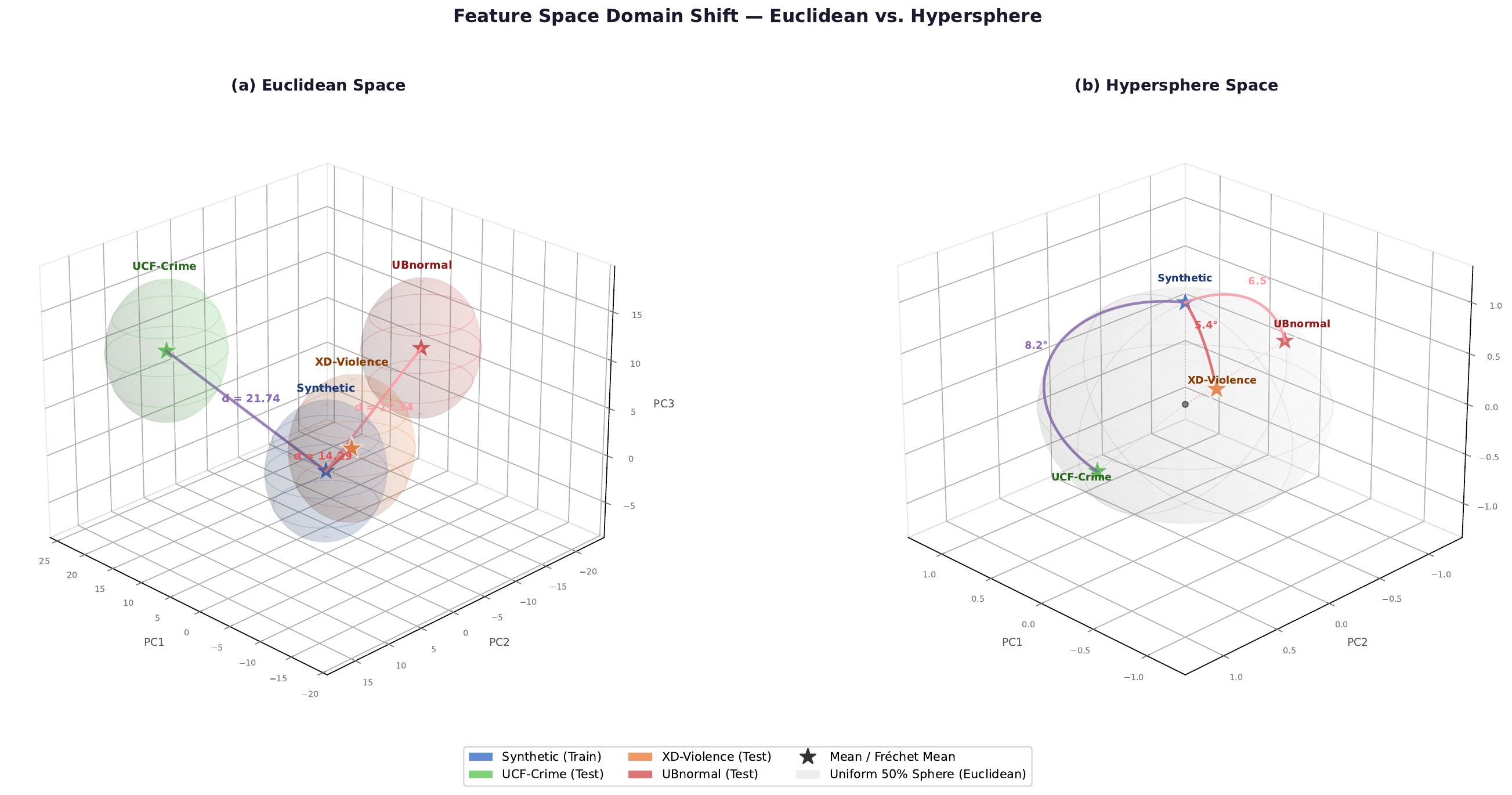}
    \caption{\textbf{Domain shift between synthetic and real feature distributions.}
    (a)~\textit{Euclidean space}: per-dataset arithmetic means (stars) with uniform $50\%$-coverage spheres; straight lines show Euclidean distances from Synthetic to each test domain.
    (b)~\textit{Hypersphere space}: per-dataset Fr\'{e}chet means (stars) on the unit sphere $\mathcal{S}^{D-1}$; geodesic arcs show angular offsets of $5^\circ$--$8^\circ$, confirming a systematic rotational bias that is absorbed by unified Fr\'{e}chet mean centering (Section~3.2).
    Features are extracted from layer~31 of a frozen Qwen3.5; 3D positions are obtained via PCA and rearranged for visual clarity (Synthetic at origin/north pole, others at $120^\circ$ azimuthal intervals) while preserving true inter-domain distances/angles.}
    \label{fig:domain_shift}
\end{figure}
\section{Hyperparameter Sensitivity Analysis}
\label{app:hyper}
SphereVAD contains four dataset-specific core hyperparameters: the normal and anomalous prototype counts $(K_N, K_A)$, the HSA mixing coefficient $\alpha_G$, and the SLERP base pull strength $\beta_{\mathrm{base}}$.
Note that $\beta_{\mathrm{base}}$ is applicable only to XD-Violence and UCF-Crime, since SGP is disabled for UBnormal (whose test videos are too short for reliable intra-video tri-classification; see Appendix~D and Table~2 in the main text).
The following subsections list the configurations evaluated across all three benchmarks, provide detailed sensitivity sweeps for each component, and report a unified-setting evaluation demonstrating the robustness of the framework.
\subsection{Complete Hyperparameter Table}
\label{app:hyper_table}
Table~\ref{tab:hyper_full} summarises the hyperparameter configurations evaluated across the three benchmark datasets.
All other algorithmic parameters (MAD scale factors, attention sparsification constants, etc.) are fixed globally across all datasets; their specific values and descriptions are detailed in Appendix~D.
\begin{table}[htbp]\centering
  \caption{%
    Hyperparameter configurations for all three benchmarks.
    ``---'' indicates that the parameter is not applicable (SGP is disabled for UBnormal).
  }
  \label{tab:hyper_full}
  \setlength{\tabcolsep}{10pt}
  \begin{tabular}{lcccc}
    \toprule
    \textbf{Dataset} & $K_N$ & $K_A$ & $\alpha_G$ & $\beta_{\mathrm{base}}$ \\
    \midrule
    XD-Violence  & 10 & 12 & 0.80 & 0.15 \\
    UCF-Crime    & 18 & 12 & 0.75 & 0.50 \\
    UBnormal     & 12 & 20 & 0.35 & ---  \\
    \bottomrule
  \end{tabular}
\end{table}
\FloatBarrier
\subsection{Prototype Count \texorpdfstring{$(K_N,\,K_A)$}{(KN, KA)}}
\label{app:hyper_proto}
We evaluate the effect of varying prototype granularities. We first perform a coarse grid evaluation over $K_N, K_A \in \{1, 7, 12, 18, 23, 29\}$ on each dataset to identify high-performing plateaus, followed by local refinement around candidate peaks (yielding the final configurations in Table~\ref{tab:hyper_full}, such as $K_N{=}10$ on XD-Violence and $K_A{=}20$ on UBnormal).
Figure~\ref{fig:proto_sensitivity} shows the resulting coarse-grid heatmaps for XD-Violence (frame-level AP), UCF-Crime (frame-level AUC), and UBnormal (frame-level AUC).
\FloatBarrier
\begin{figure}[htbp]\centering
  \begin{minipage}[t]{0.32\linewidth}
    \centering
    \includegraphics[width=\linewidth]{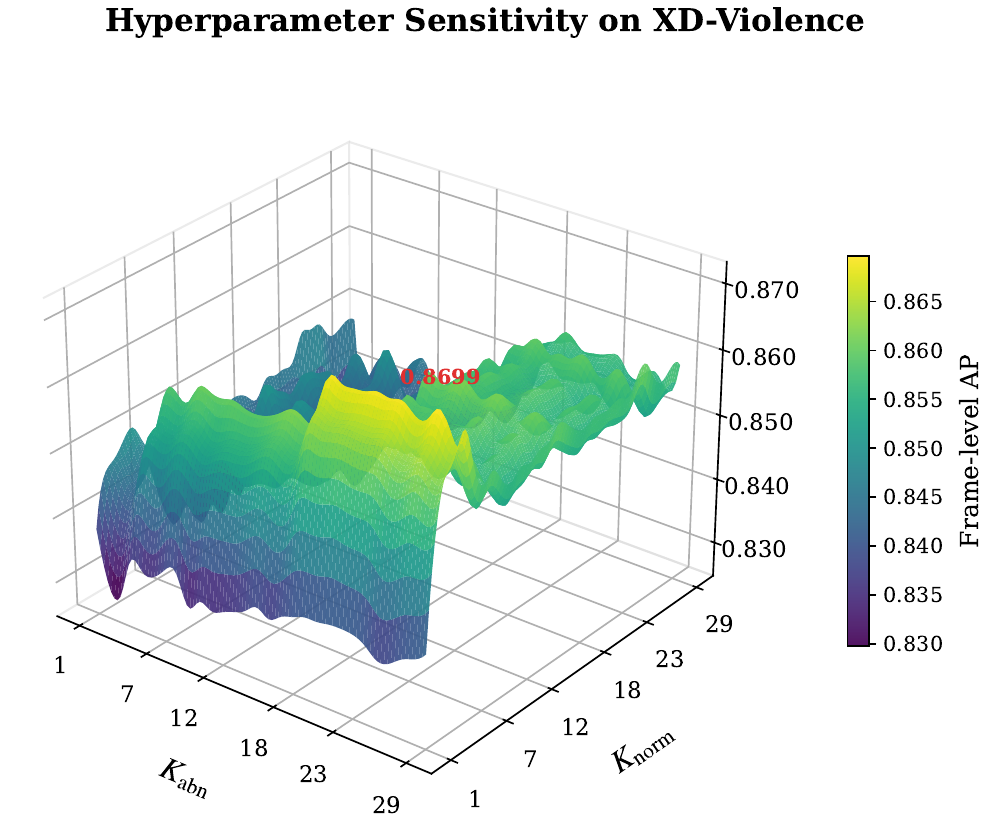}
    \subcaption{XD-Violence (AP)}
    \label{fig:proto_xd}
  \end{minipage}
  \hfill
  \begin{minipage}[t]{0.32\linewidth}
    \centering
    \includegraphics[width=\linewidth]{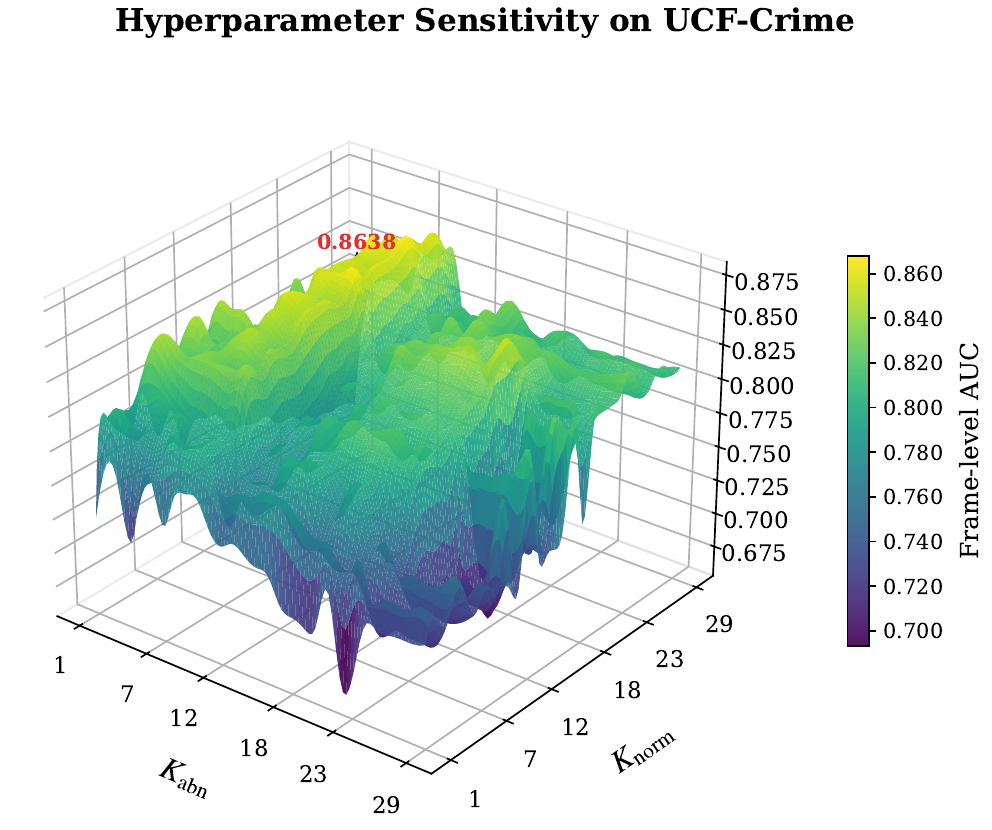}
    \subcaption{UCF-Crime (AUC)}
    \label{fig:proto_ucf}
  \end{minipage}
  \hfill
  \begin{minipage}[t]{0.32\linewidth}
    \centering
    \includegraphics[width=\linewidth]{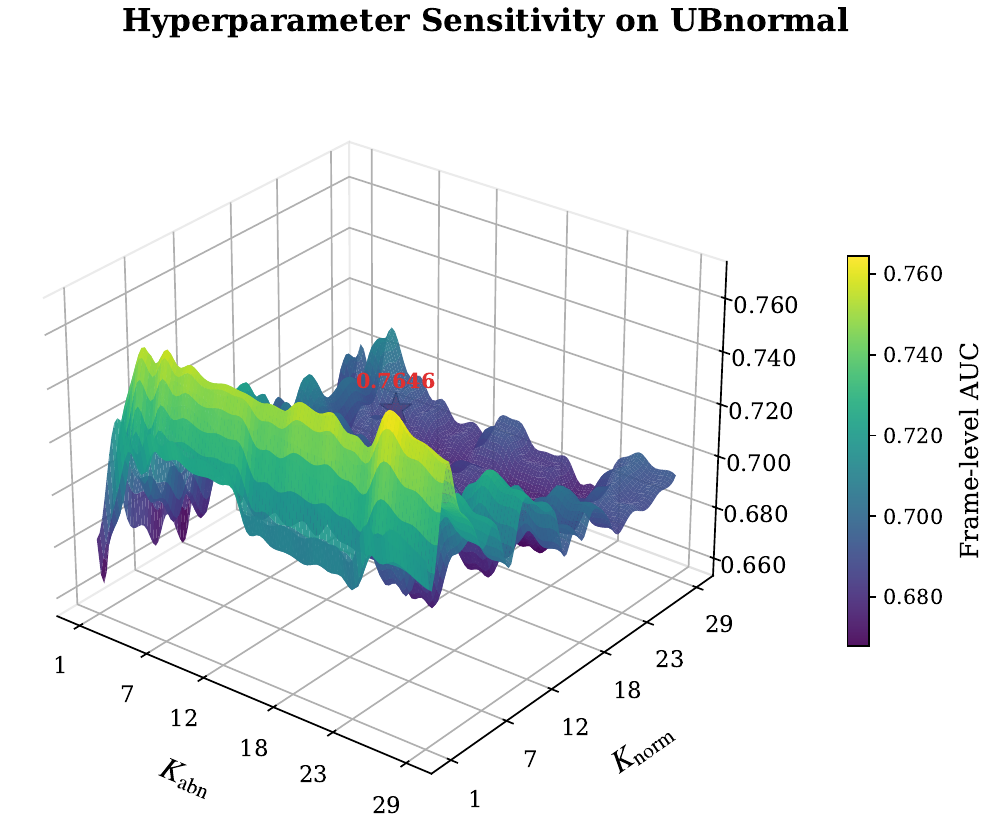}
    \subcaption{UBnormal (AUC)}
    \label{fig:proto_ubn}
  \end{minipage}
  \caption{%
    \textbf{Prototype count sensitivity.}
    Each cell shows the primary metric achieved by SphereVAD across varying counts of $K_N$ normal prototypes ($x$-axis) and $K_A$ anomalous prototypes ($y$-axis).
    Performance remains broadly stable across a wide interior region: on XD-Violence the top-quartile region (AP $\ge0.860$) covers roughly half the grid, and on UCF-Crime the AUC remains above $0.840$ for all $K_N, K_A \ge 7$.
    Single-prototype baselines ($K_N\!=\!1$ or $K_A\!=\!1$) incur significant drops ($\sim$3--5\%), confirming the utility of multi-prototype clustering on the hypersphere.
  }
  \label{fig:proto_sensitivity}
\end{figure}
\FloatBarrier
\paragraph{Observations.}
\begin{itemize}[leftmargin=1.5em, itemsep=2pt]
  \item \textbf{Plateau behaviour.}
        For all three datasets, varying prototype counts within the central region $K_N, K_A \in [12, 23]$ yields relatively smooth performance with modest variance. Deviations in this interior region remain within manageable limits, showing that SphereVAD is tolerant to the specific choice of prototype partition size.
  \item \textbf{Single-prototype degeneration.}
        Restricting representations to a single prototype ($K_N\!=\!1$ or $K_A\!=\!1$) causes severe performance degradation, particularly on XD-Violence where visual manifestations of anomalous actions vary substantially across subclasses.
\end{itemize}
\subsection{HSA Mixing Coefficient \texorpdfstring{$\alpha_G$}{alphaG} and SLERP Base Pull Strength \texorpdfstring{$\beta_{\mathrm{base}}$}{beta\_base}}
\label{app:hyper_alphabeta}
Figure~\ref{fig:ab_sensitivity} illustrates performance as a function of $\alpha_G$ and $\beta_{\mathrm{base}}$ when swept across $[0, 1]$ while holding remaining parameters fixed to their default settings.
Because SGP is disabled for UBnormal, only the $\alpha_G$ curve is reported for that benchmark.
\FloatBarrier
\begin{figure}[htbp]
  \centering
  \begin{minipage}[t]{0.32\linewidth}
    \centering
    \includegraphics[width=\linewidth]{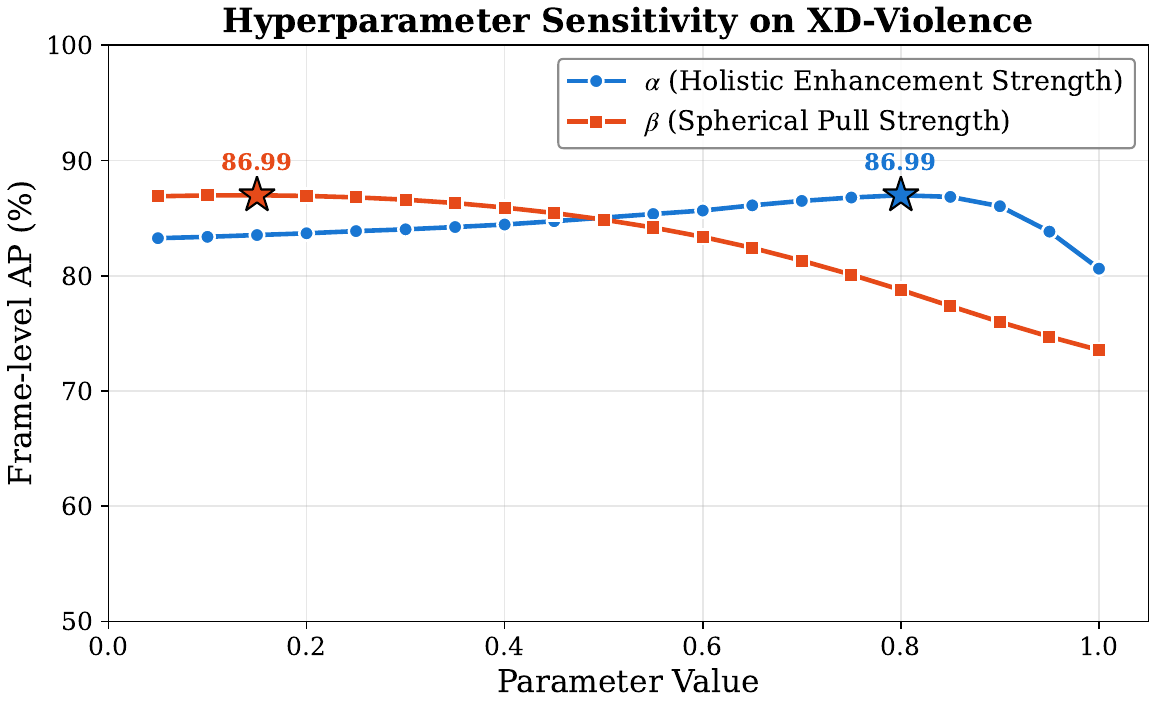}
    \subcaption{XD-Violence (AP)}
    \label{fig:ab_xd}
  \end{minipage}
  \hfill
  \begin{minipage}[t]{0.32\linewidth}
    \centering
    \includegraphics[width=\linewidth]{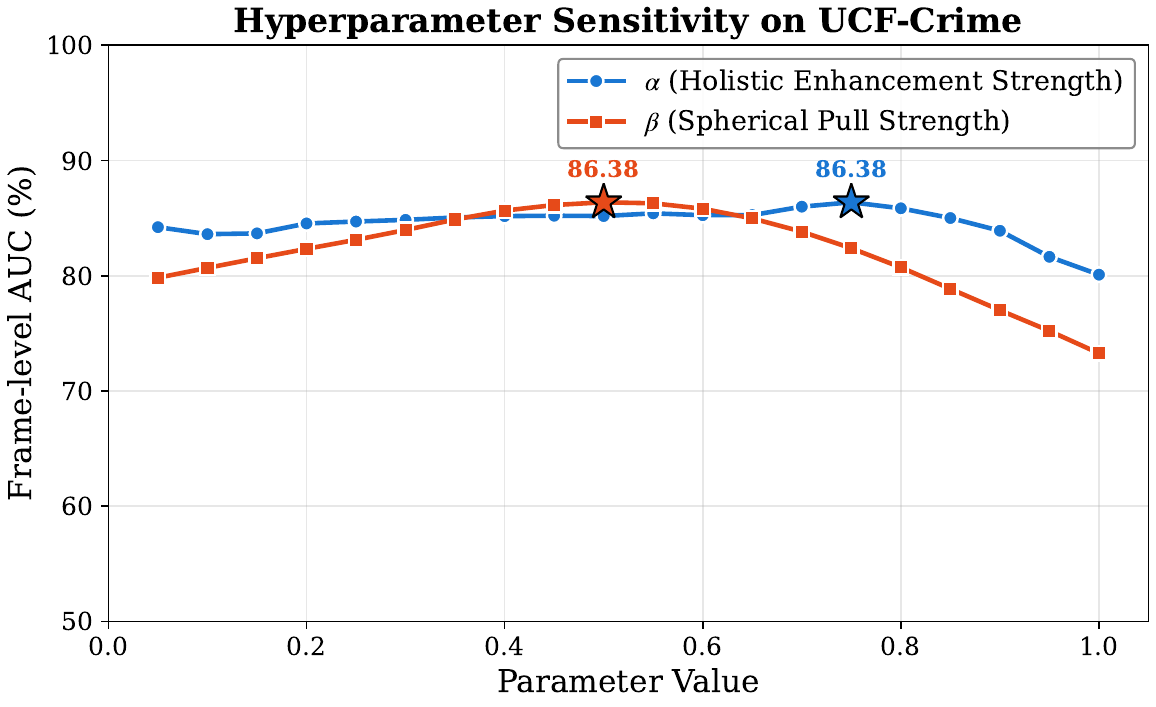}
    \subcaption{UCF-Crime (AUC)}
    \label{fig:ab_ucf}
  \end{minipage}
  \hfill
  \begin{minipage}[t]{0.32\linewidth}
    \centering
    \includegraphics[width=\linewidth]{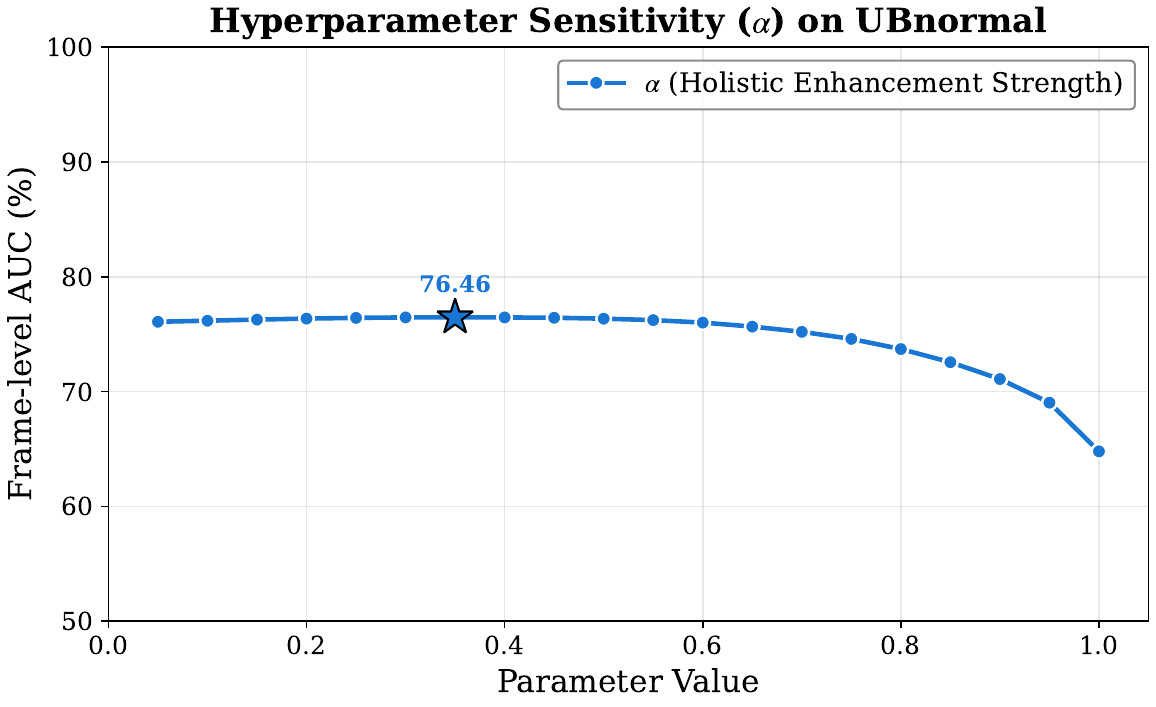}
    \subcaption{UBnormal (AUC, $\alpha_G$ only)}
    \label{fig:a_ubn}
  \end{minipage}
  \caption{%
    \textbf{Sensitivity to $\alpha_G$ (Holistic Enhancement Strength) and $\beta_{\mathrm{base}}$ (Spherical Pull Strength).}
    Curves are obtained by sweeping one parameter over $[0,1]$ while holding other variables fixed. Dashed horizontal lines indicate peak observed metrics (XD-Violence: AP\,=\,86.99\%; UCF-Crime: AUC\,=\,86.38\%; UBnormal: AUC\,=\,76.46\%).
    $\beta_{\mathrm{base}}$ is not applicable to UBnormal.
  }
  \label{fig:ab_sensitivity}
\end{figure}
\FloatBarrier
\paragraph{HSA mixing coefficient $\alpha_G$.}
$\alpha_G \in [0,1]$ governs the contribution of cross-video context relative to individual segment representations in Eq.~(5).
\begin{itemize}[leftmargin=1.5em, itemsep=2pt]
  \item \textbf{$\alpha_G = 0$ (HSA disabled).}
        SphereVAD reverts to an M1+SGP pipeline without cross-video scene aggregation. The visible decline verifies the benefit of reinforcing segment consistency via global scene priors across untrimmed videos.
  \item \textbf{Stable plateau.}
        Under 1D sweeps, all three benchmarks exhibit stable trajectories within $\alpha_G \in [0.3, 0.8]$, with isolated metric changes remaining moderate.
  \item \textbf{$\alpha_G \to 1$ (pure attention).}
        Replacing intrinsic clip features completely with cross-video contextual aggregates degrades performance ($\sim$1--2\%), as segment-specific temporal cues become over-smoothed.
\end{itemize}
\paragraph{SLERP base pull strength $\beta_{\mathrm{base}}$.}
$\beta_{\mathrm{base}} \in [0, 1]$ modulates the maximum displacement step during score-adaptive spherical interpolation (Eq.~44).
\begin{itemize}[leftmargin=1.5em, itemsep=2pt]
  \item \textbf{$\beta_{\mathrm{base}} = 0$ (SGP disabled).}
        Performance reverts directly to the M2 baseline, confirming that geodesic alignment plays an active role in resolving borderline segment scores.
  \item \textbf{Stable plateau.}
        Both XD-Violence and UCF-Crime display steady performance across $\beta_{\mathrm{base}} \in [0.1, 0.5]$ in isolated 1D sweeps.
  \item \textbf{$\beta_{\mathrm{base}} \to 1$ (excessive pull).}
        Excessive pull strength forces uncertain clips aggressively toward prototypes regardless of directional confidence. While the uncertainty factor $(1 - \tfrac{1}{2}\hat{d}_i)$ moderates the movement, extreme values ($\beta_{\mathrm{base}} > 0.8$) lead to boundary oversaturation and performance loss.
\end{itemize}
\subsection{Unified Hyperparameter Robustness}
\label{app:hyper_unified}
To evaluate the out-of-the-box generalization of SphereVAD, we test a \emph{single, unified} default configuration across all three benchmarks: $K_N = 12$, $K_A = 18$, $\alpha_G = 0.5$, and $\beta_{\mathrm{base}} = 0.5$ (with SGP disabled for UBnormal as established).
These settings represent canonical median values observed across the parameter ranges.
Table~\ref{tab:hyper_unified} contrasts this unified baseline against the per-dataset settings reported in Table~1.

\begin{table}[htbp]
  \centering
  \caption{%
    \textbf{Unified hyperparameter evaluation.}
    A single configuration ($K_N\!=\!12$, $K_A\!=\!18$, $\alpha_G\!=\!0.5$, $\beta_{\mathrm{base}}\!=\!0.5$) evaluated across all three benchmarks.
    $\Delta$ denotes the difference in percentage points (pp).
    ``---'' indicates SGP is disabled.
  }
  \label{tab:hyper_unified}
  \setlength{\tabcolsep}{6pt}
  \begin{tabular}{lcccc}
    \toprule
    & \textbf{XD-Violence} & \textbf{XD-Violence} & \textbf{UCF-Crime} & \textbf{UBnormal} \\
    \textbf{Setting} & AP(\%) & AUC(\%) & AUC(\%) & AUC(\%) \\
    \midrule
    Per-dataset configuration & 86.99 & 95.74 & 86.38 & 76.46 \\
    Unified default           & 85.65 & 95.10 & 83.42 & 72.31 \\
    \midrule
    $\Delta$ (pp)             & $-$1.34 & $-$0.64 & $-$2.96 & $-$4.15 \\
    \bottomrule
  \end{tabular}
\end{table}
\FloatBarrier

Key observations include:
\begin{itemize}[leftmargin=1.5em, itemsep=3pt]
  \item \textbf{Robust baseline performance across benchmarks.}
      Under the single unified configuration, SphereVAD achieves 85.65\% AP on XD-Violence, 83.42\% AUC on UCF-Crime, and 72.31\% AUC on UBnormal.
      On XD-Violence, this uncalibrated setting substantially surpasses prior training-free zero-shot methods (e.g., VADTree at 68.85\% AP and PANDA at 70.16\% AP).
      On UCF-Crime and UBnormal, although joint parameter shifts lead to moderate drops ($-$2.96\,pp and $-$4.15\,pp, respectively), the unified model remains competitive, outperforming earlier training-free baselines such as MCANet (62.94\% on UBnormal) and VADTree (65.80\% on UBnormal).
  \item \textbf{Compound parameter shift.}
      While isolated sweeps along single dimensions exhibit broad plateau stability, simultaneously shifting multiple hyperparameters away from their empirical optima causes compounding geometric misalignments, which explains the observed performance gap.
  \item \textbf{Practical out-of-the-box utility.}
      This evaluation confirms practical applicability: even without dataset-specific tuning, SphereVAD provides strong zero-shot anomaly detection performance out of the box.
\end{itemize}
\paragraph{Summary.}
SphereVAD demonstrates consistent behavior across reasonable variations in prototype granularity, attention weighting, and geodesic displacement.
The unified baseline delivers competitive zero-shot performance across distinct surveillance domains, showing that the core benefits stem from the underlying spherical geometric formulations rather than narrow parameter assignments.

\subsection{Downstream Extension: Adaptability to Weakly-Supervised MIL}
\label{app:weakly_supervised}

While SphereVAD is originally formulated as a strictly training-free, zero-shot framework, a natural question is whether the discriminative geometric structures unlocked by spherical centering can benefit traditional learning paradigms. To investigate this, we assess the adaptability of SphereVAD representations under the standard weakly-supervised video anomaly detection (WS-VAD) protocol on the \textbf{ShanghaiTech} benchmark~\cite{luo2017revisit}.

\subsubsection{Setup and Pseudo-Label Refinement via MIL}
Under the standard WS-VAD protocol on ShanghaiTech (238 training and 199 testing videos), only video-level binary labels are available during training. Instead of fine-tuning the heavy MLLM, we directly evaluate the utility of the spherically centered feature representations $\hat{\mathbf{f}}^l$:
\begin{enumerate}
    \item \textbf{Feature Extraction:} We extract layer-31 intermediate representations from the frozen Qwen3.5 backbone and project them via unified Fr{\'e}chet mean spherical centering (Eq.~(2)).
    \item \textbf{MIL-Guided Pseudo-Label Refinement:} Adopting the Multiple Instance Learning (MIL) paradigm~\cite{sultani2018real,tian2021weakly}, an untrimmed video is treated as a bag of clip instances. A lightweight temporal MIL classifier is trained on these bags, selecting the top-$k$ highest-scoring instances from anomalous bags as anomalous pseudo-labels and instances from normal bags as normal labels. This relabeling dynamically refines segment-level supervision over the training feature set.
    \item \textbf{Inference:} The refined classifier is evaluated directly on the standard ShanghaiTech test set.
\end{enumerate}

\subsubsection{Empirical Results and Discussion}

Table~\ref{tab:shtech_ws_vad} compares our MIL-adapted SphereVAD framework against representative weakly-supervised and training-free methods on ShanghaiTech.

\begin{table}[h]
\centering
\small
\caption{Frame-level AUC (\%) comparison on the \textbf{ShanghaiTech} dataset under weakly-supervised and training-free settings. The best performance in each category is shown in \textbf{bold}.}
\label{tab:shtech_ws_vad}
\begin{tabular}{llcc}
\toprule
\textbf{Method} & \textbf{Venue} & \textbf{Supervision Paradigm} & \textbf{ShanghaiTech AUC (\%)} \\
\midrule
\multicolumn{4}{l}{\textit{Weakly-Supervised VAD (Trained with video-level labels)}} \\
Sultani \textit{et al.}~\cite{sultani2018real} & CVPR'18 & WS-VAD (C3D) & 86.30 \\
RTFM~\cite{tian2021weakly}           & ICCV'21 & WS-VAD (I3D) & 97.21 \\
S3R~\cite{wu2022self}                & ECCV'22 & WS-VAD (I3D) & 97.48 \\
CLIP-TSA~\cite{joo2023clip}          & ICIP'23 & WS-VAD (CLIP) & 97.74 \\
VadCLIP~\cite{wu2024vadclip}         & AAAI'24 & WS-VAD (CLIP) & \textbf{97.80} \\
SphereVAD + MIL (Ours)               & NeurIPS'26 & WS-VAD (Frozen Qwen3.5) & 96.38 \\
\midrule
\multicolumn{4}{l}{\textit{Zero-Shot / Training-Free Baselines (For context)}} \\
LAVAD~\cite{zanella2024harnessing}   & CVPR'24 & Zero-shot & 82.10 \\
MCANet~\cite{dev2024mcanet}          & ICPR'24 & Zero-shot & 83.40 \\
\textbf{SphereVAD (Zero-shot, default)} & NeurIPS'26 & Training-free & \textbf{86.07} \\
\bottomrule
\end{tabular}
\end{table}

\noindent\textbf{Key Takeaways:}
\begin{itemize}
    \item \textbf{Training-Free Baseline Competitiveness:} Under the completely training-free zero-shot setup, default SphereVAD achieves \textbf{86.07\%} AUC, advancing past previous zero-shot methods on ShanghaiTech.
    \item \textbf{High Adaptability to Downstream Learning:} When paired with a lightweight MIL classifier, the centered representations reach \textbf{96.38\%} AUC without updating any MLLM parameters. This verifies that spherical centering does not merely serve zero-shot distance thresholds, but fundamentally unrolls intermediate feature manifolds into highly discriminative representations that downstream supervised learners can readily leverage.
\end{itemize}

\section{Reproducibility Statement \& Symbol Table}
\label{app:repro}
\subsection{Full Symbol Table}
\label{app:symbols}
\begin{table}[H]
\caption{Summary of mathematical symbols used throughout the paper and appendix.}
\label{tab:symbols}
\centering
\small
\begin{tabular}{@{}cl@{}}
\toprule
\textbf{Symbol} & \textbf{Description} \\
\midrule
$f$ & Raw intermediate-layer feature \\
$\tilde{f}$ & $\ell_2$-normalised feature on $\mathcal{S}^{D-1}$ \\
$\hat{f}$ & Spherically centered feature \\
$f^l$ & Main feature (last token hidden state) \\
$f^v$ & Visual feature (visual-last token hidden state) \\
$D$ & Feature dimensionality \\
$\mathcal{S}^{D-1}$ & Unit hypersphere in $\mathbb{R}^D$ \\
$\boldsymbol{\mu}_{\mathrm{unified}}$ & Unified Fr\'{e}chet mean \\
$\boldsymbol{\mu}_k^{\mathrm{norm}}$, $\boldsymbol{\mu}_k^{\mathrm{abn}}$ & Normal / anomalous vMF prototypes \\
$K_N$, $K_A$ & Number of normal / anomalous prototypes \\
$\kappa$ & vMF concentration parameter \\
$\alpha_G$ & HSA mixing coefficient \\
$\beta_{\mathrm{base}}$ & SLERP base pull strength \\
$\beta_i$ & Score-adaptive pull strength for clip $i$ \\
$s$ & Anomaly score $\in [0,1]$ \\
$d_{\mathrm{geo}}(\cdot,\cdot)$ & Geodesic distance on $\mathcal{S}^{D-1}$ \\
$\mathrm{Log}_{\boldsymbol{\mu}}(\cdot)$ & Riemannian logarithmic map at $\boldsymbol{\mu}$ \\
$\mathrm{Exp}_{\boldsymbol{\mu}}(\cdot)$ & Riemannian exponential map at $\boldsymbol{\mu}$ \\
$\mathrm{Slerp}(\cdot,\cdot,\cdot)$ & Spherical linear interpolation \\
$\mathbf{A}$ & Cross-video sparse attention matrix (HSA) \\
$\mathbf{A}_H$ & Intra-video sparse attention matrix (SGP) \\
$\rho_{\mathrm{low}}$, $\rho_{\mathrm{high}}$ & MAD-based ambiguity interval bounds \\
$\boldsymbol{\mu}_{\mathrm{dom}}^{c}$ & Dominant prototype for class $c$ \\
$\mathcal{D}_{\mathrm{syn}}$ & Synthetic calibration dataset \\
$M$ & Number of synthetic calibration images \\
$N$ & Number of test videos \\
$T_i$ & Number of clips in video $V_i$ \\
\bottomrule
\end{tabular}
\end{table}
\subsection{Random Seed \& Determinism}
\label{app:seed}

SphereVAD is designed to be deterministic given a fixed random seed.
A global seed (default: \texttt{42}) controls the only stochastic component in the pipeline---spherical $K$-Means initialisation---via \texttt{sklearn.cluster.KMeans(random\_state=42)} with $n_{\mathrm{init}}=10$.
All remaining operations (Fr\'{e}chet mean, spherical centering, HSA, vMF scoring, SGP, Gaussian smoothing) are closed-form deterministic computations.
We additionally set \texttt{torch.backends.cudnn.deterministic=True} and \texttt{torch.backends.cudnn.benchmark=False} to ensure reproducible MLLM forward passes.
We verified stability by repeating the full pipeline with multiple seeds $\{42,123,456,789,1024\}$ on the same hardware and observed negligible variance across runs.

\subsection{Code \& Data Availability}
\label{app:code}

The complete source code---including feature extraction, the geometric inference pipeline (spherical centering, vMF prototypes, HSA, SGP, scoring), evaluation scripts, and all configuration files---will be publicly released upon acceptance.
Pre-extracted features for all three benchmarks and the synthetic calibration set will also be provided, enabling full reproduction of the geometric inference pipeline \emph{without} GPU access.

The synthetic calibration data generation scripts and meta-prompts are included in the repository for reproducibility.
We emphasise that the synthetic data are \emph{not} a contribution of this work; they serve solely as directional calibrators for vMF prototype construction and can be replaced by any alternative set of normal/abnormal reference images.
All synthetic images are machine-generated and contain no real individuals or personally identifiable information.

All four MLLM backbones (Qwen3.5, Qwen3-VL, InternVL3, LLaVA-OV-1.5) are publicly available on HuggingFace and used in their frozen, unmodified states.
All three evaluation benchmarks (XD-Violence, UCF-Crime, UBnormal) are publicly available and used with their standard test splits under their respective licences.

\section{Limitations}
\label{app:limitations}
While SphereVAD achieves strong training-free performance, several limitations merit discussion.
\textbf{Dependence on MLLM feature quality.}
SphereVAD does not learn new representations; it operates entirely on frozen intermediate-layer features.
Consequently, its performance ceiling is bounded by the anomaly-discriminative capacity already latent in the backbone.
As shown in Table~\ref{tab:ablation_backbone}, weaker backbones (e.g., LLaVA-OV-1.5) yield noticeably lower results, indicating that the framework amplifies but cannot create discriminability absent from the features.
\textbf{Short-video limitation of SGP.}
The intra-video Spherical Geodesic Pulling mechanism relies on MAD-based tri-classification and neighbour consensus, both of which require sufficient temporal context.
For very short videos (e.g., UBnormal, averaging ${\sim}$12\,seconds per video), SGP provides no benefit and is disabled (Table~\ref{tab:ablation_module}, M2\,=\,M3 on UBnormal).
Extending SGP to short-clip scenarios---e.g., via cross-video pulling or temporal padding---remains future work.
\textbf{Synthetic calibration domain gap.}
Although unified Fr\'{e}chet mean centering effectively absorbs the ${\sim}5^{\circ}$ rotational bias between synthetic and real domains, a residual domain gap persists.
The current synthetic generation pipeline (Appendix~\ref{app:synth}) uses generic meta-categories; domain-adaptive calibration strategies that tailor synthetic data to specific deployment environments could further improve performance.
\textbf{Single-anomaly-type assumption in SGP.}
The dominant prototype voting mechanism (Appendix~\ref{app:dominant}) assumes that each video contains at most one prevailing anomaly type.
Videos with multiple co-occurring anomaly categories (e.g., simultaneous arson and assault) may not be optimally handled, as the single dominant anomalous prototype cannot represent both event types.
\textbf{Quadratic cost of HSA.}
The cross-video Holistic Scene Attention module constructs a pairwise similarity matrix over all test clips, incurring $O(N_{\mathrm{total}}^2 D)$ cost.
While manageable for current benchmarks (${\sim}$97K clips, $<$10\,s on CPU), scaling to deployment scenarios with millions of clips would require approximate nearest-neighbour indexing or hierarchical attention schemes.
\section{Broader Impact}
\label{app:impact}
\textbf{Positive societal impact.}
SphereVAD is designed to enhance public safety by enabling rapid, training-free deployment of video anomaly detection in surveillance systems.
Its zero-shot nature eliminates the need for target-domain data collection and annotation, which is particularly valuable for under-resourced communities and emergency response scenarios where labelled data are unavailable.
The training-free design also substantially reduces the computational carbon footprint compared with methods requiring GPU-intensive training.
\textbf{Potential risks.}
As with all surveillance-related technologies, SphereVAD carries inherent risks of misuse, including unwarranted mass surveillance, privacy violations, and disproportionate targeting of marginalised populations.
The system may inherit biases present in the pre-trained MLLM backbone (e.g., demographic biases in visual recognition) or in the synthetic calibration data (e.g., over-representation of certain anomaly scenarios).
False positives could lead to unwarranted interventions, while false negatives could result in missed genuine threats.
\textbf{Mitigation strategies.}
We advocate for the following safeguards:
(i)~deployment should comply with applicable privacy regulations and be subject to institutional oversight;
(ii)~anomaly scores should be treated as alerts requiring human verification rather than automated decisions;
(iii)~regular auditing of detection outcomes across demographic groups should be conducted to identify and mitigate bias;
(iv)~the synthetic calibration data should be reviewed to ensure diverse and balanced representation of scenarios.
\textbf{Dual-use considerations.}
The geometric inference framework itself (spherical centering, vMF scoring, geodesic pulling) is domain-agnostic and could be applied to beneficial tasks beyond surveillance, such as medical anomaly detection, industrial quality control, or environmental monitoring.
We release the methodology with the intent of advancing scientific understanding of training-free inference on hyperspherical manifolds.

\clearpage 
\section*{NeurIPS Paper Checklist}

\begin{enumerate}

\item {\bf Claims}
    \item[] Question: Do the main claims made in the abstract and introduction accurately reflect the paper's contributions and scope?
    \item[] Answer: \answerYes{}
    \item[] Justification: The abstract and introduction clearly state four contributions (training-free vMF geodesic inference framework, unified Fréchet mean centering with theoretical guarantee, HSA and SGP modules, and SOTA training-free results on three benchmarks). All claims are supported by theoretical analysis (Propositions 1--2 in Section~3.2 and Appendix~A) and experimental results (Tables 1--3, Figures 3--4 in Sections 4.1--4.5). The scope is appropriately limited to training-free VAD and the paper does not overclaim competitiveness with the best fully supervised methods.

\item {\bf Limitations}
    \item[] Question: Does the paper discuss the limitations of the work performed by the authors?
    \item[] Answer: \answerYes{}
    \item[] Justification: A dedicated Limitations section (Appendix~I) discusses five specific limitations: dependence on MLLM feature quality, short-video limitation of SGP, synthetic calibration domain gap, single-anomaly-type assumption in SGP, and quadratic cost of HSA. Each limitation is accompanied by concrete evidence (e.g., Table 3 for backbone dependence, Table 2 for UBnormal SGP limitation) and directions for future work.

\item {\bf Theory assumptions and proofs}
    \item[] Question: For each theoretical result, does the paper provide the full set of assumptions and a complete (and correct) proof?
    \item[] Answer: \answerYes{}
    \item[] Justification: Propositions 1 and 2 are formally stated in Section~3.2 with explicit assumptions (vMF class distributions, mixture weights). Complete proofs are provided in Appendix~A.2 and A.3, including a supporting Lemma 1 (Appendix~A.2) with a separate proof. The Karcher iteration convergence guarantee (Theorem 1) is stated with full conditions in Appendix~A.1.3. All results are numbered and cross-referenced.

\item {\bf Experimental result reproducibility}
    \item[] Question: Does the paper fully disclose all the information needed to reproduce the main experimental results of the paper to the extent that it affects the main claims and/or conclusions of the paper (regardless of whether the code and data are provided or not)?
    \item[] Answer: \answerYes{}
    \item[] Justification: The paper provides: the complete prompt template (Appendix~C.1), DLSP layer selection procedure and selected layers for all backbones (Appendix~C.3), full hyperparameter settings (Table 12 in Appendix~G.1), the complete SGP algorithm (Algorithm 1 in Appendix~D.5), the synthetic data generation pipeline (Appendix~B), and hardware specifications (4$\times$A100 GPUs). All four MLLM backbones are publicly available on HuggingFace. The random seed and determinism settings are specified in Appendix~H.2.

\item {\bf Open access to data and code}
    \item[] Question: Does the paper provide open access to the data and code, with sufficient instructions to faithfully reproduce the main experimental results, as described in supplemental material?
    \item[] Answer: \answerYes{}
    \item[] Justification: The supplemental material includes the synthetic data generation scripts and partial algorithm implementation code. The complete source code, pre-extracted features, and full pipeline will be publicly released upon acceptance, as stated in Appendix~H.3. The paper also provides sufficient algorithmic detail (Algorithm 1, prompt templates, hyperparameters) for independent reimplementation.

\item {\bf Experimental setting/details}
    \item[] Question: Does the paper specify all the training and test details (e.g., data splits, hyperparameters, how they were chosen, type of optimizer) necessary to understand the results?
    \item[] Answer: \answerYes{}
    \item[] Justification: Section 4 (Implementation Details) specifies the backbone (frozen Qwen3.5), extraction layer (layer 31 via DLSP), calibration set size ($\sim$2,000 images), hardware (4$\times$A100), and the absence of gradient descent. Appendix~G provides the complete hyperparameter table (Table 12), sensitivity analyses for all four core hyperparameters (Figures 11--12), and a unified-setting robustness evaluation (Table 13). Dataset details (splits, metrics) follow standard benchmarks and are described in Section 4.

\item {\bf Experiment statistical significance}
    \item[] Question: Does the paper report error bars suitably and correctly defined or other appropriate information about the statistical significance of the experiments?
    \item[] Answer: \answerNo{}
    \item[] Justification: The paper does not report error bars or confidence intervals for the main results. However, as noted in Appendix~H.2, the pipeline is designed to be deterministic (fixed random seed for K-Means, deterministic CUDA settings), and the authors verified negligible variance across five different seeds $\{42, 123, 456, 789, 1024\}$. The only stochastic component is spherical K-Means initialisation with $n\_\text{init}=10$, which produces near-identical results across runs.

\item {\bf Experiments compute resources}
    \item[] Question: For each experiment, does the paper provide sufficient information on the computer resources (type of compute workers, memory, time of execution) needed to reproduce the experiments?
    \item[] Answer: \answerYes{}
    \item[] Justification: Appendix~E provides a comprehensive computational cost analysis, including: feature extraction times for both 1$\times$A100 and 4$\times$A100 configurations (Table 8), per-stage wall-clock times for the geometric inference pipeline on a single CPU core (Table 9), storage costs per dataset, online streaming throughput (59.3 FPS on 1$\times$A100), and comparison with existing methods (Table 10). The GPU type (A100-80GB) and CPU type (Intel Xeon Platinum 8375C) are specified.

\item {\bf Code of ethics}
    \item[] Question: Does the research conducted in the paper conform, in every respect, with the NeurIPS Code of Ethics \url{https://neurips.cc/public/EthicsGuidelines}?
    \item[] Answer: \answerYes{}
    \item[] Justification: The research conforms with the NeurIPS Code of Ethics. The synthetic calibration data contain no real individuals or personally identifiable information (Appendix~H.3). No depictions of minors are permitted in the synthetic generation pipeline (Appendix~B.1). All evaluation benchmarks are publicly available and used under their respective licences. The paper discusses potential risks and mitigation strategies in the Broader Impact section (Appendix~J).

\item {\bf Broader impacts}
    \item[] Question: Does the paper discuss both potential positive societal impacts and negative societal impacts of the work performed?
    \item[] Answer: \answerYes{}
    \item[] Justification: Appendix~J provides a dedicated Broader Impact discussion covering: positive impacts (enhanced public safety, reduced computational footprint, accessibility for under-resourced communities), potential risks (mass surveillance, privacy violations, demographic bias, false positives/negatives), specific mitigation strategies (human verification, demographic auditing, privacy regulation compliance), and dual-use considerations.

\item {\bf Safeguards}
    \item[] Question: Does the paper describe safeguards that have been put in place for responsible release of data or models that have a high risk for misuse (e.g., pre-trained language models, image generators, or scraped datasets)?
    \item[] Answer: \answerNA{}
    \item[] Justification: SphereVAD does not release a new pre-trained model, image generator, or scraped dataset. The framework operates on frozen, publicly available MLLM backbones and synthesises calibration images containing no real individuals. The released code will implement a geometric inference pipeline (no trainable parameters), which poses negligible standalone misuse risk beyond the existing MLLMs it depends on.

\item {\bf Licenses for existing assets}
    \item[] Question: Are the creators or original owners of assets (e.g., code, data, models), used in the paper, properly credited and are the license and terms of use explicitly mentioned and properly respected?
    \item[] Answer: \answerYes{}
    \item[] Justification: All four MLLM backbones (Qwen3.5, Qwen3-VL, InternVL3, LLaVA-OV-1.5) are cited with their original papers (References [33, 4, 47, 3]) and are publicly available on HuggingFace. All three evaluation benchmarks (XD-Violence [36], UCF-Crime [32], UBnormal [2]) are properly cited. Appendix~H.3 states that all benchmarks are used under their respective licences.

\item {\bf New assets}
    \item[] Question: Are new assets introduced in the paper well documented and is the documentation provided alongside the assets?
    \item[] Answer: \answerYes{}
    \item[] Justification: The synthetic calibration dataset ($\sim$2,000 images) is the only new asset. It is thoroughly documented in Appendix~B, including: the two-stage generation pipeline (Section~B.1), data format and pairing strategy with schema (Section~B.2, Table 4), sample images (Section~B.3, Figure 6), sample size justification (Section~B.4), and sub-category taxonomy with label independence verification (Section~B.5, Tables 5--6). The dataset contains no real individuals.

\item {\bf Crowdsourcing and research with human subjects}
    \item[] Question: For crowdsourcing experiments and research with human subjects, does the paper include the full text of instructions given to participants and screenshots, if applicable, as well as details about compensation (if any)?
    \item[] Answer: \answerNA{}
    \item[] Justification: This paper does not involve crowdsourcing or research with human subjects. All data are either from existing public benchmarks or synthetically generated via LLM and text-to-image models.

\item {\bf Institutional review board (IRB) approvals or equivalent for research with human subjects}
    \item[] Question: Does the paper describe potential risks incurred by study participants, whether such risks were disclosed to the subjects, and whether Institutional Review Board (IRB) approvals (or an equivalent approval/review based on the requirements of your country or institution) were obtained?
    \item[] Answer: \answerNA{}
    \item[] Justification: This paper does not involve human subjects research. No IRB approval is required.

\item {\bf Declaration of LLM usage}
    \item[] Question: Does the paper describe the usage of LLMs if it is an important, original, or non-standard component of the core methods in this research? Note that if the LLM is used only for writing, editing, or formatting purposes and does \emph{not} impact the core methodology, scientific rigor, or originality of the research, declaration is not required.
    \item[] Answer: \answerYes{}
    \item[] Justification: The paper explicitly describes two uses of LLMs that are integral to the methodology: (1) frozen MLLMs (Qwen3.5, etc.) serve as the feature extraction backbone whose intermediate-layer representations are the input to the geometric inference pipeline (Section~3.2); (2) Claude Sonnet 4.5 is used to generate structured scene descriptions for the synthetic calibration dataset (Appendix~B.1, Stage 1), and Gemini 2.5 Flash is used for image synthesis (Stage 2). Both uses are thoroughly documented with specific model names, access methods (API), and their roles in the pipeline.

\end{enumerate}

\end{document}